\documentclass{article}
\usepackage{iclr2027_conference,times}
\iclrfinalcopy  

\usepackage[utf8]{inputenc} 
\usepackage[T1]{fontenc}    

\usepackage{url}            
\usepackage{booktabs}       
\usepackage{nicefrac}       
\usepackage{microtype}      
\usepackage{xcolor}         

\usepackage{xparse}
\usepackage{xspace}
\usepackage{xpatch}
\usepackage{textcomp}
\usepackage{lipsum}
\usepackage{todonotes}

\usepackage[normalem]{ulem}   
\usepackage{soul}             

\usepackage{amsmath}
\usepackage{amssymb}
\usepackage{amsfonts}
\usepackage{mathtools}
\usepackage{mathrsfs}
\usepackage{amsthm}
\usepackage{bm}
\usepackage{bbm}
\usepackage{empheq}
\usepackage{extarrows}
\usepackage[makeroom]{cancel}

\usepackage{pifont}
\usepackage{bbding}
\usepackage{dirtytalk}
\usepackage{siunitx}

\usepackage{enumitem}

\usepackage{graphicx}
\usepackage{float}
\usepackage{wrapfig}
\graphicspath{{./figs/}}

\usepackage{makecell}
\usepackage{multirow}
\usepackage{array}
\usepackage{tabularx}
\usepackage{xltabular}
\usepackage{tabularray}
\usepackage{colortbl}

\newcolumntype{Y}{>{\raggedright\arraybackslash}X}
\newcolumntype{C}{>{\centering\arraybackslash}X}

\usepackage{algorithm}
\usepackage{algpseudocode}

\usepackage{fontawesome5}
\usepackage{tcolorbox}
\usepackage{tikz}

\usepackage{listings}
\lstdefinestyle{py}{
    language=Python,
    basicstyle=\ttfamily\small,
    backgroundcolor=\color{gray!5},
    frame=lines,
    columns=fullflexible,
    keepspaces=true,
    showstringspaces=false
}

\usepackage[
    pdfencoding=auto,
    colorlinks=true,
    linkcolor=colorrefs,   
    citecolor=colorcite,   
    urlcolor=linkblue,
]{hyperref}
\usepackage[nameinlink]{cleveref}

\definecolor{neuro}{HTML}{39A300}
\definecolor{comp}{HTML}{077AF8}
\definecolor{theo}{HTML}{EF4136}
\definecolor{RowColorTop}{HTML}{e0dccb}
\definecolor{RowColorDark}{HTML}{f2efe6}
\definecolor{RowColorLight}{HTML}{fdfbf7}
\definecolor{MSBlue}{rgb}{.204,.353,.541}
\definecolor{MSLightBlue}{rgb}{.31,.506,.741}

\definecolor{BlogBlue}{HTML}{2d6a99}
\definecolor{BlogBG}{HTML}{f9f6f5}

\definecolor{py_0}{HTML}{1F77B4}
\definecolor{py_1}{HTML}{FF7F0E}
\definecolor{py_2}{HTML}{2CA02C}
\definecolor{py_3}{HTML}{D62728}
\definecolor{py_4}{HTML}{9467BD}
\definecolor{py_5}{HTML}{8C564B}
\definecolor{py_6}{HTML}{E377C2}
\definecolor{py_7}{HTML}{7F7F7F}
\definecolor{py_8}{HTML}{BCBD22}
\definecolor{py_9}{HTML}{17BECF}

\definecolor{gr}{HTML}{37C532}
\definecolor{grey}{HTML}{4D4D4D}
\definecolor{gray}{HTML}{4D4D4D}
\definecolor{lightgray}{HTML}{666666}
\definecolor{lightlightgray}{HTML}{999999}
\definecolor{color_enc}{HTML}{ED1C24}
\definecolor{color_dec}{HTML}{29ABE2}
\definecolor{color_fwd}{HTML}{2ea52a}
\definecolor{color_decenc}{HTML}{dd1aff}
\def\enc#1{{\color{color_enc}{#1}}\xspace}
\def\dec#1{{\color{color_dec}{#1}}\xspace}

\definecolor{colorcite}{HTML}{7038b2}
\definecolor{colorrefs}{HTML}{8e5325}
\definecolor{linkblue}{rgb}{0.3, 0.4, 0.8}
\definecolor{journalred}{rgb}{0.8, 0.1, 0.3}

\DeclarePairedDelimiter{\norminner}{\lVert}{\rVert}

\NewDocumentCommand{\norm}{s o m}{%
    \IfBooleanTF{#1}%
    {\norminner{#3}}
    {%
        \IfNoValueTF{#2}
        {\norminner*{#3}}
        {\norminner[#2]{#3}}
    }%
}

\DeclarePairedDelimiterX{\klinner}[2]{(}{)}{#1\,\delimsize\|\,#2}

\NewDocumentCommand{\kl}{s o m m}{%
    \cD_{\mathtt{KL}}%
    \IfBooleanTF{#1}%
    {\klinner{#3}{#4}}
    {%
        \IfNoValueTF{#2}%
        {\klinner*{#3}{#4}}
        {\klinner[#2]{#3}{#4}}
    }%
}

\newcommand{\KL}[2]{\cD_{\mathtt{KL}}\Big{(}{#1}\,\big\|\,{#2}\Big{)}}

\DeclarePairedDelimiter{\expectinner}{[}{]}
\newsavebox{\expcontent}
\newsavebox{\expfences}

\NewDocumentCommand{\expect}{s o m m}{%
    \EE_{#3}%
    \IfBooleanTF{#1}%
    {\expectinner{#4}}
    {%
        \IfNoValueTF{#2}%
        {
         \sbox{\expcontent}{$\displaystyle #4$}%
         \sbox{\expfences}{$\displaystyle \expectinner{\vphantom{X}}$}%
         \ifdim\dimexpr\ht\expcontent+\dp\expcontent\relax>\dimexpr\ht\expfences+\dp\expfences\relax
             \!\expectinner*{#4}%
         \else
             \expectinner{#4}%
         \fi
        }%
        {
         \!\expectinner[#2]{#4}%
        }%
    }%
}

\NewDocumentCommand{\pderiv}{o m e{^}}{%
    \frac{\partial\IfValueT{#3}{^{#3}}\IfValueT{#1}{#1}}
         {\partial #2\IfValueT{#3}{^{#3}}}
}

\NewDocumentCommand{\deriv}{o m e{^}}{%
    \frac{d\IfValueT{#3}{^{#3}}\IfValueT{#1}{#1}}
         {d #2\IfValueT{#3}{^{#3}}}
}

\newcommand{\argmin}[1]{\underset{#1}{\operatorname{argmin}}\ }

\newcommand{\minimize}[1]{\underset{#1}{\operatorname{minimize}}\ }

\DeclareRobustCommand{\ba}{\bm{a}}
\DeclareRobustCommand{\bb}{\bm{b}}

\DeclareRobustCommand{\bd}{\bm{d}}
\DeclareRobustCommand{\be}{\bm{e}}
\DeclareRobustCommand{\bg}{\bm{g}}

\DeclareRobustCommand{\bp}{\bm{p}}

\DeclareRobustCommand{\br}{\bm{r}}
\DeclareRobustCommand{\bs}{\bm{s}}

\DeclareRobustCommand{\bu}{\bm{u}}
\DeclareRobustCommand{\bv}{\bm{v}}
\DeclareRobustCommand{\bw}{\bm{w}}
\DeclareRobustCommand{\bx}{\bm{x}}

\DeclareRobustCommand{\bz}{\bm{z}}

\DeclareRobustCommand{\bdelta}{\bm{\delta}}
\DeclareRobustCommand{\bepsilon}{\bm{\epsilon}}

\DeclareRobustCommand{\blambda}{\bm{\lambda}}
\DeclareRobustCommand{\bmu}{\bm{\mu}}

\DeclareRobustCommand{\bphi}{\bm{\phi}}
\DeclareRobustCommand{\bpi}{\bm{\pi}}

\DeclareRobustCommand{\btheta}{\bm{\theta}}

\DeclareRobustCommand{\bPhi}{\bm{\Phi}}

\DeclareRobustCommand{\bSigma}{\bm{\Sigma}}

\DeclareRobustCommand{\bA}{\mathbf{A}}  
\DeclareRobustCommand{\bC}{\mathbf{C}}  
\DeclareRobustCommand{\bD}{\mathbf{D}}  
\DeclareRobustCommand{\bE}{\mathbf{E}}  
\DeclareRobustCommand{\bF}{\mathbf{F}}  
\DeclareRobustCommand{\bG}{\mathbf{G}}  
\DeclareRobustCommand{\bI}{\mathbf{I}}  
\DeclareRobustCommand{\bJ}{\mathbf{J}}  
\DeclareRobustCommand{\bM}{\mathbf{M}}  
\DeclareRobustCommand{\bQ}{\mathbf{Q}}  
\DeclareRobustCommand{\bR}{\mathbf{R}}  
\DeclareRobustCommand{\bS}{\mathbf{S}}  
\DeclareRobustCommand{\bV}{\mathbf{V}}  
\DeclareRobustCommand{\bW}{\mathbf{W}}  
\DeclareRobustCommand{\bZ}{\mathbf{Z}}  

\DeclareRobustCommand{\cD}{\mathcal{D}}  
\DeclareRobustCommand{\cF}{\mathcal{F}}  
\DeclareRobustCommand{\cN}{\mathcal{N}}  
\DeclareRobustCommand{\cR}{\mathcal{R}}  

\DeclareRobustCommand{\EE}{\mathbb{E}}  
\DeclareRobustCommand{\RR}{\mathbb{R}}  

\def\1{\bm{1}}

\newcommand{\ssub}[1]{_{\scriptscriptstyle #1}}

\DeclareMathOperator{\Tr}{Tr}

\DeclareMathOperator{\diag}{diag}

\NewDocumentCommand{\pwrld}{ e{^} }{%
  p\IfValueT{#1}{^{#1}}_{\text{world}}%
}
\NewDocumentCommand{\pself}{ e{^} }{%
  p\IfValueT{#1}{^{#1}}_{\text{\textcolor{color_dec}{self}-world}}%
}
\NewDocumentCommand{\pbrn}{ e{^} }{%
  p\IfValueT{#1}{^{#1}}_{{\dec{\text{brain}}}}%
}
\NewDocumentCommand{\Pobs}{ e{^} }{%
  \hat{\dec{P}}\IfValueT{#1}{^{#1}}_{\text{obs}}%
}

\NewDocumentCommand{\boxit}{ O{} +m }{
    \begin{tcolorbox}[
        width=\linewidth,     
        colback=yellow!20, 
        colframe=darkgray,    
        halign title=left,    
        halign=center,        
        before upper={\thickmuskip=5mu\relax \medmuskip=4mu\relax},
        #1
    ]
    #2
    \end{tcolorbox}
}

\NewDocumentCommand{\note}{ O{Note} +m }{%
    \begin{tcolorbox}[title={#1},colback=BlogBG,colframe=MSBlue,sharp corners,boxrule=0.6pt]
    #2
    \end{tcolorbox}
}

\newcommand{\mlion}{\ifmmode\mathrm{M\mbox{-}Lion}\else M-Lion\xspace\fi}
\newcommand{\mmuon}{\ifmmode\mathrm{M\mbox{-}Muon}\else M-Muon\xspace\fi}
\newcommand{\msignum}{\ifmmode\mathrm{M\mbox{-}Signum}\else M-Signum\xspace\fi}
\newcommand{\LogLion}{\ifmmode\mathrm{Log\mbox{-}Lion}\else Log-Lion\xspace\fi}
\newcommand{\TanhLion}{\ifmmode\mathrm{Tanh\mbox{-}Lion}\else Tanh-Lion\xspace\fi}
\newcommand{\LogMuon}{\ifmmode\mathrm{Log\mbox{-}Muon}\else Log-Muon\xspace\fi}
\newcommand{\TanhMuon}{\ifmmode\mathrm{Tanh\mbox{-}Muon}\else Tanh-Muon\xspace\fi}

\def\sayit#1{\say{\textit{{#1}}}}

\def\entry#1#2{${#1}\scriptscriptstyle{\color{lightgray}\pm{#2}}$}   

\NewDocumentCommand{\quoteblue}{m o o}{
    \IfValueTF{#2}{\def\LeftMargin{#2}}{\def\LeftMargin{0.10}} 
    \IfValueTF{#3}{\def\RightMargin{#3}}{\def\RightMargin{\LeftMargin}}
    \hspace{\LeftMargin\textwidth}%
    \begin{tabular}{
        !{\color{BlogBlue}\vrule width 2pt}
        p{\dimexpr\textwidth - \LeftMargin\textwidth - \RightMargin\textwidth\relax}
    }
    {#1} 
    \end{tabular}
}
\NewDocumentCommand{\circled}{ O{1.8pt} O{0.8pt} m }{
    \tikz[baseline=(char.base)]{
        \node[shape=circle,draw,line width=#2,inner sep=#1] (char) {#3};
    }
}
\newcommand{\underbracegray}[2]{
    {\color{lightlightgray}
        \underbrace{\color{black}#1}_{\color{black}#2}
    }
}
\makeatletter
\newcommand{\namelabel}[1]{%
  \phantomsection
  \renewcommand{\@currentlabel}{#1}%
  \label{#1}%
}
\makeatother
\makeatletter
\RenewDocumentCommand{\boxed}{O{10pt} O{0.8pt} m}{%
  \begingroup
  \setlength{\fboxsep}{#1}%
  \setlength{\fboxrule}{#2}%
  \fbox{$\m@th\displaystyle #3$}%
  \endgroup
}
\makeatother
\makeatletter
\NewDocumentCommand{\boxedcap}{O{1.0pt} O{0.2pt} m}{%
  \begingroup
  \setlength{\fboxsep}{#1}%
  \setlength{\fboxrule}{#2}%
  \fbox{$\m@th\displaystyle #3$}%
  \endgroup
}
\makeatother

\let\oldtodo\todo
\RenewDocumentCommand{\todo}{O{} m}{%
    \oldtodo[%
        backgroundcolor=orange,
        textcolor=white,
        size=\small,
        #1
    ]{#2}
}

\crefname{figure}{Fig.}{Figs.}
\Crefname{figure}{Figure}{Figures}

\crefname{equation}{eq.}{eqs.}
\Crefname{equation}{Equation}{Equations}

\crefname{table}{Table}{Tables}
\Crefname{table}{Table}{Tables}

\crefname{algorithm}{Alg.}{Algs.}
\Crefname{algorithm}{Algorithm}{Algorithms}

\newtheoremstyle{liteplain}
  {\dimexpr\parskip+6pt\relax} 
  {10pt}                        
  {\itshape}                    
  {}                            
  {\bfseries}                   
  {.}                           
  {0.5em}                       
  {}                            

\newtheoremstyle{litedefinition}
  {\dimexpr\parskip+6pt\relax} 
  {10pt}                        
  {\normalfont}                 
  {}                            
  {\bfseries}                   
  {.}                           
  {0.5em}                       
  {}                            

\theoremstyle{liteplain}

\newtheorem{theorem}{Theorem}[section]
\newtheorem{proposition}[theorem]{Proposition}

\newtheorem{corollary}[theorem]{Corollary}

\theoremstyle{litedefinition}

\theoremstyle{liteplain}
\newtheorem*{theorem*}{Theorem}
\newtheorem*{proposition*}{Proposition}
\newtheorem*{lemma*}{Lemma}
\newtheorem*{corollary*}{Corollary}

\theoremstyle{litedefinition}
\newtheorem*{definition*}{Definition}
\newtheorem*{example*}{Example}
\newtheorem*{remark*}{Remark}

\crefname{theorem}{theorem}{theorems}
\Crefname{theorem}{Theorem}{Theorems}

\crefname{proposition}{proposition}{propositions}
\Crefname{proposition}{Proposition}{Propositions}

\crefname{lemma}{lemma}{lemmas}
\Crefname{lemma}{Lemma}{Lemmas}

\crefname{corollary}{corollary}{corollaries}
\Crefname{corollary}{Corollary}{Corollaries}

\crefname{definition}{definition}{definitions}
\Crefname{definition}{Definition}{Definitions}

\crefname{example}{example}{examples}
\Crefname{example}{Example}{Examples}

\crefname{remark}{remark}{remarks}
\Crefname{remark}{Remark}{Remarks}

\newcommand{\papertitle}{Inference and learning in sparse autoencoders as natural gradient flow\xspace}

\title{\papertitle}

\author{%
    \textbf{
        Hadi Vafaii$^{1,2,*}$ \quad
        Tejas Rao$^{2,*}$ \quad
        David Chanin$^{3,4}$ \quad
        Thomas Fel$^{4}$ \quad
        Jacob L.~Yates$^{2}$
    }\\[0.3em]
    \textbf{
        Bruno Olshausen$^{2}$ \quad
        David Klindt$^{5}$ \quad
        Dileep George$^{1}$ \quad
        Miguel L\'{a}zaro-Gredilla$^{1}$
    }\\[0.7em]
    $^{1}$Astera Institute \qquad
    $^{2}$UC Berkeley \qquad
    $^{3}$UCL \qquad
    $^{4}$Goodfire AI \qquad
    $^{5}$CSHL\\[0.5em]
    $^{*}${\small \textit{Co-first authors, correspondence to}: \quad
    \href{mailto:hadi@astera.org}{\texttt{hadi@astera.org}} \quad
    \href{mailto:tejasrao@berkeley.edu}{\texttt{tejasrao@berkeley.edu}}}
}

\newcommand{\befond}{\ensuremath{\mathrm{Be} }\text{FOND}\xspace}
\newcommand{\synthsaebench}{\texttt{SynthSAEBench}\xspace}
\newcommand{\saebench}{\texttt{SAEBench}\xspace}

\begin{document}

\maketitle
\lhead{}

\addtocontents{toc}{\protect\setcounter{tocdepth}{0}}

\begin{abstract}
Sparse autoencoders are widely used to uncover interpretable features in neural networks, yet reliable recovery remains difficult when features overlap or activate infrequently. These challenges involve both inferring which features explain an input and learning the dictionary that represents them. Here, we unify inference and dictionary learning as natural-gradient flows on a shared variational free energy. We instantiate this framework as \befond, an encoder-free sparse coding model with closed-form inference and learning dynamics. We show how recurrent explaining away reduces interference between overlapping features, while Fisher preconditioning can compensate for the slow learning of rare features. On synthetic data, \befond improves dictionary recovery and rare-feature detection, with a growing advantage over amortized baselines as superposition increases. On language-model activations, it improves single-feature concept detection and selective intervention, outperforming pretrained reference SAEs with substantially less training data. Its feature quality continues to improve with dictionary width, whereas the evaluated baselines largely plateau. Together, these results show how improving inference and learning within a unified probabilistic framework can make better use of data and dictionary capacity to interpret and intervene on neural representations.
\end{abstract}

\section{Introduction}
\label{sec:introduction}
%
Sparse autoencoders (SAEs) infer sparse representations of data in terms of a learned, overcomplete dictionary of features. Their roots lie in sparse coding \citep{olshausen1996emergence,lewicki2000learning}, motivated by the organization of sensory representations in biological systems \citep{olshausen1997sparse,olshausen2004sparse}. SAEs have since become central in interpreting neural network representations \citep[reviewed in][]{klindt2026unifying}, especially those of language models \citep{yun2021transformer,bricken2023monosemanticity,cunningham2024sparse,templeton2024scaling,gao2025scaling}.

Despite these successes, SAEs face challenges in both inference and dictionary learning. Standard SAEs use a trained neural network to perform one-pass inference---known as \textit{amortized} inference---which can struggle to disentangle overlapping features in superposition \citep{oneill2025compute}. But improving inference alone is insufficient. In controlled experiments, \citet{pacela2026stop} identified dictionary learning as the bottleneck: replacing an SAE's encoder with iterative inference on the same learned dictionary failed to close the feature-recovery gap. These problems are coupled during training: the dictionary determines which explanations of the data are available, while inferred feature activations shape the learning signal used to update the dictionary. Therefore, inference errors can compromise dictionary learning, and a poor dictionary can limit what inference can recover.

Unlike standard SAEs, classical sparse coding methods typically perform inference through iterative, per-sample optimization \citep{rozell2008sparse,gregor2010learning}. More recently in this lineage, \citet{vafaii2025brain} introduced the \textit{\underline{F}ree-energy \underline{O}nline \underline{N}atural-gradient \underline{D}ynamics} (FOND) framework, which formulates inference as natural-gradient flow on variational free energy. FOND's guiding principle is to respect the intrinsic geometry of the statistical manifold: updates are measured by their effect on the represented probability distribution through the Fisher information metric \citep{amari1998natural,martens2020new,khan2023bayesian}. However, FOND did not address learning through the same geometric principle. Its generative model parameters were learned by backpropagating through the inference trajectory, with the associated computational and memory costs.

Here, we extend this principle to learning. Just as \enc{inference} follows natural-gradient flow on free energy with respect to \enc{variational} parameters, we define \dec{learning} as natural-gradient flow on the same objective with respect to \dec{generative} parameters. Both processes thus respect the statistical geometry of the distributions they update. We instantiate this extended FOND framework as \befond: an encoder-free iterative SAE with \underline{Be}rnoulli latents, a factorized prior, and a linear Gaussian decoder. We derive closed-form expressions for its inference and learning dynamics, allowing the dictionary to be learned without backpropagating through the inference trajectory.

\begin{figure}[t]
    \centering
    \includegraphics[width=\textwidth]{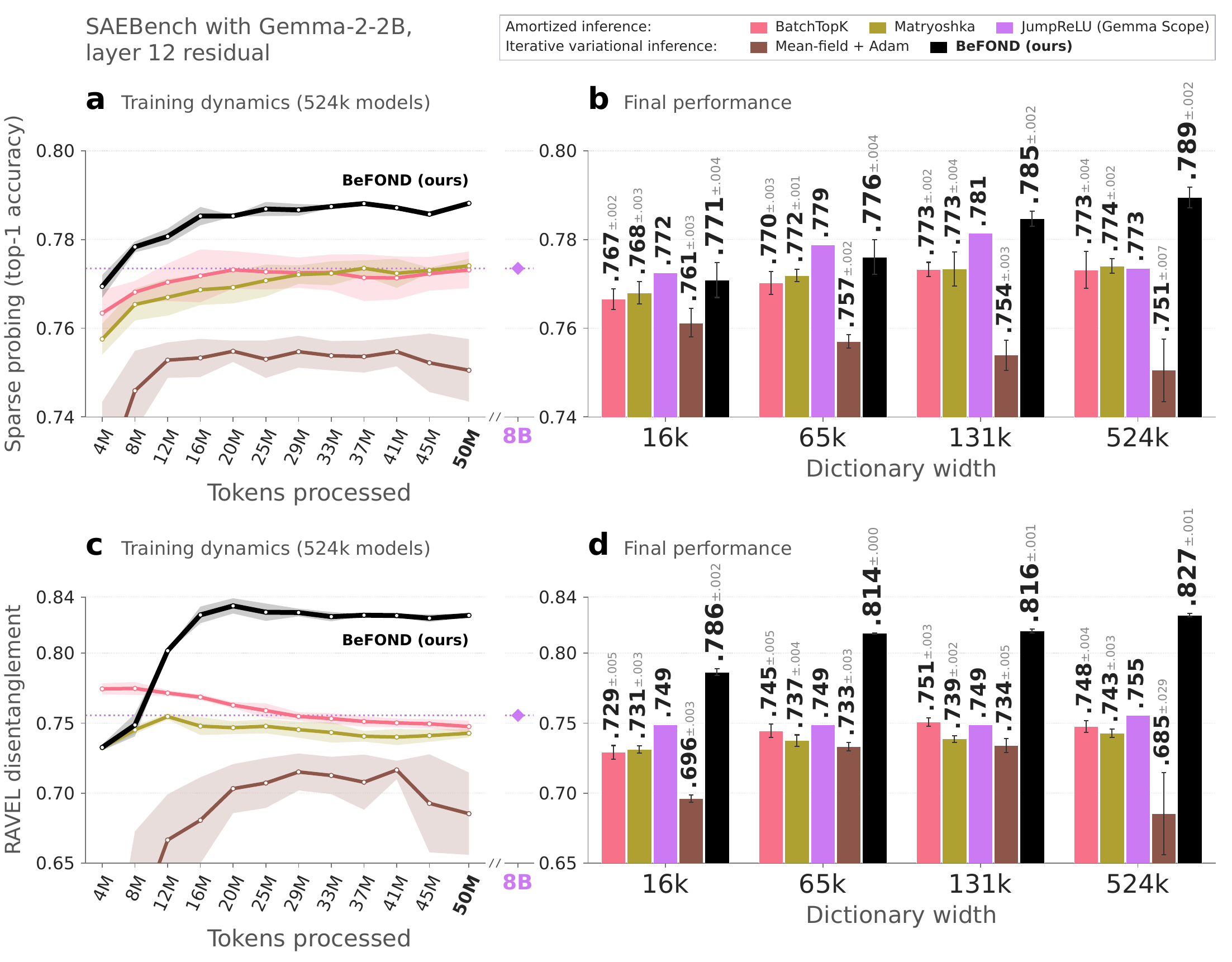}
    \caption{\textbf{Feature quality and sample efficiency on Gemma-2-2B activations.} \textbf{a,c}: Training dynamics at $524$k width. \textbf{b,d}: Final performance across dictionary widths. Details in \cref{sec:experiments}.}
    \label{fig:gemma}
\end{figure}

On the \enc{inference} side, \befond implements \textit{explaining away} through recurrent interactions between overlapping features, enabling lateral competition leading to sparse, disentangled representations. We isolate the contribution of these interactions and show that they become increasingly important as feature superposition increases. On the \dec{learning} side, we show how infrequent activation slows Euclidean dictionary learning. We prove that Fisher preconditioning removes this frequency dependence, giving rare and common features the same local learning rate. Motivated by this result, we hypothesize that Fisher preconditioning would improve rare-feature recovery in \befond. Empirically, \befond learns more accurate dictionaries and detects rare features with few misses or false positives, while using substantially fewer training samples than SAE baselines and Euclidean ablations.

We evaluate \befond on synthetic data with known ground-truth features using \synthsaebench \citep{chanin2026synthsaebench}, and on language-model representations using \saebench \citep{karvonen2025saebench}. Baselines include amortized SAEs---BatchTopK \citep{bussmann2024batchtopk}, Matryoshka BatchTopK \citep{bussmann2025learning}, and JumpReLU \citep{rajamanoharan2024jumping}---and an iterative mean-field variational baseline with Adam decoder learning \citep{kingma2015adam}.

On Gemma-2-2B layer-12 activations, \befond outperforms all evaluated baselines on both sparse probing \citep{kantamneni2025sparse} and RAVEL disentanglement \citep{huang2024ravel}. \befond's performance continues to improve as we increase dictionary width from $16$k to $524$k latents, whereas the baselines plateau (\cref{fig:gemma}b,d). These gains also emerge with substantially less training data: at a dictionary width of $524$k, \befond matches the evaluated Gemma Scope JumpReLU checkpoints \citep{lieberum2024gemma} in sparse probing after just $\sim$6M training tokens, compared with the 8B tokens used to train those checkpoints, achieving over 1{,}000$\times$ reduction in training tokens (\cref{fig:gemma}a).

Summary of our contributions and paper outline:

\textbf{(1)} We review necessary background (\cref{sec:background}). We then extend FOND \citep{vafaii2025brain} to unify \enc{inference} and \dec{learning} as natural-gradient flows on variational free energy, instantiate it as Bernoulli-distributed \befond, and derive its dynamics in closed-form (\cref{sec:theory} and \cref{app:full-model}).

\textbf{(2)} We analyze \befond's recurrent ``explaining away'' (\cref{sec:theory_inference}) and use ablations to establish its contribution to feature recovery under increasing superposition (\cref{app:superposition-results}).

\textbf{(3)} We characterize how Fisher preconditioning in \befond's learning rule compensates for feature rarity (\cref{sec:theory_learning}), and test its consequences for dictionary learning through theoretical analysis (\cref{thm:exact-rarity-compensation}), synthetic experiments, and optimizer ablations (\cref{app:learning-geometry,app:geometry-results,app:full-width-results}).

\textbf{(4)} We evaluate \befond on language-model representations, examining feature quality, sample efficiency, and scaling with dictionary width (\cref{sec:experiments}).

\section{Background and Related Work}
\label{sec:background}
\paragraph{Notation.}
Bold lowercase and uppercase symbols denote vectors and matrices, such as $\ba$ and $\bA$. \enc{Inference} and \dec{learning} quantities appear in \enc{red} and \dec{blue}, respectively.

\paragraph{Sparse coding.}
Classical sparse coding \citep{olshausen1996emergence,olshausen1997sparse,
lewicki2000learning,aharon2006rm,rozell2008sparse,
gregor2010learning,mairal2010online} models each observation $\bx \in \RR^M$ as a linear combination of dictionary atoms with sparse coefficients $\enc{\bz}$: 
\begin{equation}\label{eq:linear-generative-model}
    \bx \;=\;
    \dec{\bPhi} \enc{\bz} + \bepsilon~,
    \qquad \bepsilon \sim \cN(\bm{0}, \bI_M)~,
\end{equation}
The \textit{dictionary} $\dec{\bPhi}\in\RR^{M\times K}$ has \textit{atoms} $\dec{\bPhi}_{\cdot i}\equiv\dec{\bphi}^{(i)}$ as its columns.

We consider \textit{overcomplete} dictionaries, whose atoms outnumber the input dimensions ($K>M$), and therefore cannot all be mutually orthogonal, making some overlap between atoms unavoidable.

\dec{Learning} estimates the dictionary from data, while \enc{inference} finds sparse codes that explain each input. These processes are coupled because inference uses the current dictionary, and dictionary updates depend on the inferred codes. In the overcomplete setting, \textit{superposition}---overlap between atoms---makes sparse code inference particularly challenging (see \cref{app:inference-challenge-superposition}).

\paragraph{Amortized versus iterative inference.} 
Standard SAEs learn an encoder
$\enc{\bz} = f_{\enc{\bphi}}(\bx)$ that predicts a sparse code
in one forward pass. These include ReLU
\citep{bricken2023monosemanticity,cunningham2024sparse,templeton2024scaling},
JumpReLU \citep{rajamanoharan2024jumping},
gated \citep{rajamanoharan2024improving},
TopK \citep{makhzani2014k,gao2025scaling},
BatchTopK \citep{bussmann2024batchtopk},
and Matryoshka BatchTopK \citep{bussmann2025learning} SAEs.
Training this mapping across samples \textit{amortizes} inference, replacing a separate test-time optimization for each input.

Iterative sparse coding instead refines the code for each observation \citep{rozell2008sparse,beck2009fast,fang2022learning}. In continuous time, this takes the schematic form, $\dot{\enc{\bz}} = g(\enc{\bz},\dec{\bPhi},\bx)$. Iterative inference has yielded mixed results on language-model representations \citep{chanin2025training,sun2026price}. See \cref{app:iterative sparse} for further discussion. 

%

\paragraph{Variational inference and free energy.}
A generative model combines a prior with a likelihood:
\begin{equation}\label{eq:background-probabilistic-model}
    \underbracegray{
        p_{\dec{\btheta}}(\bx,\bz)
        }{\text{joint distribution}}
    \;=
    \underbracegray{
        p_{\dec{\btheta}}(\bx\mid\bz)
        }{\text{conditional likelihood}}
    \;
    \underbracegray{
        p_{\dec{\btheta}}(\bz)
        }{\text{prior}}~,
    \qquad\quad
    \text{e.g., }\;\;
    p_{\dec{\btheta}}(\bx\mid\bz)
    \;\;=
    \underbracegray{
        \cN(\bx;\dec{\bPhi}\bz,\bI_M)
        }{\text{probabilistic form of \cref{eq:linear-generative-model}}
    }.
\end{equation}
Bayes' rule gives $p_{\dec{\btheta}}(\bz\mid\bx)=p_{\dec{\btheta}}(\bx,\bz)/p_{\dec{\btheta}}(\bx)$. Computing the evidence $p_{\dec{\btheta}}(\bx)$ requires summing or integrating over latent configurations, which is generally intractable for large overcomplete models. \textit{Variational inference} \citep[\cref{app:free-energy-derivation};][]{jordan1999introduction,blei2017variational} turns this problem into optimization: choose an \textit{approximate posterior} $q_{\enc{\bu}}(\bz\mid\bx)$ and optimize its parameters $\enc{\bu}$. For example, $\enc{\bu}$ can be Gaussian means and variances, or Bernoulli logits. Define the \textit{variational free energy}:
\begin{equation}\label{eq:background-free-energy}
    \cF(\bx; \enc{\bu}, \dec{\btheta})
    \;\coloneqq\;
    \expect{q_{\enc{\bu}}}{-\log\frac{p_{\dec{\btheta}}(\bx,\bz)}{q_{\enc{\bu}}(\bz\mid\bx)}}
    \;=\;
    -\expect[\Big]{q_{\enc{\bu}}}{\log p_{\dec{\btheta}}(\bx\mid\bz)}
    +
    \kl[\Big]{q_{\enc{\bu}}(\bz\mid\bx)}{p_{\dec{\btheta}}(\bz)}~.
\end{equation}
Free energy is the same quantity as the negative evidence lower bound, $\cF=-\mathrm{ELBO}$. Its terms balance explaining the observation against departing from the prior. The connection between free-energy minimization and posterior inference follows from the folloing identity:
\begin{equation}\label{eq:background-free-energy-posterior}
    \cF(\bx; \enc{\bu},\dec{\btheta})
    \;=\;
    \kl{q_{\enc{\bu}}(\bz\mid\bx)}{p_{\dec{\btheta}}(\bz\mid\bx)}
    \;-\;
    \log p_{\dec{\btheta}}(\bx)~.
\end{equation}
This shows that, with $\bx$ and $\dec{\btheta}$ fixed, reducing $\cF$ reduces the KL divergence to the true posterior because the evidence is independent of $\enc{\bu}$ (\cref{app:free-energy-derivation}). An amortized encoder can predict $\enc{\bu} = f_\enc{\phi}(\bx)$, as in VAEs \citep{kingma2014auto}, or inference can optimize it separately for each input:
\begin{equation}\label{eq:inference-as-optimization}
    \enc{\bu}^*(\bx)
    \;=\;
    \argmin{\enc{\bu}}
    \cF(\bx; \enc{\bu}, \dec{\btheta})~.
\end{equation}
Holding $\bx$ and $\dec{\btheta}$ fixed, write the objective as $\cF(\enc{\bu})$. Starting from some initialization $\enc{\bu}_0$, gradient descent with step size $\eta>0$ gives:
\begin{equation}
    \enc{\bu}_{t+1}
    \;=\;
    \enc{\bu}_{t}
    \;-\;
    \eta\,\nabla_{\enc{\bu}}\cF(\enc{\bu}_t)~.
\end{equation}
This update measures movement using Euclidean distance in parameter space (\cref{app:trust-region}). However, equal-sized parameter changes can produce different changes in beliefs $q_\enc{\bu}$, and reparameterizing the same distributions can change the descent trajectory. Measuring movement between distributions instead leads to the geometry of their \textit{statistical manifold} \citep{amari2000methods}.

\paragraph{From gradient descent to natural gradients.}
Swapping the Euclidean geometry of gradient descent with the Fisher metric gives the \textit{natural-gradient} update:
\begin{equation}\label{eq:fond-ng-inference}
\boxed{
    \enc{\bu}_{t+1}
    \;=\;
    \enc{\bu}_{t}
    \;-\;
    \eta\,\bG_u^{-1}(\enc{\bu}_t)
    \nabla_{\enc{\bu}}\cF(\enc{\bu}_t)~.
}
\end{equation} 
where $\bG_u$ is the Fisher information matrix of the variational family:
\begin{equation}
    \bG_u(\enc{\bu})
    \;=\;
    \EE_{\bz\sim q_{\enc{\bu}}(\cdot\mid\bx)}
    \left[
        \big(\nabla_{\enc{\bu}}\log q_{\enc{\bu}}(\bz\mid\bx)\big)
        \big(\nabla_{\enc{\bu}}\log q_{\enc{\bu}}(\bz\mid\bx)\big)^\top
    \right]~.
\end{equation}
The Fisher matrix measures local statistical distinguishability, so the update accounts for how parameter changes affect the represented distribution \citep{amari1998natural,amari2000methods,amari2016information}. Natural gradients have a long history in variational inference, underpinning stochastic variational inference \citep{hoffman2013stochastic}, the Bayesian learning rule \citep{khan2023bayesian}, and Bayesian Online Natural Gradient \citep{jones2024bayesian}. Closed-form natural gradients need not be available, and computing the Fisher and applying its inverse can be expensive, motivating approximations \citep{martens2015optimizing,martens2020new}. Natural Reweighted Wake--Sleep uses natural gradients for generative and recognition networks \citep{varady2022natural}, and recent work preconditions TopK SAE decoder updates using a damped empirical code Gram matrix \citep{picozzi2026information}. Both approaches retain learned encoders, whereas \befond directly optimizes variational posterior parameters for each input through recurrent inference.

Building on this geometric perspective, \citet{vafaii2025brain} introduced FOND by prescribing natural-gradient flow on $\cF$ with respect to $\enc{\bu}$. A forward Euler discretization yields \cref{eq:fond-ng-inference}.

In the appendix, we motivate natural gradients using the standard \textit{trust region} argument \citep[\cref{app:trust-region};][]{martens2020new,conn2000trust,ollivier2017information}, as well as a broader ``physics of adaptation'' perspective (\cref{app:physics_of_adaptation}), inspired by the \textit{principle of general covariance} \citep{einstein1916grundlage}.

\section{FOND unifies inference and learning as natural gradient flow}
\label{sec:theory}
%
FOND \citep{vafaii2025brain} formulated inference as natural gradient descent on variational free energy (\cref{eq:fond-ng-inference}). But it stopped short of prescribing learning dynamics through the same principle.

The geometric argument applies equally to learning: just as $\enc{\bu}$ changes the approximate posterior $q_{\enc{\bu}}$, the generative parameters $\dec{\btheta}$ change the joint distribution $p_{\dec{\btheta}}(\bx,\bz)$. This motivates extending FOND by defining learning as natural-gradient flow on the same free energy with respect to $\dec{\btheta}$.

Thus, inference and learning are both treated as dynamical processes within a unified framework:

\namelabel{fond}
\boxit[title = {FOND: Free-energy Online Natural-gradient Dynamics}, halign = left, halign title = center]{
\begingroup
\renewcommand{\tabularxcolumn}[1]{m{#1}}

\noindent
\begin{tabularx}{\linewidth}{
    @{} >{\bfseries\raggedright\arraybackslash}m{0.25\linewidth}
    @{\hspace{5mm}} X @{}
}
\hspace{6mm}
Flexible choices:
&
Choose a generative model $p_{\dec{\btheta}}(\bx, \bz)$, an inference model $q_{\enc{\bu}}(\bz \mid \bx)$,
and calculate the variational free energy
$\cF(\enc{\bu}, \dec{\btheta})$.
\\[3.5mm]
\hspace{-2mm}
FOND's prescription:
&
Define inference and learning as natural gradient flows on $\cF$:
\end{tabularx}

\medskip
\noindent
\begin{tabularx}{\linewidth}{
    @{} >{\centering\arraybackslash}X
    >{\centering\arraybackslash}X @{}
}
\textbf{Inference:}\quad
$\displaystyle
    \dot{\enc{\bu}}
    \,\coloneqq\,
    -\bG_u^{-1}(\enc{\bu})\,
    \pderiv[\cF(\enc{\bu}, \dec{\btheta})]{\enc{\bu}}~,
$
&
\textbf{Learning:}\quad
$\displaystyle
    \dot{\dec{\btheta}}
    \,\coloneqq\,
    -\bG_\theta^{-1}(\dec{\btheta})\,
    \pderiv[\cF(\enc{\bu}, \dec{\btheta})]{\dec{\btheta}}~.
$
\end{tabularx}
\endgroup
}

\subsection{\befond: a Bernoulli-distributed iterative sparse coding model}
\label{sec:theory-befond}
We instantiate our \hyperref[fond]{extended FOND} by following its three steps: choose the distributions, compute the free energy, and derive the dynamics.

\paragraph{Distribution specifications.}
A generative model requires a prior and a conditional likelihood (\cref{eq:background-probabilistic-model}). We choose a factorized Bernoulli prior and a linear Gaussian decoder, together with a factorized Bernoulli variational posterior:
\begin{center}
\begingroup
\setlength{\tabcolsep}{9pt}
\renewcommand{\arraystretch}{1.35}
\newcommand{\befondmodeleqnumber}[1]{%
  \refstepcounter{equation}\label{#1}\textup{(\theequation)}%
}
\begin{tabular}{@{}c|c|c@{}}
\toprule
\textit{Likelihood}\enspace $p_{\dec{\btheta}}(\bx\mid\bz)$
& \textit{Prior}\enspace $p_{\dec{\btheta}}(\bz)$
& \textit{Posterior}\enspace $q_{\enc{\bu}}(\bz\mid\bx)$ \\
\midrule
$\displaystyle\cN(\bx;\dec{\bPhi}\bz,\dec{\bSigma})$
\enspace\befondmodeleqnumber{eq:befond-likelihood}
& $\displaystyle\prod_{i=1}^{K}\operatorname{Bernoulli}(z_i;\dec{\pi}_{0,i})$
\enspace\befondmodeleqnumber{eq:befond-prior}
& $\displaystyle\prod_{i=1}^{K}\operatorname{Bernoulli}(z_i;\enc{r}_i)$
\enspace\befondmodeleqnumber{eq:befond-posterior} \\ \addlinespace[-2mm]
$\dec{\bSigma}$ {\small diagonal and positive}
& $\dec{\bpi}_0 = \sigma(\dec{\bu}_0)$
& $\enc{\br} = \sigma(\enc{\bu})$ \\
\bottomrule
\end{tabular}
\endgroup
\end{center}
Here $\bz\in\{0,1\}^{K}$, and $\sigma(\cdot)$ is the logistic sigmoid. The generative parameters are $\dec{\btheta}\equiv(\dec{\bPhi},\dec{\bSigma},\dec{\bu}_0)$, and the variational parameters $\enc{\bu}$ are the posterior logits. 


\paragraph{Free energy.}
Define $\be \coloneqq \bx-\dec{\bPhi}\enc{\br}$ as the mean reconstruction residual, $\dec{\bW}  \coloneqq \dec{\bPhi}^{\top}\dec{\bSigma}^{-1}\dec{\bPhi}$ as the precision-weighted Gram matrix, and $\enc{\bC} \coloneqq \diag[\enc{\br}\odot(\bm{1}-\enc{\br})]$ as the posterior covariance. We weight the prior KL by $\beta>0$, with $\beta=1$ recovering the usual variational free energy. Because of the linear Gaussian assumption (\cref{eq:befond-likelihood}), the log-likelihood expectation can be evaluated exactly (\cref{app:befond-free-energy}):
\begin{align}
    \cF(\bx;\enc{\bu},\dec{\btheta})
    \;\;=\;\;{}& \frac12\be^{\top}\dec{\bSigma}^{-1}\be
    \;+\; \frac12\Tr(\dec{\bW}\enc{\bC})
    \;+\; \frac12\log\det\dec{\bSigma}\notag\\
    &\quad+\;\beta
    \sum_{i=1}^{K}\left[
        \enc{r}_i\log\frac{\enc{r}_i}{\dec{\pi}_{0,i}}
        +(1-\enc{r}_i)\log\frac{1-\enc{r}_i}{1-\dec{\pi}_{0,i}}
    \right].
    \label{eq:befond-free-energy}
\end{align}
\textbf{Dynamics.}  
Applying \hyperref[fond]{FOND's natural-gradient prescription}, $\dot{\enc{\bu}} \coloneqq -\bG_u^{-1}(\enc{\bu})\nabla_{\enc{\bu}}\cF(\enc{\bu}, \dec{\btheta})$ and $\dot{\dec{\btheta}} \coloneqq -\bG_\theta^{-1}(\dec{\btheta})\nabla_{\dec{\btheta}}\cF(\enc{\bu}, \dec{\btheta})$, with the Fisher metrics of the variational posterior and generative joint distribution, yields the following closed-form inference and learning dynamics (\cref{app:befond-flow-derivations}):
\begin{center}
\begingroup
\setlength{\tabcolsep}{4pt}
\renewcommand{\arraystretch}{1.4}
\newcommand{\befondfloweqnumber}[1]{%
  \refstepcounter{equation}\label{#1}\textup{(\theequation)}%
}
\begin{tabular}{@{}l@{\hspace{1em}}c|c|c@{}}
\toprule
\textit{Inference}
& \multicolumn{3}{c@{}}{%
  $\displaystyle
    \dot{\enc{\bu}}
    \;\;=\;
    \underbracegray{
        \dec{\bPhi}^{\top}\dec{\bSigma}^{-1}\bx
    }{\text{feedforward drive}}
    \,-\,
    \underbracegray{
        \dec{\bW} \enc{\br}
    }{\substack{\text{recurrent}\\\text{interactions}}}
    \,-\,
    \underbracegray{\tfrac{1}{2}\diag(\dec{\bW})\odot(\bm{1}-2\enc{\br})}{\text{variance penalty}}
    \,-\,
    \underbracegray{\beta(\enc{\bu}-\dec{\bu}_0)}{\text{homeostatic leak}}
  $\qquad\befondfloweqnumber{eq:befond-inference-flow}%
} \\[3pt]
\midrule
\textit{Learning}
& $\displaystyle
    \dot{\dec{\bPhi}}
    = \eta\ssub{\Phi}\bigl(\bx\enc{\br}^{\top}-\dec{\bPhi}\enc{\bE}\bigr)\dec{\bE}_p^{-1}
  $\enspace\befondfloweqnumber{eq:befond-dictionary-flow}
& $\displaystyle
    \dot{\dec{\bSigma}}
    = \eta\ssub{\Sigma}(\enc{\bV}-\dec{\bSigma})
  $\enspace\befondfloweqnumber{eq:befond-variance-flow}
& $\displaystyle
    \dot{\dec{\bpi}}_0
    = \eta_0\beta(\enc{\br}-\dec{\bpi}_0)
  $\enspace\befondfloweqnumber{eq:befond-prior-flow} \\[2pt]
\bottomrule
\end{tabular}
\endgroup
\end{center}
Here $\eta\ssub{\Phi},\eta\ssub{\Sigma},\eta_0>0$ set the learning timescales, and prior learning is expressed in mean coordinates $\dec{\bpi}_0=\sigma(\dec{\bu}_0)$. The posterior and prior second moments are:
\begin{equation}
\enc{\bE}
\;=\;
\enc{\bC} \,+\, \enc{\br}\enc{\br}^{\top}
\;\equiv\;
\diag[\enc{\br}\odot(\bm{1}-\enc{\br})] \,+\, \enc{\br}\enc{\br}^{\top}~,
\qquad
\dec{\bE}_p
\;=\;
\diag[\dec{\bpi}_0\odot(\bm{1}-\dec{\bpi}_0)] 
\,+\, \dec{\bpi}_0\dec{\bpi}_0^{\top}~,
\end{equation}
and $\enc{\bV}$ is diagonal, with $\enc{V}_{\! mm} = e_m^2+(\dec{\bPhi}\enc{\bC}\dec{\bPhi}^{\top})_{mm}$, the expected squared reconstruction error in each observed dimension. Learning holds the inferred posterior fixed. For minibatches, we average $\bx\enc{\br}^{\top}$, $\enc{\bE}$, $\enc{\bV}$, and $\enc{\br}$ before applying the learning rules. These equations keep the prior fixed during inference. Our implementation uses a rolling-reference variant (\cref{app:online-prior-working-memory}).

\paragraph{Time discretization.}
Turning continuous flows into an implementable algorithm requires time discretization, which can affect stability. We use exponential integration for both \enc{inference} \citep[ETD1;][]{cox2002exponential,hochbruck2010exponential} and \dec{learning} (including matrix-free decoder updates; \cite{almohy2011computing,trefethen2019approximation}) to preserve linear relaxation within each step. See \cref{app:numerical-implementation} for derivations and \cref{alg:fond} for the full discretized algorithm.

\section{Theoretical analyses, synthetic experiments, and ablations}
\label{sec:mechanistic-analysis}
We study \enc{inference} and \dec{learning} using theory and synthetic data with known features and activations. Ablations test lateral interactions and compare separately tuned natural and Euclidean recipes.

\subsection{\texorpdfstring{\enc{Inference}}{Inference}: \befond's recurrent term enables explaining away}
\label{sec:theory_inference}
Superposition mixes evidence from overlapping features \citep{elhage2022superposition,klindt2026unifying}, making inference of sparse codes difficult (\cref{app:inference-challenge-superposition}). Here, we show how \befond's iterative inference mitigates this interference through a mechanism known as \textit{explaining away}. As an example, wet grass could be explained by rain or a sprinkler. Observing a wet roof supports rain without changing the sprinkler's direct evidence from the grass. As rain explains
more of the wetness, the sprinkler becomes less likely.
This reduction in belief in one explanation through evidence
for another is known as \textit{explaining away} \citep{murphy2012machine}. The challenge is that the influence of one piece of evidence depends on the others. An encoder that compresses the input to one fixed projection cannot generally capture the full dependence, as the following proposition establishes.

\begin{proposition}[A limitation of single-layer inference]
\label{prop:one-layer-inference-K}
Suppose the data are generated by the Bernoulli–Gaussian model:
$$
\bx \;=\; \dec{\bPhi}\bz \,+\, \boldsymbol\epsilon~,
\qquad
z_j \overset{\mathrm{ind}}{\sim} \operatorname{Bernoulli}(\dec{\pi}_j)~,
\qquad
\boldsymbol\epsilon\sim\mathcal N(0,I)~,
$$
where $0<\dec{\pi}_j<1$ for every $j$, and the noise is independent of $\bz$.

Suppose a dictionary atom $\dec{\bphi}^{(i)}$ is neither orthogonal nor parallel to another atom $\dec{\bphi}^{(j)}$.
Then there exist no weights $\enc{\bw}_i$, bias $\enc{b}_i$,
and scalar function $f_i$ such that:
$$
f_i(\enc{\bw}_i^\top\bx+\enc{b}_i)
\,\;=\;\,
p(z_i=1\mid\bx)
\,\;=\;\,
\EE[z_i\mid\bx]~,
\qquad
\text{for every }\bx\in\mathbb R^d~.
$$
\end{proposition}

For almost
every dictionary with $K,M\ge2$, no latent falls under the exception, so for generic dictionaries a feature’s exact posterior probability cannot be computed from one fixed input projection followed by a scalar nonlinearity. The proof shows that overlapping, nonparallel atoms make the posterior log-odds gradient change direction with the input, which is impossible for a function of one fixed projection (\cref{app:proof-one-layer-inference-K}). This result generalizes \citet{oneill2025compute} and covers ReLU \citep{bricken2023monosemanticity,cunningham2024sparse}, JumpReLU (including Gemma Scope) \citep{rajamanoharan2024jumping,lieberum2024gemma}, weight-sharing gated \citep{rajamanoharan2024improving}, and fixed-threshold BatchTopK and Matryoshka BatchTopK SAEs \citep{bussmann2024batchtopk,bussmann2025learning}, all of which compute each latent from a single projection of the input (\cref{app:proof-one-layer-inference-K}). TopK SAEs fall outside this class, since whether a latent fires depends on every other pre-activation.

\befond's inference dynamics (\cref{eq:befond-inference-flow}) couple explanations through the off-diagonal entries of $-(\dec{\bPhi}^{\top}\dec{\bPhi})\enc{\br}$: positively overlapping features suppress one another's activation drive. A wet roof strengthens rain, which explains the lawn's wetness and suppresses the sprinkler. Repeated updates make this competition depend on the full observation (\cref{fig:explaining-away}).

\begin{figure}[!t]
    \centering
    \includegraphics[width=\linewidth]{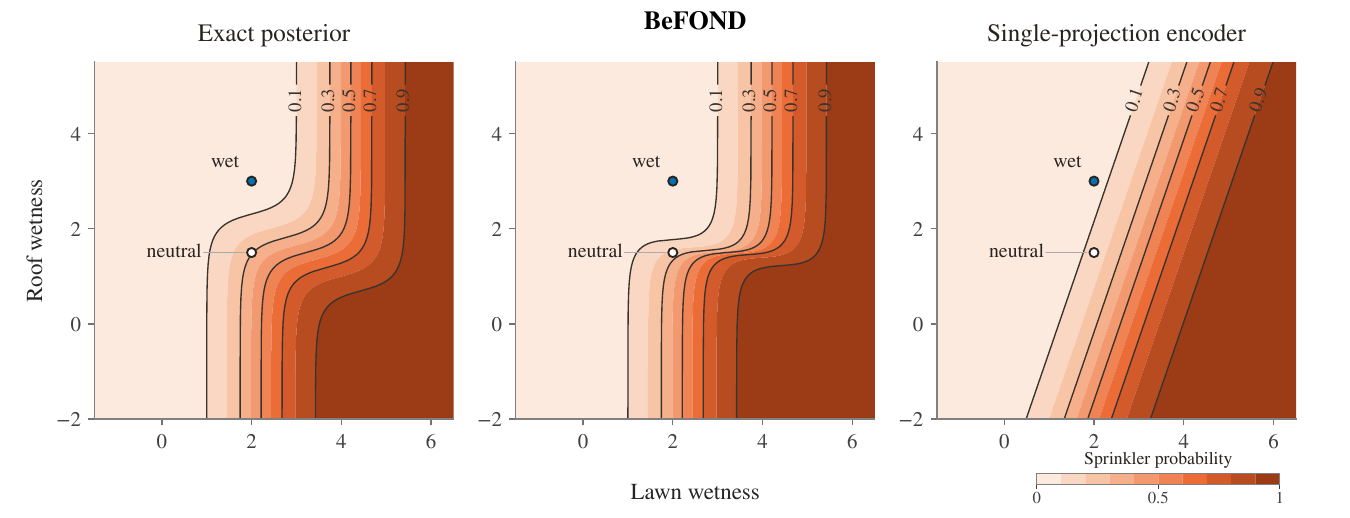}
    \caption{Heatmaps show sprinkler probabilities under exact inference, converged \befond with a fixed prior, and a numerically optimized single-projection encoder. The marked observations have identical lawn wetness but different roof evidence: neutral (white, halfway dry / wet) or wet (blue). Fixed-prior inference and encoder fitting are detailed in \cref{app:explaining-away-example}.}
\label{fig:explaining-away} 
\end{figure}

\Cref{app:proof-befond-kl} gives an approximate-inference
counterpart to \cref{prop:one-layer-inference-K}
and sufficient conditions under which \befond can achieve
strictly lower expected posterior KL than the best possible
single-projection amortized encoder. We now empirically quantify the effects of explaining away.

\paragraph{Explaining away under superposition.} 
\synthsaebench \citep{chanin2026synthsaebench} makes this problem measurable by providing synthetic data with known atoms and activations, together with feature correlations. Superposition is measured as the mean maximum (mm) absolute cosine similarity:
\begin{equation}
    \rho_\mathrm{mm}
    \;=\;
    \frac{1}{K_\star}\sum_{i=1}^{K_\star}
    \max_{j\ne i}
    \frac{
    \left|(\dec{\bphi}_{\star}^{(i)})^\top
    \dec{\bphi}_{\star}^{(j)}\right|
    }{
    \|\dec{\bphi}_{\star}^{(i)}\|\,
    \|\dec{\bphi}_{\star}^{(j)}\|
    }~,
\end{equation}
where $\dec{\bphi}_{\star}^{(i)}$ are the $K_\star$
ground-truth feature directions.

Larger $\rho_\mathrm{mm}$ means that each feature has, on average,
a more strongly overlapping alternative. Reducing the input dimension at fixed feature count increases this overlap, modeling the interference created when many features share a real network's activation space.

\begin{wrapfigure}[22]{R}{0.55\textwidth}
    \centering
    \includegraphics[width=\linewidth]{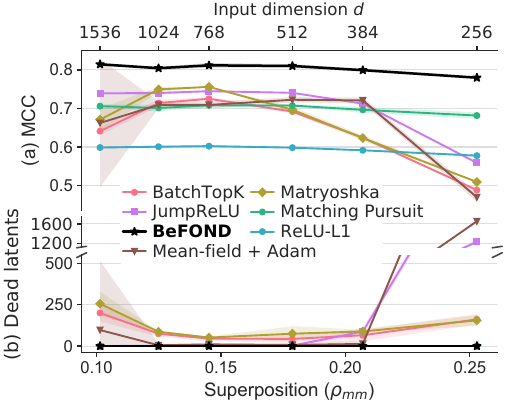}
    \caption{\befond maintains the highest recovery, and without dead latents ($K=4{,}096$). Means and 95\% confidence intervals over 3 seeds. Exact values are reported in \cref{app:superposition-results} and \cref{tab:superposition-baselines}.}
    \label{fig:superposition}
\end{wrapfigure}

As overlap increases, a strong detection score becomes less
diagnostic of a feature's presence: another feature may
explain the same evidence. Explaining away addresses this
ambiguity by discounting evidence already accounted for
by competing causes.

We measure dictionary recovery using the \textit{mean correlation coefficient} (MCC), which quantifies alignment between learned and ground-truth feature directions. \befond achieves the highest MCC throughout the sweep, declining only modestly as superposition increases, while maintaining zero dead latents (\cref{fig:superposition}).  

\paragraph{Ablating lateral interactions.}
In \cref{app:superposition-results}, we investigate the impact of ablating lateral interactions on dictionary recovery. \Cref{fig:superposition-offdiag} shows that removing lateral interactions reduces dictionary recovery as feature overlap increases. Restoring suppressive interactions substantially improves recovery, whereas excitatory-only inference activates nearly the entire dictionary without comparable recovery.

\subsection{\texorpdfstring{\dec{Learning}}{Learning}: \befond's natural gradient flow improves rare feature learning}
\label{sec:theory_learning}

Features extracted from language models vary widely in activation frequency \citep{gao2025scaling}, and larger SAE dictionaries recover increasingly rare concepts~\citep{bricken2023monosemanticity,
templeton2024scaling}. Scaling dictionary width therefore places greater demands on learning from rare examples: additional neurons must capture infrequent features to expand the model's coverage.

\paragraph{Learning an infrequently active feature.}
For an isolated feature active with probability $\rho$, averaging
over inputs scales the decoder gradient by $\rho$.
The same decoder error therefore produces a smaller update
for a rarer feature. Matched Fisher preconditioning divides
out this frequency factor.
\begin{theorem}[Exact rarity compensation]
\label{thm:exact-rarity-compensation}
For $p(\bx\mid\bz) = \cN(\bx; \dec{\bPhi}\bz,\bI)$,
fix a training minibatch and hold its inferred distributions
and prior references fixed. Let $\dec{\bPhi}^*$ minimize
$\cF$. If the averaged second moments satisfy
$\dec{\bE}_p = \enc{\bE} \succ 0$, then:
$$
\begin{aligned}
\textit{Euclidean gradient :}\;\quad
\dot{\dec{\bPhi}}
&\,\;=\;\,
-\eta\ssub{\Phi}\nabla_{\dec{\bPhi}}\cF
\,\;=\;\,
-\eta\ssub{\Phi}(\dec{\bPhi}-\dec{\bPhi}^*)\enc{\bE}~,
\\[3pt]
\textit{Natural gradient :}\;\quad
\dot{\dec{\bPhi}}
&\,\;=\;\,
-\eta\ssub{\Phi}\bG_\Phi^{-1}\nabla_{\dec{\bPhi}}\cF
\,\;=\;\,
-\eta\ssub{\Phi}(\dec{\bPhi}-\dec{\bPhi}^*)~,
\end{aligned}
$$
where $\bG_\Phi=\dec{\bE}_p\otimes\bI$ and $\eta\ssub{\Phi}>0$.

Natural flow reduces every decoder error by $e^{-\eta\ssub{\Phi}\Delta t}$ over an interval $\Delta t\ge0$. For the same fractional error reduction, Euclidean flow takes $\kappa(\enc{\bE}) \ge \rho_{\max}/\rho_{\min}$ times as long in its slowest eigendirection as in its fastest, where $\rho_i=\enc{\bE}_{ii}$ is the minibatch-averaged posterior activation probability.
\end{theorem} 

The rate difference becomes more pronounced as the frequency range
widens. If $\rho_i=c\,i^{-\alpha}$, with $\alpha>0$ and
$i=1,\ldots,K$, the slowest Euclidean direction takes at least
$K^\alpha$ times as long as the fastest to achieve the same fractional
error reduction. Under the theorem's matched-moment assumptions,
natural flow gives every direction the same rate.
For the proof and the extension to unequal prior and posterior second moments, see \cref{app:proof-exact-rarity-compensation}. \Cref{app:adam-conditioning} shows that, with known codes and frozen Adam normalization, an optimer used by many common SAE baselines \citep{bussmann2024batchtopk,rajamanoharan2024jumping}, Zipfian feature frequencies leave a local condition number of at least $K^{\alpha/2}$, whereas matched Fisher preconditioning reduces it to one.

\paragraph{Recovering rare features.}
The \textit{full-width} SynthSAEBench setting~\citep{chanin2026synthsaebench}
tests whether decoder capacity recovers the rare tail, with 16,384
learned latents for 16,384 generating features. This supplies one
decoder slot per true feature, so complete recovery requires representing
rare features alongside common ones. MCC measures recovery of the
feature directions, while Rare F1 measures detection of the $25\%$
least frequent generating features, averaging their activation F1 scores. \befond leads all metrics, exceeding the evaluated baselines
(\cref{fig:synthsae}; \Cref{tab:synthsae-final}).

\begin{figure}[t]
    \centering
    \includegraphics[width=\linewidth]{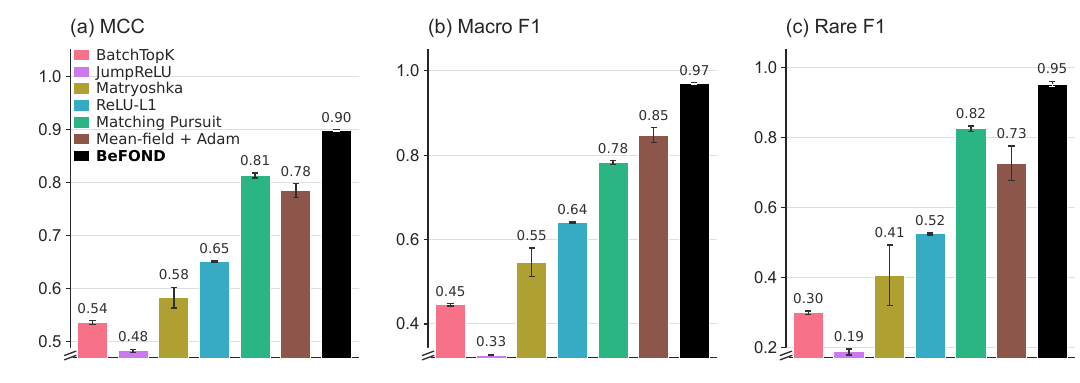} 
    \caption{\befond outperforms all evaluated baselines on dictionary recover (MCC) and feature detection (Macro F1, Rare F1) in \synthsaebench ($d=768$, $K=K_\star=16{,}384$). Rare F1 covers the least frequent quarter of features. Bars show means over 3 initialization seeds with 95\% confidence intervals, evaluated on $10^6$ held-out samples. \Cref{tab:synthsae-final} gives individual-seed results.}
    \label{fig:synthsae} 
\end{figure}

\paragraph{Ablating inference and learning geometry.}
The natural-inference/natural-learning recipe reaches stronger recovery
with fewer decoder updates than the separately tuned Euclidean variants
(\cref{fig:euclidean}; \cref{app:geometry-results}).
\Cref{app:learning-supplemental-results} reports learning curves
against the other baselines, and \Cref{tab:synthsae-final} reports
dead-latent counts alongside recovery.



\section{Evaluating \befond on language-model activations}
\label{sec:experiments}
We evaluate whether \befond's inference and learning benefits studied in \cref{sec:mechanistic-analysis} translate to improved feature quality in language-model representations. We focus on three questions: (1) whether \befond outperforms widely used SAE baselines on Gemma 2B scale activations, (2) whether these benefits increase with dictionary width, and (3) how early they emerge during training.

\paragraph{Experimental setup.}
We evaluate feature quality under a $50$M-token training budget,
alongside released Gemma Scope JumpReLU checkpoints trained on
$4$B tokens at $16$k width and $8$B tokens at widths $32$k--$524$k
\citep{lieberum2024gemma,rajamanoharan2024jumping}.
Our primary comparison trains \befond, BatchTopK
\citep{bussmann2024batchtopk}, Matryoshka BatchTopK
\citep{bussmann2025learning}, and a mean-field variational baseline
with Adam decoder learning \citep{kingma2015adam}.
These models use whitened Gemma-2-2B \citep{gemmateam2024gemma2} layer-12 residual-stream activations ($d=2{,}304$) extracted from Pile Uncopyrighted \citep{monology2023pile}, with dictionary widths from $16{,}384$ to $524{,}288$ latents. We report means and sample standard deviations over three \befond training seeds and five baseline training seeds, using $T=10$ inference steps for \befond and $T=20$ for mean-field at evaluation. Amortized baselines use the selected target sparsities in \cref{tab:gemma}. Gemma Scope checkpoints retain their original preprocessing. \Cref{app:gemma-full-results} provides further details.

\paragraph{Evaluation metrics.}
Reconstruction quality alone does not establish whether an SAE has learned useful features \citep{gao2025scaling,karvonen2025saebench}. We therefore evaluate two complementary capabilities using \saebench \citep{karvonen2025saebench}.

\textit{Sparse probing} \citep{kantamneni2025sparse} tests concept detection using a logistic probe restricted to $k=1$ or $k=5$ selected SAE features. We average task-level accuracy, ROC-AUC, and F1 across $113$ classification tasks. Accuracy measures correct predictions, ROC-AUC measures the ranking of positive versus negative examples, and F1 summarizes the balance between precision and recall.

\textit{RAVEL} \citep[Resolving Attribute–Value Entanglements in Language Models;][]{huang2024ravel} tests selective intervention on represented attributes. Its \textit{cause} score measures successful changes to a target attribute, while \textit{isolation} measures preservation of other attributes, and \textit{disentanglement} is the mean of these two scores.

Higher scores are better for all reported metrics. We report full numerical results in \cref{tab:gemma}.

\paragraph{Scaling with dictionary width.}
\befond achieves the highest mean RAVEL disentanglement at every
tested width, including comparisons with the evaluated Gemma Scope
checkpoints (\cref{fig:gemma}d).
At $524$k latents, its score is $0.827\pm0.001$, compared with
$0.748\pm0.004$ for BatchTopK, the strongest amortized baseline
trained on $50$M tokens (\cref{tab:gemma}).
This $7.9$ percentage-point advantage reflects improvements in
both cause and isolation: \befond scores $0.800$ and $0.854$,
respectively, versus $0.719$ and $0.776$ for BatchTopK.

Single-feature probing also favors \befond at large widths.
At $524$k, its accuracy is $0.789\pm0.002$, versus
$0.774\pm0.002$ for Matryoshka BatchTopK, an improvement of
$1.5$ percentage points.
Across widths from $16$k to $524$k, \befond's single-feature
accuracy increases from $0.771$ to $0.789$, and its RAVEL score
increases from $0.786$ to $0.827$ (\cref{fig:gemma}b,d).
Over the same range, BatchTopK and Matryoshka each improve by
approximately $0.6$ percentage points in single-feature accuracy.
Thus, increasing dictionary width produces larger gains for \befond than for either amortized baseline.

Differences between methods are smaller in five-feature probing accuracy, and the AUC and F1 rankings vary across widths (\cref{tab:gemma-scope-l12}).

\paragraph{High feature quality emerges early in training.}
At $524$k latents, \befond reaches the evaluated Gemma Scope
single-feature probing score by $\sim$6M training tokens (\cref{fig:gemma}a). The corresponding Gemma Scope checkpoints were trained on 8B tokens, a training-budget ratio of over 1{,}000$\times$. By 50M tokens, \befond exceeds their single-feature accuracy ($0.789$ versus $0.7735$) and RAVEL disentanglement ($0.827$ versus $0.7555$), using 160$\times$ fewer training tokens.

\paragraph{Comparison with iterative mean-field inference.}
The mean-field baseline performs iterative variational inference
and learns its decoder using Adam.
At $524$k, it reaches single-feature accuracy of
$0.751\pm0.007$ and RAVEL disentanglement of $0.685\pm0.029$,
significantly underperforming \befond on both metrics (\cref{tab:gemma}).

\section{Conclusion and Discussion}
\label{sec:discussion}
%

We formulate inference and dictionary learning as natural-gradient flows on a shared variational free energy (the FOND framework). Our theory establishes limitations of single-projection encoders and shows how Fisher preconditioning compensates for infrequent activation under frozen, matched moments. We introduce \befond, which improves synthetic feature recovery and supports increasing feature quality at larger dictionary widths on Gemma.

\paragraph{Limitations.}
Repeated inference steps and numerical decoder updates increase \befond's computational cost (\cref{app:complexity,app:training-runtime}). Future work will focus on improving computational efficiency, speeding up inference convergence by adding momentum, and extending the factorized variational family with pairwise interactions to represent dependencies between latent features.

On Gemma-2-2B layer-12 activations, \befond improves concept detection and selective intervention. Its gains increase with dictionary width while the evaluated baselines largely plateau. Together with recurrent explaining away and improved rare-feature recovery, these results show the practical value of taking the statistical geometry of inference and learning seriously.

The broader lesson is how adaptation dynamics should be derived:

\boxit[title = {\textit{Physics of Adaptation}: deriving covariant adaptation dynamics}, halign = center, halign title = center]{
\enc{Inference} and \dec{learning} dynamics should respect the intrinsic geometry\\and symmetries of the underlying model
(\cref{app:physics_of_adaptation}).
}

\section*{Code and data}
Our code, data, and model checkpoints are available here:  

\href{https://github.com/hadivafaii/BeFOND}{\color{purple}https://github.com/hadivafaii/BeFOND}

\subsection*{AI use statement}

We used generative AI tools to assist with mathematical derivations, software implementation and debugging, experimental analysis and visualization, literature searches, and manuscript drafting and editing. The authors are responsible for the final text, results, and claims.

\ificlrfinal

%

%

\fi


\bibliography{refs}
\bibliographystyle{iclr2027_conference}


\clearpage
\appendix



\addtocontents{toc}{\protect\setcounter{tocdepth}{2}}

\renewcommand{\contentsname}{%
    \texorpdfstring{%
        {\fontsize{17pt}{20pt}\selectfont
        Appendix: \papertitle}%
    }{Appendix: \papertitle}%
    \bigskip
}

{
\hypersetup{hidelinks}
\tableofcontents
}

\clearpage
\section{Extended background}
\label{app:extended-background} 

\subsection{Superposition and inference}
\label{app:inference-challenge-superposition}
Inferring a sparse code can be difficult even with a perfectly
learned dictionary. To isolate the inference problem, consider
the noiseless generative process $\bx = \bPhi^*\bz$, where the ground truth dictionary $\bPhi^* \in \RR^{M \times K}$ has unit-norm columns, and suppose we have learned the true dictionary $\dec{\bPhi} = \bPhi^*$. Projecting the observation onto each dictionary atom gives:
\begin{equation}
    \bigl(\dec{\bPhi}^{\top}\bx\bigr)_i
    \;=\;
    z_i
    \,+\,
    \underbrace{
        \sum_{j\neq i}
        \bigl(\dec{\bphi}^{(i)}\bigr)^{\top}
        \dec{\bphi}^{(j)} z_j
    }_{\text{interference from other features}}~.
\end{equation}
The projection combines the feature's own activity with
contributions from other active features, weighted by their
overlap with its dictionary atom. Even when only one feature
$j$ is active, an inactive feature $i$ can receive a nonzero
projection. When several features are active, their contributions
can reinforce or cancel one another, making projection scores
less reliable indicators of individual feature activity \citep{elhage2022superposition}.

For an orthonormal dictionary, all interference terms vanish,
and inference reduces to a single projection:
$\enc{\bz} = \dec{\bPhi}^{\top}\bx$.
However, an overcomplete dictionary ($K > M$) cannot have
mutually orthogonal columns, so this interference cannot be
eliminated for all feature pairs.

Therefore, inference should distinguish a feature's own contribution from evidence attributable to other active features. This motivates \textit{explaining away}: support for a candidate feature decreases when other features already explain the evidence that initially supported it \citep{murphy2012machine}. Such interactions allow inference to account for overlap between dictionary atoms
rather than interpreting each projection independently \citep{rozell2008sparse}.

\subsection{Iterative sparse encoders.}\label{app:iterative sparse} 
Iterative encoders refine sparse representations through successive updates rather than a single feed-forward prediction. ISTA, FISTA, and soft-thresholding LCA optimize $\ell_1$-regularized reconstruction \citep{daubechies2004iterative,beck2009fast,rozell2008sparse}, with LCA implementing lateral competition. LISTA learns a finite-step approximation \citep{gregor2010learning}, while matching pursuit (MP) greedily selects atoms to explain the remaining residual~\citep{mallat1993matching}.

With a fixed pretrained dictionary, better inference remains limited by the learned feature directions. Additional ISTA refinement improves reconstruction with little change in BatchTopK probe scores; non-amortized ISTA worsens splitting for Gemma TopK and increases dead latents for Pythia Standard SAEs~\citep{sun2026price}. FISTA applied to fixed SAE dictionaries also fails to close the feature-recovery gap in controlled experiments, implicating the learned dictionary as a bottleneck~\citep{pacela2026stop}. On Pythia-160M, however, Sun et al. report improved targeted-probe selectivity and generally higher generative intervention scores with semi- and non-amortized inference \citep[Appendix~C]{sun2026price}. 

Approaches that also learn the dictionary have shown promising results: MP-SAE captures hierarchical features through residual-guided selection and supports different inference-time sparsity levels~\citep{costa2025flat}, while LocA-SAE seeks similar feature-quality benefits through grouped amortized encoders rather than explicit iterative refinement~\citep{sun2026price}. But better reconstruction has not consistently translated into better interpretability tools. On Gemma-2-2B, MP-SAEs improve reconstruction but underperform the tested feed-forward baselines on sparse probing, with substantially longer training. The evaluator therefore recommends against them as a first choice for practical interpretability~\citep{chanin2025training}. LocA-SAE eliminates dead latents on Gemma but scores below the Standard SAE on the same proxy probes, so improved utilization need not yield more predictive features~\citep{sun2026price}. Across these evaluations, reconstruction, probing, and intervention metrics need not improve together, and the practical benefit of iterative refinement in these prior work remains task-dependent. 

\subsection{Variational inference and the free energy objective}
\label{app:free-energy-derivation}

We provide a pedagogical introduction to variational inference and derive the free energy objective. In this derivation, we distinguish two KL divergences: one measures how well we approximate the true posterior, while the other measures departure from the prior. They play different roles in the same exact identity.

Finally, we explain why minimizing free energy improves inference using a few lines of algebra.

\paragraph{Variational posterior inference.}
Consider an observation $\bx\in\RR^M$ and $K$-dimensional latent variables $\bz$. For Bernoulli latents, $\bz \in \{0, 1\}^K$. A generative model specifies a prior $p_{\dec{\btheta}}(\bz)$ and a likelihood $p_{\dec{\btheta}}(\bx\mid\bz)$, whose product is the joint distribution. Bayes' rule gives:
\begin{equation}\label{eq:fe-bayes}
    p_{\dec{\btheta}}(\bz\mid\bx)
    \;=\;
    \frac{p_{\dec{\btheta}}(\bx,\bz)}{p_{\dec{\btheta}}(\bx)}~,
    \qquad
    p_{\dec{\btheta}}(\bx)
    \;=\;
    \sum_{\bz\in\{0,1\}^K}
    p_{\dec{\btheta}}(\bx\mid\bz)\,p_{\dec{\btheta}}(\bz)~.
\end{equation}
The denominator is the \textit{model evidence}: the probability density of the observation after summing over all latent explanations. This sum contains $2^K$ terms, making direct evaluation generally infeasible for large $K$. The \say{true posterior} here means the exact posterior under the chosen generative model.

Variational inference is a powerful approximation method that replaces this computation with an optimization problem \citep{jordan1999introduction, blei2017variational}. In variational inference, we introduce a tractable approximate posterior $q_{\enc{\bu}}(\bz\mid\bx)$ and adjust its variational parameters $\enc{\bu}$ to minimize
\begin{equation}\label{eq:fe-inference-kl}
    \kl[\Big]{q_{\enc{\bu}}(\bz\mid\bx)}{p_{\dec{\btheta}}(\bz\mid\bx)}
    \;\coloneqq\;
    \sum_{\bz}q_{\enc{\bu}}(\bz\mid\bx)
    \log\frac{q_{\enc{\bu}}(\bz\mid\bx)}{p_{\dec{\btheta}}(\bz\mid\bx)}~.
\end{equation}
This KL divergence is nonnegative and vanishes exactly when the two distributions agree \citep{blei2017variational}. But its value still depends on the intractable posterior, so we can't directly compute and optimize it. The next step isolates that dependence, leaving an objective we can optimize without evaluating the evidence.

\paragraph{Free-energy derivation.}
Fix $\bx$ and $\dec{\btheta}$ throughout the derivation. To keep the algebra readable, write $\enc{q}(\bz)\equiv q_{\enc{\bu}}(\bz\mid\bx)$ and suppress the parameter subscript on $p$. All expectations below are over $\bz\sim \enc{q}$. For finite Bernoulli logits and a positive Gaussian covariance, the distributions are strictly positive, so the following ratios are well defined.

Since $\enc{q}$ is normalized, $\sum_{\bz}\enc{q}(\bz)=1$. Multiply the log evidence by this sum, apply Bayes' rule, and insert $\enc{q}(\bz)/\enc{q}(\bz)=1$:
\begin{equation}\label{eq:fe-evidence-derivation}
\begin{aligned}
    \log p(\bx)
    &\;=\;
    \log p(\bx)\underbracegray{\sum_{\bz}\enc{q}(\bz)}{=\,1}
    \;=\;\expect{q}{\log p(\bx)}
    \\[2mm]
    &\;=\;
    \expect{q}{\log\frac{p(\bx,\bz)}{p(\bz\mid\bx)}}
    \\[2mm]
    &\;=\;
    \expect{q}{\log\left[
        \frac{p(\bx,\bz)}{p(\bz\mid\bx)}
        \frac{\enc{q}(\bz)}{\enc{q}(\bz)}\right]}
    \\[2mm]
    &\;=\;
    \expect{q}{\log\frac{p(\bx,\bz)}{\enc{q}(\bz)}}
    \;+\;
    \expect{q}{\log\frac{\enc{q}(\bz)}{p(\bz\mid\bx)}}
    \\[2mm]
    &\;=\;
    \underbracegray{\expect{q}{\log\frac{p(\bx,\bz)}{\enc{q}(\bz)}}}
        {\mathrm{ELBO}\;=\;-\cF}
    \;+\;\kl{\enc{q}(\bz)}{p(\bz\mid\bx)}~.
\end{aligned}
\end{equation}
All we did was multiply by one twice, then rearrange logarithms. We didn't have to invoke Jensen's inequality as it is typically done. Instead, we have an exact decomposition of the log evidence. Nonnegativity of the final KL gives the \textit{evidence lower bound}, $\mathrm{ELBO}\leq\log p(\bx)$. Its negative is the variational free energy, $\cF\equiv-\mathrm{ELBO}$ \citep{neal1998view,kingma2014auto}.

\paragraph{Posterior and prior KL divergences.}
Expanding $p_\dec{\btheta}(\bx,\bz) = p_\dec{\btheta}(\bx\mid\bz)p_\dec{\btheta}(\bz)$ inside the ELBO reveals a second KL \citep{kingma2014auto}. With $\cF\equiv\cF(\bx;\enc{\bu},\dec{\btheta})$, we have:
\begin{equation}\label{eq:fe-two-kls}
\boxed[7pt][0.6pt]{
\begin{aligned}
    \underbracegray{\log p_{\dec{\btheta}}(\bx)}{\text{log evidence}}
    &\;\;=\;\;
    \underbracegray{-\cF}{\mathrm{ELBO}}
    \;+\;
    \underbracegray{
        \kl{q_{\enc{\bu}}(\bz\mid\bx)}{p_{\dec{\btheta}}(\bz\mid\bx)}
    }{\substack{\text{inference KL: original objective}\\\text{(gap to the true posterior)}}}
    \\[4mm]
    \cF
    &\;\;=\;\;
    \underbracegray{
        \expect{q_{\enc{\bu}}}{-\log p_{\dec{\btheta}}(\bx\mid\bz)}
    }{\text{reconstruction term (accuracy)}}
    \;+\;
    \underbracegray{
        \kl{q_{\enc{\bu}}(\bz\mid\bx)}{p_{\dec{\btheta}}(\bz)}
    }{\text{prior KL (coding rate)}}~.
\end{aligned}
}
\end{equation}
The \textit{inference KL} compares the approximate posterior with the true posterior, $p_{\dec{\btheta}}(\bz\mid\bx)$. This is the original objective in \cref{eq:fe-inference-kl}. The \textit{prior KL}, by contrast, compares the approximate posterior with $p_{\dec{\btheta}}(\bz)$, before conditioning on the observation. It is a component of free energy, and can be interpreted as a \textit{coding rate}: the cost of representing the input-dependent distribution relative to the prior \citep{alemi2018fixing}. We use natural logarithms, so this cost is measured in nats.

These quantities need not decrease together. If inference starts at the prior, the prior KL is initially zero. Explaining an informative observation generally requires moving away from that prior, increasing the coding rate while reducing the inference KL. Free energy balances this cost against the improvement in expected log likelihood. Minimizing the prior KL alone would simply return the prior and ignore the observation.

\paragraph{Free energy and posterior approximation.}
Rearranging the first line of \cref{eq:fe-two-kls} yields
\begin{equation}\label{eq:fe-posterior-identity}
    \cF(\bx;\enc{\bu},\dec{\btheta})
    \;=\;
    \kl{q_{\enc{\bu}}(\bz\mid\bx)}{p_{\dec{\btheta}}(\bz\mid\bx)}
    \;-\;\log p_{\dec{\btheta}}(\bx)~.
\end{equation}
The evidence depends on the generative model and observation, but not on the variational parameters. Therefore, with $\bx$ and $\dec{\btheta}$ fixed, every decrease in $\cF$ is an equal decrease in the inference KL \citep{blei2017variational}. Equivalently, we have:
\begin{equation}\label{eq:fe-gradient-equivalence}
    \nabla_{\enc{\bu}}\cF
    \;=\;
    \nabla_{\enc{\bu}}
    \kl{q_{\enc{\bu}}(\bz\mid\bx)}{p_{\dec{\btheta}}(\bz\mid\bx)}~.
\end{equation}
This is the sense in which free energy minimization improves inference. It holds regardless of whether the variational parameters are produced by an amortized encoder or refined iteratively. The bound becomes tight, $\cF=-\log p_{\dec{\btheta}}(\bx)$, only when the approximate posterior equals the true posterior. A restricted variational family may not contain that posterior \citep{wainwright2008graphical,blei2017variational}; for example, a factorized Bernoulli distribution cannot represent all the dependencies induced by overlapping dictionary atoms. Optimization can reduce the KL without eliminating this approximation gap.

\subsection{Trust regions and natural-gradient inference}
\label{app:trust-region}

Inference moves the approximate posterior $q_{\enc{\bu}}$ to reduce free energy. But specifying the objective is only part of specifying the dynamics. We also need to decide how far beliefs should move at each step, and what counts as a small change. This choice of distance is what takes us from ordinary gradient descent to the Fisher geometry used by FOND \citep{vafaii2025brain}.

Throughout this derivation, we hold the observation $\bx$, generative parameters $\dec{\btheta}$, and prior reference fixed. We write the objective as $\cF(\enc{\bu})$ and suppress conditioning on $\bx$ in $q_{\enc{\bu}}$.

\paragraph{From a gradient to a step.}
At the current iterate $\enc{\bu}_t$, define $\bg_t\coloneqq\nabla_{\enc{\bu}}\cF(\enc{\bu}_t)$. For a small displacement $\Delta\enc{\bu}$, the gradient predicts the change in free energy:
\begin{equation}\label{eq:tr-linearization}
    \cF(\enc{\bu}_t+\Delta\enc{\bu})
    \;=\;
    \cF(\enc{\bu}_t)
    \;+\;\bg_t^{\top}\Delta\enc{\bu}
    \;+\;o\!\left(\norm{\Delta\enc{\bu}}_2\right)~.
\end{equation}
There is a catch: if $\bg_t\neq\bm{0}$, the linear term keeps decreasing as we move farther along $-\bg_t$. Taken literally, it would tell us to take an infinitely large step. Yet, the prediction in \cref{eq:tr-linearization} is only useful nearby. We can keep the step finite by attaching a quadratic cost to movement:
\begin{equation}\label{eq:tr-euclidean-penalty}
    \Delta\enc{\bu}
    \;=\;
    \argmin{\bdelta}
    \left\{
        \underbracegray{\bg_t^{\top}\bdelta}{\substack{\text{predicted}\\\text{change in }\cF}}
        \;+\;
        \underbracegray{\frac{1}{2\eta}\norm{\bdelta}_2^2}{\text{cost of the step}}
    \right\}
    \;=\;
    -\eta\,\bg_t~,
    \qquad \eta\;>\;0~.
\end{equation}
The last equality follows by differentiating with respect to $\bdelta$ and setting $\bg_t+\bdelta/\eta=\bm{0}$. This is ordinary gradient descent: a balance between the predicted improvement and the cost of moving away from the current iterate \citep{martens2020new}. Increasing $\eta$ makes movement cheaper, allowing a longer step. The quadratic term belongs to the update rule; the approximation to $\cF$ itself remains first order.

\paragraph{Trust-region formulation.}
We can express the same tradeoff as a movement budget \citep[Sec.~9.4]{boyd2004convex}. Allow a step only within a ball around the current iterate, and choose the one that most reduces the linearized objective:
\begin{equation}\label{eq:tr-euclidean-constraint}
    \minimize{\bdelta}\;\bg_t^{\top}\bdelta
    \qquad\text{subject to}\qquad
    \norm{\bdelta}_2^2\;\leq\;\epsilon~.
\end{equation}
The radius $\sqrt{\epsilon}$ sets how far we are willing to trust the local prediction. For $\epsilon>0$ and $\bg_t\neq\bm{0}$, the optimum uses the full budget, giving $\Delta\enc{\bu}=-\sqrt{\epsilon}\,\bg_t/\norm{\bg_t}_2$. It agrees with \cref{eq:tr-euclidean-penalty} when $\epsilon=\eta^2\norm{\bg_t}_2^2$. Thus, a fixed step size $\eta$ lets a large gradient spend more of the movement budget, while a fixed $\epsilon$ enforces the same radius at each step.

\paragraph{Changing the geometry of the step.}
Euclidean distance measures every direction with the same ruler. More generally, a positive-definite matrix $\bM_t$ assigns the squared length $\bdelta^{\top}\bM_t\bdelta$ to a displacement. We evaluate this metric at the current iterate and hold it fixed while choosing the step. Replacing the Euclidean penalty gives:
\begin{equation}\label{eq:tr-metric-penalty}
    \Delta\enc{\bu}
    \;=\;
    \argmin{\bdelta}
    \left\{
        \bg_t^{\top}\bdelta
        \;+\;\frac{1}{2\eta}\bdelta^{\top}\bM_t\bdelta
    \right\}
    \;=\;-\eta\,\bM_t^{-1}\bg_t~.
\end{equation}
The corresponding constraint, $\bdelta^{\top}\bM_t\bdelta\leq\epsilon$, has solution
\begin{equation}\label{eq:tr-metric-constraint}
    \Delta\enc{\bu}
    \;=\;
    -\sqrt{\frac{\epsilon}{\bg_t^{\top}\bM_t^{-1}\bg_t}}\,
    \bM_t^{-1}\bg_t~,
    \qquad \bg_t\;\neq\;\bm{0}~.
\end{equation}
The ball has become an ellipsoid. Along an eigenvector with a large metric eigenvalue, movement costs more, so the same gradient component produces a smaller update. This is why ``steepest descent'' needs a geometry: the best direction depends on how we measure the cost of getting there \citep{amari1998natural,boyd2004convex}.

\paragraph{Measuring changes in beliefs.}
So far, we have measured movement in parameter space. But $\enc{\bu}$ is a description of a belief, and equal parameter changes can mean very different belief changes. For example, a small change in a Bernoulli logit shifts the activation probability most strongly near $1/2$, and much less near $0$ or $1$. A unit of logit displacement therefore has no fixed meaning as a change in belief. We should price the step by how much it changes the distribution it represents.

KL divergence gives us such a measure:
\begin{equation}\label{eq:tr-belief-kl}
    \kl{q_{\enc{\bu}+\Delta\enc{\bu}}}{q_{\enc{\bu}}}
    \;=\;
    \expect{q_{\enc{\bu}+\Delta\enc{\bu}}}{
        \log\frac{q_{\enc{\bu}+\Delta\enc{\bu}}(\bz)}{q_{\enc{\bu}}(\bz)}
    }~.
\end{equation}
\citet{hobson1969new} characterized KL as an information measure for finite-outcome belief revision; \cref{app:physics_of_adaptation} develops this motivation further. Here, it measures the change between successive approximate posteriors. The prior KL in free energy has a different job: it helps define the objective by comparing the posterior with its prior reference (\cref{app:free-energy-derivation}). The stepwise KL tells us how much we change our mind in one update. Many small updates can still carry beliefs far from the prior.

\paragraph{From local KL divergence to the Fisher metric.}
For a smooth family with common positive support and sufficient regularity to differentiate under the expectation, define the score $\bs_{\enc{\bu}}(\bz)\coloneqq\nabla_{\enc{\bu}}\log q_{\enc{\bu}}(\bz)$. Normalization gives $\expect{q_{\enc{\bu}}}{\bs_{\enc{\bu}}}=\bm{0}$. When the two distributions coincide, the KL divergence and its first derivative vanish. Its leading contribution is therefore quadratic, and its Hessian is the Fisher information:
\begin{align}
    \bG_u(\enc{\bu})
    \;&=\;
    \expect{q_{\enc{\bu}}}{
        \bs_{\enc{\bu}}(\bz)\bs_{\enc{\bu}}(\bz)^{\top}
    }~,
    \label{eq:tr-fisher}
    \\[4pt]
    \kl{q_{\enc{\bu}+\Delta\enc{\bu}}}{q_{\enc{\bu}}}
    \;&=\;
    \frac{1}{2}\,\Delta\enc{\bu}^{\top}\bG_u(\enc{\bu})\Delta\enc{\bu}
    \;+\;o\!\left(\norm{\Delta\enc{\bu}}_2^2\right)~.
    \label{eq:tr-kl-expansion}
\end{align}
The Fisher metric measures how sensitive the distribution is to movement in each direction. A large Fisher length means the step makes the old and new beliefs easier to distinguish, spending more of the same KL budget. Reversing the KL arguments gives the same quadratic term \citep{martens2020new}: KL is asymmetric beyond this local approximation, but both orientations supply the same local ruler.

There is also a statistical reason to favor this ruler. A sufficient statistic retains all the information about the parameter; replacing the full observation by that statistic should preserve statistical lengths. For regular minimal exponential families, Dowty's Chentsov-type theorem makes this requirement precise. Together with the appropriate scaling for independent replication and additional compatibility and continuity assumptions, preservation by canonical sufficient statistics characterizes the Fisher metric up to a positive constant \citep{dowty2018chentsov}.

\paragraph{The natural-gradient update.}
Set $\bM_t=\bG_u(\enc{\bu}_t)$ in \cref{eq:tr-metric-penalty}. To leading order, the movement cost is now the KL divergence divided by $\eta$. Balancing this cost against the predicted decrease in free energy yields:
\begin{equation}\label{eq:tr-natural-step}
\boxed{
    \enc{\bu}_{t+1}
    \;=\;
    \enc{\bu}_t
    \;-\;\eta\,\bG_u(\enc{\bu}_t)^{-1}
    \nabla_{\enc{\bu}}\cF(\enc{\bu}_t)~.
}
\end{equation}
This is natural-gradient descent wherever the Fisher is nonsingular \citep{amari1998natural,martens2020new}. Constraining $\kl{q_{\enc{\bu}+\Delta\enc{\bu}}}{q_{\enc{\bu}}}\leq\epsilon/2$ and using its quadratic approximation gives the same direction, with the budget-dependent normalization in \cref{eq:tr-metric-constraint}.

For the continuous flow $\dot{\enc{\bu}}=-\bG_u^{-1}\nabla_{\enc{\bu}}\cF$, the rate of change of free energy is:
\begin{equation}\label{eq:tr-flow-descent}
    \deriv[\cF]{t}
    \;=\;
    -\big(\nabla_{\enc{\bu}}\cF\big)^{\top}
    \bG_u^{-1}\nabla_{\enc{\bu}}\cF
    \;\leq\;0~.
\end{equation}
Wherever the gradient is nonzero, free energy strictly decreases along the flow. After discretization, the step size still matters: a large jump can leave the region where the local prediction is accurate, so an arbitrary finite update need not decrease $\cF$.

\paragraph{A Gaussian inference example.}
To see what changing the ruler buys us, consider a Gaussian target posterior $p(\bz\mid\bx)=\cN(\bz;\bmu,\bS)$ and a variational family $q_{\enc{\bu}}(\bz)=\cN(\bz;\enc{\bu},\bS)$, with fixed $\bS\succ0$. In this illustration, $\enc{\bu}$ denotes the posterior mean, rather than the Bernoulli logits used in \befond. For the ordinary variational free energy ($\beta=1$), \cref{eq:fe-posterior-identity} gives:
\begin{equation}\label{eq:tr-gaussian-objective}
    \cF(\enc{\bu})-\cF(\bmu)
    \;=\;
    \kl{q_{\enc{\bu}}}{p(\cdot\mid\bx)}
    \;=\;
    \frac{1}{2}\,(\enc{\bu}-\bmu)^{\top}\bS^{-1}(\enc{\bu}-\bmu)~.
\end{equation}
The score is $\bS^{-1}(\bz-\enc{\bu})$, giving $\bG_u=\bS^{-1}$. Hence:
\begin{equation}\label{eq:tr-gaussian-gradients}
\begin{aligned}
    \nabla_{\enc{\bu}}\cF
    \;&=\;\bS^{-1}(\enc{\bu}-\bmu)~,
    \\[3pt]
    \bG_u^{-1}\nabla_{\enc{\bu}}\cF
    \;&=\;\enc{\bu}-\bmu~.
\end{aligned}
\end{equation}
Suppose the target posterior is very certain along one direction, but uncertain along another. Along an eigenvector of $\bS$ with eigenvalue $\lambda_j$, a Euclidean step multiplies the mean error by $1-\eta/\lambda_j$. The narrow, certain direction forces a small $\eta$ to avoid overshooting; with that same step size, the broad direction barely moves. Convergence in all directions requires $0<\eta<2\lambda_{\min}(\bS)$.

The natural-gradient step instead removes the precision factor $\bS^{-1}$: every component of the mean error contracts by $1-\eta$. It corrects the same fraction of the remaining error in each direction, with convergence for $0<\eta<2$. For this mean-only quadratic problem, $\eta=1$ reaches the target mean in one step.

\paragraph{Reparameterization invariance.}
The belief itself should not care whether we describe it by a probability, a logit, or some other smooth coordinate. Let $\enc{\bv}=f(\enc{\bu})$ be a smooth invertible reparameterization, with Jacobian $\bJ=\pderiv[\enc{\bv}]{\enc{\bu}}$. The chain rule and the score definition give:
\begin{equation}\label{eq:tr-coordinate-transform}
\begin{aligned}
    \nabla_{\enc{\bv}}\cF
    \;&=\;\bJ^{-\top}\nabla_{\enc{\bu}}\cF~,
    \\[3pt]
    \bG_v
    \;&=\;\bJ^{-\top}\bG_u\bJ^{-1}~,
    \\[3pt]
    \bG_v^{-1}\nabla_{\enc{\bv}}\cF
    \;&=\;\bJ\bG_u^{-1}\nabla_{\enc{\bu}}\cF~.
\end{aligned}
\end{equation}
The gradient and metric change together, so the natural-gradient velocity obeys $\dot{\enc{\bv}}=\bJ\dot{\enc{\bu}}$, exactly as a velocity should. Wherever the flow is well defined, both coordinate systems describe the same evolution of distributions \citep{amari1998natural,martens2020new}. This is the optimization counterpart of the covariance principle developed in \cref{app:physics_of_adaptation}.

This invariance is a property of the continuous dynamics. A finite Euler step preserves it exactly under affine coordinate changes, but generally only to first order under nonlinear ones \citep{martens2020new}. The geometry fixes the direction of motion, and discretization determines how faithfully our finite steps follow it.

\subsection{``Physics of Adaptation'': deriving covariant inference and learning rules}
\label{app:physics_of_adaptation}

In this section, we take a physics-inspired approach to dynamics and motivate the use of natural gradients through this angle. Specifically, we argue that the dynamics of inference and learning can be derived from \textit{stationary-action principles} \citep{landau1975classical}, supplemented by dissipative forces. We emphasize how the underlying geometry of the system must constrain our choice of the kinetic Lagrangian. For beliefs existing on statistical manifolds, we will introduce the \textit{kinetic energy of beliefs} as a natural choice, and use it to derive natural gradient descent in the overdamped (massless) limit of belief dynamics.

\paragraph{The principle of stationary action.}
Physics is fundamentally about dynamics. The most successful theories in physics---including Newtonian mechanics, special and general relativity, electromagnetism, and quantum field theory---share two core attributes. They all: (1) identify some dynamical objects of interest, and (2) derive their dynamics via the \textit{principle of stationary action}:

\boxit[title = {\textit{Principle of stationary action}: The closest physics has come to a theory of everything}]{
Physical systems evolve to extremize \textit{action} $S$ (time integral of the \textit{Lagrangian} $L$):
\begin{equation}
S \;= \int \! L \, dt, \quad \xrightarrow{\,\substack{\text{real physical systems follow}\\\text{trajectories that extremize action}}\;} \quad \boxed{\delta S = 0}
\end{equation}
\vspace{-3mm}
}

\paragraph{Euler--Lagrange equations.}
Extremizing the action leads directly to the Euler--Lagrange equations:
\begin{equation}\label{eq:euler-lagrange}
    \delta S = 0 \quad \Longrightarrow \quad
    \frac{d}{dt} \left( \frac{\partial L}{\partial \dot{\bx}} \right) - \frac{\partial L}{\partial \bx} \;=\; \bm{0}~.
\end{equation}

Here, $\bx$ is position, $\dot{\bx} = d\bx/dt$ is velocity, and $L$ is the \textit{Lagrangian}. In classical mechanics, the Lagrangian is simply kinetic minus potential energy, $L = T - V$, and the kinetic energy is given by $T_\text{classical} = \tfrac{1}{2}m \norm{\dot{\bx}}_2^2$, where m is the rest mass.

As an illustrative example, apply \cref{eq:euler-lagrange} to $L_\text{classical}(\bx, \dot{\bx}) = T_\text{classical}(\dot{\bx}) - V(\bx)$:
\begin{equation}
\begin{aligned}
    \frac{d}{dt} \left( \frac{\partial L_\text{classical}}{\partial \dot{\bx}} \right) - \frac{\partial L_\text{classical}}{\partial \bx}
    \;\;&=\;\;
    \frac{d}{dt} \left( \frac{\partial T_\text{classical}}{\partial \dot{\bx}} \right) + \frac{\partial V(\bx)}{\partial \bx}
    \\[3pt]
    \;\;&=\;\;
    \frac{d}{dt} \left( m \dot{\bx} \right) + \nabla V
    \;\,=\,\;
    m \ddot{\bx} + \nabla V
    \;\,=\;\,
    \bm{0}~.
\end{aligned}
\end{equation}

Defining force $\bF \coloneqq -\nabla V$ and acceleration $\ba \coloneqq \ddot{\bx}$, we recover Newton's second law, $\bF = m \ba$.

\paragraph{Generalized Euler--Lagrange equations.}
\Cref{eq:euler-lagrange} describes the dynamics when all forces are represented in the Lagrangian. To include dissipation and other non-conservative forces not represented in $L$, we add a generalized force vector $\bQ$. This gives:
\begin{equation}
    \frac{d}{dt} \left( \frac{\partial L}{\partial \dot{\bx}} \right) - \frac{\partial L}{\partial \bx} \;=\; \bQ~.
\end{equation}

Defining the canonical momentum $\bp \coloneqq \partial L / \partial \dot{\bx}$, the generalized Euler-Lagrange equation is:
\begin{equation}\label{eq:generalized-euler-lagrange}
\boxed{
    \frac{d}{dt}\, \bp \;-\; \frac{\partial L}{\partial \bx}
    \;=\; \bQ~.
}
\end{equation}

\paragraph{Adaptation dynamics.}
The dynamics of various physical systems can be derived from \cref{eq:generalized-euler-lagrange}, by making different choices for the Lagrangian $L$ and the non-conservative forces $\bQ$. We argue that the \textit{dynamics of adaptation} is no different. We treat the belief state $\enc{\bu}(t)$ as the position of a fictitious particle moving in a potential $V(\enc{\bu})$, which we identify as the loss function. The same picture applies to the generative parameters $\dec{\btheta}(t)$ during learning with potential $V(\dec{\btheta})$.

FOND unifies both processes as variational free energy minimization, giving $\cF(\enc{\bu}, \dec{\btheta}) \equiv V(\enc{\bu}, \dec{\btheta})$ as one shared potential energy: inference changes $\enc{\bu}$ with $\dec{\btheta}$ held fixed, while learning changes $\dec{\btheta}$ with $\enc{\bu}$ held fixed. This leaves two modeling choices open: the kinetic part of the Lagrangian $T$, and the dissipation force $\bQ$. Different choices lead to different inference and learning dynamics and, after discretization, different optimization algorithms, as we demonstrate below.

\paragraph{The principle of general covariance.}
With the potential energy fixed by the loss, let us turn to the kinetic part of the Lagrangian. How should we choose it? Should we just do a bunch of trial and error? Or is there a principled way to approach this? Physics gives us a basic requirement:

\boxit[title = {The principle of general covariance}, halign = {center}, halign title = {center}]{\sayit{The general laws of nature are to be expressed by equations which hold good
\\[2pt]
for all systems of co-ordinates, that is, are co-variant with respect to any
\\[2pt]
substitutions whatever (generally co-variant).} --- \citet{einstein1916grundlage}}

Imagine tossing a ball across a room. If two people track its motion---one standing upright and one tilted sideways---they will record different numbers for its horizontal and vertical motion. Yet, both of them watch the ball land in the exact same spot. The ball has no idea about the observers watching it; it follows a physical path that does not change simply because different observers choose to describe it in different ways.

More concretely, consider Newton's second law, $\bF = m\ba$. When you change your coordinate system, the actual numerical values of force and acceleration change, but the mathematical form of the law does not: $\bF = m\ba$ becomes $\bF' = m\ba'$ because both sides ``co-vary'' in lockstep under the transformation. In this sense, Newton's law describes a \sayit{co-variant} dynamical rule under Galilean transformations. If rotating or shifting our reference frame altered where the ball landed, the equation wouldn't be describing the ball at all---it would merely be describing our choice of coordinates.

\paragraph{Descartes, geometry, and algebra.}
This idea is not new. It goes back to Descartes' link between geometry and algebra: choose a coordinate system, and a point becomes a tuple of numbers. We have done this so often that it is easy to forget what comes with it: redundancy in the description. The origin, the orientation of the axes, and the units are all choices of the modeler. The same physical trajectory admits many different numerical descriptions. Covariance requires the equations to transform consistently between them, so that they all predict the same physical motion.

\paragraph{The principle of general covariance for inference and learning.}
Once again, we argue that the dynamics of adaptation are no different. A Bernoulli belief can be described by its activation probability or by its logit. These describe exactly the same distribution. Why should the learner's response to evidence change just because we switched between these two descriptions? We require the \textit{continuous dynamics} of inference and learning to describe the same evolution of distributions under smooth, invertible reparameterizations \citep{amari1998natural,martens2020new}.

There is also a concrete computational reason for this requirement. Take a scalar belief parameter $\enc{u}$ and write it instead as $\enc{v} = 10\enc{u}$. If we apply ordinary Euclidean gradient descent in $\enc{v}$, the gradient becomes ten times smaller, and converting the step back to $\enc{u}$ introduces another factor of ten. Therefore, with the same learning rate, the update in $\enc{u}$ becomes 100$\times$ smaller. Same belief, same evidence, same objective, but a different pace of adaptation, just because we changed the units.

We can fix this scalar example by re-tuning the learning rate. But when different coordinates are rescaled differently, a single learning rate cannot generally restore all their relative update speeds. Such artificial conditioning can force small steps for stability in some directions while leaving others barely moving \citep{martens2020new}. With a limited inference or training budget, this can determine which beliefs or features we have time to learn. We may then spend extra computation, and extra stochastic training samples, compensating for arbitrary scales chosen by the engineer.

A coordinate choice should not get to decide how well we can adapt. Therefore, we draw inspiration from \citet{einstein1916grundlage} and postulate the same thing for physics of adaptation:

\boxit[title = {The principle of general covariance, applied to inference and learning}, halign = {center}, halign title = {center}]{The physics of adaptation must not depend on your choice of coordinate system.}

To implement this principle, we need a measure of motion that remain invariant under coordinate transformations. This is what motivates the kinetic energy of beliefs below.

\paragraph{The kinetic energy of beliefs.}
We can arrive at a suitable kinetic term from two directions. The first is geometric. Suppose some parameters $\blambda$ live on a manifold with Riemannian metric $\bG(\blambda)$. What is the simplest scalar, quadratic in velocity, that measures how fast they move? The metric gives us the squared speed, so for the kinetic energy, the natural choice is:
\begin{equation}\label{eq:kinetic-geometric}
\boxed{
    T(\blambda, \dot{\blambda})
    \;=\;
    \frac{1}{2}\,m\,\dot{\blambda}^{\top}\bG(\blambda)\dot{\blambda}~.
}
\end{equation}
Here, $\dot{\blambda} = d\blambda/dt$, and $m>0$ is a scalar that sets the relative scale of kinetic and potential energy. This expression is nonnegative, vanishes at rest, and treats motion in either direction equally. Crucially, it is invariant under a change of coordinates: the velocity and metric transform together, leaving their contraction unchanged. In Euclidean space, $\bG=\bI$, and we recover the familiar $\tfrac{1}{2}m\norm*{\dot{\blambda}}_2^2$.

This gives us the correct form of a kinetic energy on Riemannian manifolds, but we still need to choose the geometry of belief space.

The second route starts with a more physical question: \textit{what does it cost to change one's mind}?

\citet{zellner1988optimal} derived Bayes' rule as an optimal information processor, with no information lost between its inputs and outputs. Later, \citet{zellner2002information} showed that introducing costs of changing beliefs can alter the optimal update, producing dynamic partial adjustment rather than an immediate revision. This is the distinction we need: specifying the posterior we ought to reach does not tell us how a physical learner gets there, or what the journey costs. We take the next step by giving beliefs a kinetic term, representing their resistance to change as inertia.

For a brain, this cost has a concrete physical basis. Brains are made of matter, and changing their state takes energy. Fast changes in neural activity involve ion flows across cell membranes, and restoring the ion gradients costs ATP \citep{attwell2001energy}. Slower changes in synaptic strength can involve protein synthesis and physical remodeling of dendritic spines, supported by local metabolic resources \citep{rangaraju2019spatially}. Inference and learning take place through these physical processes. Why should the beliefs they implement be free to move?

To quantify this motion, let $q_{\blambda(t)}$ represent the learner's current belief. How much information separates it from the belief a short time $h$ later? \citet{hobson1969new} asked how to measure the information in a message that changes probabilities from $P$ to $Q$. He made a list of intuitive requirements: the measure should be continuous, independent of how outcomes are labeled, zero when nothing changes, and larger when a message rules out more equally likely possibilities. It should also obey a composition rule: describing an outcome in stages should give the same information as describing it all at once. For finite outcome spaces, these requirements uniquely select $\kl{Q}{P}$, up to a positive constant. Its continuous counterpart retains coordinate invariance because the density Jacobians cancel in the ratio.

Following \citep{hobson1969new}, we measure the information in a local belief change by:
\begin{equation}\label{eq:belief-information-gain}
    \mathcal{I}_t(h)
    \;\coloneqq\;
    \kl{q_{\blambda(t+h)}}{q_{\blambda(t)}}~.
\end{equation}

For a regular, smooth family with common support, the KL divergence and its first derivative vanish when the two distributions coincide. Its leading contribution is therefore quadratic, with the Fisher information as its Hessian \citep{amari2000methods,amari2016information}:
\begin{equation}\label{eq:belief-kl-fisher}
\begin{aligned}
    \kl{q_{\blambda+\Delta\blambda}}{q_{\blambda}}
    \;&=\;
    \frac{1}{2}\,\Delta\blambda^{\top}\bG(\blambda)\Delta\blambda
    \;+\;o\!\left(\norm{\Delta\blambda}^2\right),
    \\[3pt]
    \bG(\blambda)
    \;&=\;
    \expect{q_{\blambda}}{\big(\nabla_{\!\blambda}\log q_{\blambda}\big)
    \big(\nabla_{\!\blambda}\log q_{\blambda}\big)^{\top}}~.
\end{aligned}
\end{equation}

Along a smooth trajectory, $\Delta\blambda = h\dot{\blambda}+O(h^2)$. Dividing the local information gain by $h^2$ and taking the limit gives what we call the \textit{kinetic energy of beliefs}:
\begin{equation}\label{eq:kinetic-energy-beliefs}
\boxed{
\begin{aligned}
    T_\text{belief}(\blambda,\dot{\blambda})
    \;&\coloneqq\;
    \lim_{h\to0}\frac{\mathcal{I}_t(h)}{h^2}
    \\[5pt]
    \;&=\;
    \frac{1}{2}\left.\frac{d^2\mathcal{I}_t(h)}{dh^2}\right|_{h=0}
    \;=\;
    \frac{1}{2}\,\dot{\blambda}^{\top}\bG(\blambda)\dot{\blambda}~.
\end{aligned}
}
\end{equation}

The two arguments yield the same quadratic form (\cref{eq:kinetic-energy-beliefs,eq:kinetic-geometric}). From the geometry argument, we derived a squared speed, and information processing supplied the Fisher metric that measures it. This speed tracks changes in the distribution itself, so it is invariant to rescaling the parameters. Doubling it quadruples $T_\text{belief}$, just as doubling a classical particle's speed quadruples $T_\text{classical}$.

At this stage, $T_\text{belief}$ has units of information per time squared. Multiplying by $m$ gives the kinetic term $T=mT_\text{belief}$ in \cref{eq:kinetic-geometric}, with $m\bG(\blambda)$ acting as a position-dependent mass matrix. We interpret $m$ as setting the price of belief motion relative to reducing the loss. For a biological learner, this suggests a metabolic interpretation: moving beliefs faster requires more physical resources. Whether the brain's ATP expenditure follows this precise quadratic form, and whether it shares the Fisher metric's directional dependence, are open questions. We adopt this as a speculative interpretation of the kinetic term. The geometric derivation (\cref{eq:kinetic-geometric}) stands on its own.

Before we move on, we briefly mention a direct connection to recent stochastic thermodynamics. \citet{auconi2026information} study the \textit{intrinsic speed} of an evolving distribution, whose square is the temporal Fisher information, $v_\text{info}^2=\expect{q_t}{(\partial_t\log q_t)^2}=2T_\text{belief}$. For Markov jump and Fokker--Planck dynamics near a steady state, they relate $2v_\text{info}^2$ to the second time derivative of the KL divergence toward that state, and use their difference to bound steady-state entropy production. Thus, the same geometric measure of motion appears in physical constraints on dissipation.

Finally, this construction (\cref{eq:kinetic-energy-beliefs,eq:kinetic-geometric}) applies to both sides of FOND: inference uses the Fisher metric of the variational posterior, $\bG_u$, while learning uses that of the generative joint distribution, $\bG_\theta$. With the kinetic term chosen, we can now turn to the remaining ingredient: dissipation.

\paragraph{Dissipative equations of motion for beliefs.}
We now have both parts of the Lagrangian, $L=T-V$. Below, we will address the last missing piece: dissipation force $\bQ$.

But first, we make the Lagrangian and the canonical momentum explicit. Let $g_{ij}$ denote the entries of $\bG(\blambda)$ and $g^{ij}$ those of its inverse, with repeated indices summed. During each inference or learning phase, we hold the other parameter block fixed and take $V$ and $\bG$ to have no explicit time dependence. The Lagrangian and canonical momentum are then:
\begin{equation}\label{eq:belief-lagrangian-momentum}
\begin{aligned}
    L(\blambda,\dot{\blambda})
    \;&=\;
    \frac{m}{2}\,g_{ij}(\blambda)\dot{\lambda}^i\dot{\lambda}^j
    - V(\blambda),
    \\[3pt]
    p_i
    \;&\coloneqq\;
    \frac{\partial L}{\partial\dot{\lambda}^i}
    \;=\;m\,g_{ij}(\blambda)\dot{\lambda}^j,
    \qquad
    \dot{\lambda}^i=\frac{1}{m}\,g^{ij}(\blambda)p_j~.
\end{aligned}
\end{equation}

For dissipation, a reasonable standard choice is linear damping proportional to canonical momentum, $Q_i \coloneqq -p_i/\tau_\text{damp}$, where $\tau_\text{damp}>0$ is a constant damping time. This is \textit{isotropic with respect to the belief-space metric}: every direction has the same damping rate when measured in that geometry. It introduces no preferred direction beyond the metric we already chose. Other dissipation laws are possible, but we use this one as a simple, covariant starting point.

This force follows from the quadratic Rayleigh dissipation function:
\begin{equation}\label{eq:rayleigh}
\begin{aligned}
    \cR(\blambda,\dot{\blambda})
    \;&=\;
    \frac{T}{\tau_\text{damp}}
    \;=\;
    \frac{m}{2\tau_\text{damp}}\,
    \dot{\blambda}^{\top}\bG(\blambda)\dot{\blambda},
    \\[3pt]
    Q_i\;&=-\frac{\partial\cR}{\partial\dot{\lambda}^i}
    =-\frac{p_i}{\tau_\text{damp}}~.
\end{aligned}
\end{equation}
It removes energy at a rate $d(T+V)/dt=-2T/\tau_\text{damp}\leq0$. Thus, the learner can lose momentum and settle, rather than oscillating or orbiting indefinitely under conservative forces.

Substituting $L$ and $Q_i$ into the generalized Euler--Lagrange equation (\cref{eq:generalized-euler-lagrange}) gives:
\begin{equation}\label{eq:belief-euler-lagrange}
    \frac{dp_i}{dt}
    \,-\,
    \frac{m}{2}\,\partial_i g_{jk}\,\dot{\lambda}^j\dot{\lambda}^k
    \,+\,\partial_i V
    \;=\;
    -\frac{p_i}{\tau_\text{damp}}~,
\end{equation}
where $\partial_i\coloneqq\partial/\partial\lambda^i$. Using $\dot{\lambda}^i=g^{ij}p_j/m$ and the inverse-metric identity $\partial_i g^{jk}=-g^{ja}(\partial_i g_{ab})g^{bk}$, we obtain the full, dissipative equations of motion:
\begin{equation}\label{eq:dissipative-eom}
\boxed{
    \underbracegray{\frac{dp_i}{dt}}{\text{inertia}}
    \;+\;
    \underbracegray{\frac{1}{2m}\,\partial_i g^{jk}\,p_jp_k}{\text{geometric corrections}}
    \;+\;
    \underbracegray{\frac{p_i}{\tau_\text{damp}}}{\text{dissipation}}
    \;\;=\;\;
    \underbracegray{-\partial_i V}{\text{task force}}~.
}
\end{equation}

The geometric corrections arise because the metric changes as the learner moves. They are needed for the inertial dynamics to transform consistently between coordinate systems (i.e., to remain covariant). Equivalently, writing the equations in terms of velocity gives:
\begin{equation}\label{eq:belief-velocity-eom}
    m\left(\ddot{\lambda}^i+\Gamma^i_{jk}\dot{\lambda}^j\dot{\lambda}^k\right)
    \,+\,
    \frac{m}{\tau_\text{damp}}\,\dot{\lambda}^i
    \;=\;
    -g^{ij}\partial_j V~,
\end{equation}
where $\Gamma^i_{jk}=\tfrac{1}{2}g^{i\ell}(\partial_jg_{\ell k}+\partial_kg_{\ell j}-\partial_\ell g_{jk})$ are the Christoffel symbols of the metric. The acceleration and geometric correction together describe acceleration on the statistical manifold.

\paragraph{The overdamped (massless) limit.}
To recover first-order adaptation dynamics, we assume that momentum relaxes much faster than beliefs change. We make this precise by taking:
\begin{equation}\label{eq:def-overdamped}
\boxed{
    m\,\to\,0,
    \qquad
    \tau_\text{damp}\,\to\,0,
    \qquad
    \frac{m}{\tau_\text{damp}}\,\to\,\zeta\in(0,\infty)~.
}
\end{equation}
The last condition matters: inertia vanishes while friction remains finite. Heuristically, we can think of this as a \sayit{massless} learner. Imagine pushing a feather through thick honey. Its motion responds almost immediately to the balance between your push and the viscous drag, with very little coasting from its own momentum. The same intuition applies here: beliefs keep moving under the task force, but retain negligible momentum on the timescale of adaptation.

Consider the slower dynamics after the initial rapid relaxation of momentum, in a smooth region where the metric is nonsingular and the velocity, acceleration, and metric derivatives remain bounded as $m$ decreases. Then $p_i=O(m)$, so both $dp_i/dt$ and the geometric correction in \cref{eq:dissipative-eom} are $O(m)$. The damping term, however, approaches $\zeta g_{ij}\dot{\lambda}^j$ and remains finite. The surviving force balance is:
\begin{equation}\label{eq:belief-overdamped-balance}
    \zeta\,\bG(\blambda)\dot{\blambda}
    \;=\;-\nabla_{\!\blambda}V(\blambda)~.
\end{equation}

\paragraph{Natural-gradient dynamics and FOND's dual flows.}
Solving this balance for velocity gives:
\begin{equation}\label{eq:belief-natural-flow}
\boxed{
    \dot{\blambda}
    \;=\;
    -\frac{1}{\zeta}\,
    \bG(\blambda)^{-1}\nabla_{\!\blambda}V(\blambda)~.
}
\end{equation}
This is natural-gradient flow. The inverse metric appears because friction converts force into velocity through the geometry of belief space. The geometric corrections have vanished with inertia, but the geometry itself remains. Setting $V=\cF$ and applying this result separately to inference and learning gives \hyperref[fond]{the two FOND natural-gradient flows}:
\begin{equation}\label{eq:physics-fond-flows}
\boxed{
\begin{aligned}
    \dot{\enc{\bu}}
    \;&=\;
    -\frac{1}{\zeta_u}\,
    \bG_u(\enc{\bu})^{-1}
    \nabla_{\!\enc{\bu}}\cF(\enc{\bu},\dec{\btheta})~,
    \\[4pt]
    \dot{\dec{\btheta}}
    \;&=\;
    -\frac{1}{\zeta_\theta}\,
    \bG_\theta(\dec{\btheta})^{-1}
    \nabla_{\!\dec{\btheta}}\cF(\enc{\bu},\dec{\btheta})~.
\end{aligned}
}
\end{equation}
Here, $\zeta_u,\zeta_\theta>0$ set the pace of the respective processes and can be absorbed into their inference and learning timescales. Each flow uses its own Fisher metric, with the other parameter block held fixed.

\paragraph{Deriving the Bayesian learning rule \citep[BLR,][]{khan2023bayesian}.}
An Euler discretization of \cref{eq:belief-natural-flow}, with $\eta\coloneqq\Delta t/\zeta$, gives the familiar update:
\begin{equation}\label{eq:ngd}
\boxed{
    \blambda_{k+1}
    \;=\;
    \blambda_k
    \;-\;
    \eta\,\bG(\blambda_k)^{-1}\nabla_{\!\blambda}V(\blambda_k)~.
}
\end{equation}
This is just natural gradient descent, and replacing the gradient by a minibatch estimate gives its stochastic version.

For exponential-family variational beliefs, writing \cref{eq:ngd} in natural parameters and setting $V=\cF$ recovers the \textit{Bayesian learning rule} (BLR) of \citet{khan2023bayesian}. By varying the candidate distributions and the approximations to their natural gradients, BLR unified learning rules ranging from SGD and Newton's method to Kalman filtering. Here, we have derived the BLR from the physics of adaptation: it emerges as the overdamped limit of covariant belief dynamics.

\Cref{tab:ngd-special-cases} summarizes how the same update gives familiar learning rules under different choices of metric and loss.

\begin{table}[ht]
\centering
\caption{Special cases of \cref{eq:ngd} under different choices of metric $\bG$ and loss $V$. The BLR row uses exponential-family variational beliefs in natural parameters.}
\label{tab:ngd-special-cases}
\begin{tabular}{@{}lll@{}}
\toprule
\textbf{Algorithm} & \textbf{Metric} $\bG$ & \textbf{Loss} $V$ \\
\midrule
Gradient Descent \citep[GD,][]{cauchy1847methode} & $\bI$ (Euclidean) & Generic task loss \\[4pt]
Natural Gradient Descent \citep[NGD,][]{amari1998natural} & Fisher--Rao & Generic task loss \\[4pt]
Bayesian Learning Rule \citep[BLR,][]{khan2023bayesian} & Fisher--Rao & $\cF$ (free energy) \\
\bottomrule
\end{tabular}
\end{table}

\paragraph{Summary.}
Let us take stock of what we have done. We started from a physics-centered view of ``dynamics is everything,'' and demonstrated how dynamics can be derived from applying the principle of stationary action plus dissipation. We then identified the loss as a potential, required adaptation to be independent of arbitrary coordinates (principle of general covariance), and used local information gain to construct the kinetic energy of beliefs. A standard choice of friction then gave us the full dissipative equations of motion, including inertia and geometric corrections. Natural-gradient inference and learning followed in the overdamped (massless) limit, where that inertia is negligible.

\paragraph{Physics of adaptation as a broad framework for deriving inference and learning dynamics.}
We could have motivated the same practical update through a KL trust region (\cref{app:trust-region}), or appealed directly to Amari's natural gradient \citep{amari1998natural}. Those arguments are enough to justify the rule. We hope this route also reveals the richer framework underneath it: beliefs moving through a statistical manifold, driven by a task and constrained by the physical cost of adaptation. Future work should consider deriving other dynamics within that same framework by keeping inertia, changing the dissipation law, or considering more generalized belief states with different, perhaps adaptive, geometries that evolve in time in response to experience.

\section{Extended FOND model and dynamics}
\label{app:full-model}

\subsection{Model and objective}
\paragraph{Bernoulli instantiation of FOND.}
Recovering a feature dictionary requires two coupled operations: infer which features explain each observation, and use those explanations to improve the dictionary. FOND derives the first operation as recurrent natural-gradient descent on variational free energy \citep{vafaii2025brain}. Its original Poisson instantiation learns the decoder by backpropagating through the inference trajectory. Here we extend the natural-gradient prescription to learning itself. Inference changes the approximate posterior $q_{\enc{\bu}}$, and learning changes the generative distribution $p_{\dec{\btheta}}$. The two natural-gradient flows use their respective Fisher metrics to measure local changes in these distributions \citep{amari1998natural,martens2020new}. \Cref{app:trust-region} develops this geometric motivation.

\paragraph{Bernoulli features and a Gaussian decoder.}
In \befond, each latent indicates whether a feature is present, and its decoder atom specifies that feature's contribution to the observation. The posterior mean $\enc r_i$ therefore expresses confidence in feature presence. For $K$ binary latents and an $M$-dimensional observation, the local inference model is
\begin{align}
 p(\bz)&=\prod_{i=1}^K\operatorname{Bernoulli}(z_i;\dec\pi_{\mathrm{ref},i}),
 &q_{\enc\bu}(\bz\mid\bx)&=\prod_{i=1}^K\operatorname{Bernoulli}(z_i;\enc r_i),\notag\\
 \enc\br&=\sigma(\enc\bu),
 &p(\bx\mid\bz)&=\mathcal N(\bx;\dec\bPhi\bz,\dec\bSigma),
 \label{eq:app-model}
\end{align}
where $\dec\bSigma$ is diagonal and positive. Equivalently, $\bx=\dec\bPhi\bz+\boldsymbol\epsilon$ with $\boldsymbol\epsilon\sim\mathcal N(0,\dec\bSigma)$: the noise describes variation not explained by the selected atoms, while the atom itself sets the amplitude and direction of a present feature. The learned initial prior $\dec{\boldsymbol\pi}_0$ supplies the first reference, and the rolling rule in \cref{app:online-prior-working-memory} updates that reference as inference proceeds. A learned decoder bias accounts for a nonzero observation mean by replacing $\bx$ with $\bx-\dec\bb$. \Cref{app:bias} gives its update and inference-response correction.

Conditioning on the observation couples otherwise independent latent features: the likelihood contains cross-terms $\dec W_{ij}z_i z_j$, where $\dec W_{ij}=\dec{\boldsymbol\phi}_i^\top\dec\bSigma^{-1}\dec{\boldsymbol\phi}_j$. When atoms overlap, the evidence for one feature depends on which others explain the input. Factorizing $q$ restricts the represented posterior dependencies, but each optimized mean can still depend on the other means through the shared reconstruction (\cref{app:four-flows}).

The Poisson and Bernoulli choices both admit analytic moments and Fisher metrics \citep{wainwright2008graphical}. In canonical coordinates,
\begin{equation}
\begin{array}{c|c|c|c}
\text{latent family} & \text{support} & \enc r_i & G_{u,ii}=\operatorname{Var}_q(z_i) \\
\hline
\text{Poisson} & \{0,1,2,\ldots\} & e^{\enc u_i} & \enc r_i \\
\text{Bernoulli} & \{0,1\} & \sigma(\enc u_i) & \enc r_i(1-\enc r_i)
\end{array}
\label{eq:app-latent-families}
\end{equation}
The variance in the last column enters both the expected reconstruction loss and the geometry of inference. This is the connection that allows the distributional assumptions to determine the dynamics.

\paragraph{Free-energy decomposition.}
\label{app:befond-free-energy}
\Cref{app:free-energy-derivation} derives the fixed-prior ELBO identity. Here we evaluate its two terms for the Bernoulli--Gaussian model, holding the current reference fixed.
A distribution over codes assigns probability to several possible reconstructions. Write $\bdelta=\bz-\enc\br$ and $\be=\bx-\dec\bPhi\enc\br$. Since $\mathbb E_q[\bdelta]=0$ and $\mathbb E_q[\bdelta\bdelta^\top]=\enc\bC$, expanding the square gives:
\begin{align}
\EE_q\!\Big[\|\bx-\dec\bPhi\bz\|_{\dec\bSigma^{-1}}^2\Big]
&\;=\;
\be^\top\dec\bSigma^{-1}\be-2\be^\top\dec\bSigma^{-1}\dec\bPhi\mathbb E_q[\bdelta]
+\mathbb E_q[\bdelta^\top\dec\bW\bdelta]\notag\\
&=\|\be\|_{\dec\bSigma^{-1}}^2+\Tr(\dec\bW\enc\bC),
\end{align}
where $\dec\bW=\dec\bPhi^\top\dec\bSigma^{-1}\dec\bPhi$ and $\enc\bC=\operatorname{diag}[\enc\br\odot(\bm1-\enc\br)]$. The cross-term vanishes because the centered code has zero mean.
For the KL term, factorization turns the log ratio of products into a sum:
\begin{align}
D_{\mathrm{KL}}(q\|p)
&=\sum_i\mathbb E_q\!\left[z_i\log\frac{\enc r_i}{\dec\pi_{\mathrm{ref},i}}+(1-z_i)\log\frac{1-\enc r_i}{1-\dec\pi_{\mathrm{ref},i}}\right]\notag\\
&=\sum_i\left[\enc r_i\log\frac{\enc r_i}{\dec\pi_{\mathrm{ref},i}}+(1-\enc r_i)\log\frac{1-\enc r_i}{1-\dec\pi_{\mathrm{ref},i}}\right].
\end{align}
A decoder that reconstructs $\bx$ from $\enc\br$ can therefore still have a large expected error if the individual latent configurations produce different observations. Including this uncertainty cost and the KL penalty gives
\begin{align}
 \cF={}&\tfrac12(\bx-\dec\bPhi\enc\br)^\top\dec\bSigma^{-1}(\bx-\dec\bPhi\enc\br)
 +\tfrac12\sum_i\dec W_{ii}\enc r_i(1-\enc r_i)
 +\tfrac12\log\det\dec\bSigma\notag\\
 &+\beta\sum_i\left[\enc r_i\log\frac{\enc r_i}{\dec\pi_{\mathrm{ref},i}}
 +(1-\enc r_i)\log\frac{1-\enc r_i}{1-\dec\pi_{\mathrm{ref},i}}\right]+\mathrm{const},
 \label{eq:app-free-energy}
\end{align}
with $\dec\bW=\dec\bPhi^\top\dec\bSigma^{-1}\dec\bPhi$. The log determinant normalizes the Gaussian likelihood, and $\beta$ weights departure from the current prior reference. At $\beta=1$ and a fixed generative prior, this is the negative evidence lower bound. The Bernoulli uncertainty term also cancels the quadratic self-interaction in each coordinate, leaving reconstruction affine in $\enc r_i$ when the other means are fixed (\cref{app:bernoulli-sharpness}). Overlapping atoms still couple different coordinates, so a factorized posterior does not imply independent inference updates.

The moments needed for inference and learning are available without sampling latent codes:
\begin{align}
 \enc\bC&=\operatorname{diag}[\enc\br\odot(\bm1-\enc\br)],
 &\enc\bE&=\enc\bC+\enc\br\enc\br^\top,\notag\\
 \dec\bE_p&=\operatorname{diag}[\dec{\boldsymbol\pi}_{\mathrm{ref}}\odot(\bm1-\dec{\boldsymbol\pi}_{\mathrm{ref}})]
 +\dec{\boldsymbol\pi}_{\mathrm{ref}}\dec{\boldsymbol\pi}_{\mathrm{ref}}^\top,
 &\enc V_{mm}&=(\bx-\dec\bPhi\enc\br)_m^2+(\dec\bPhi\enc\bC\dec\bPhi^\top)_{mm}.
 \label{eq:app-moments}
\end{align}

Here $\enc\bE$ and $\dec\bE_p$ are the posterior and prior second moments, and $\enc V_{mm}$ is the posterior-expected squared residual in observed dimension $m$.

\paragraph{Bernoulli inference in \befond.}
For our factorized Bernoulli posterior, $\enc{\bu}$ contains logits
and $\enc{r}_i=\sigma(\enc{u}_i)$.
The score is $\bz-\enc{\br}$, so
\begin{equation}
\mathbf G_u
=\operatorname{diag}\!\left(\enc{\br}\odot
(\bm1-\enc{\br})\right),
\qquad
\Delta\enc{\bu}^{\top}\mathbf G_u\Delta\enc{\bu}
=\sum_i\enc{r}_i(1-\enc{r}_i)(\Delta\enc{u}_i)^2.
\label{eq:tr-bernoulli-length}
\end{equation}
For a fixed small logit displacement, the distribution changes less
when a Bernoulli probability is near zero or one. In probability
coordinates, the same infinitesimal length is
$\sum_i(d\enc{r}_i)^2/[\enc{r}_i(1-\enc{r}_i)]$:
a fixed probability displacement is more costly near the boundary.
For finite logits, all these variances are positive.
Applying the inverse Fisher cancels the variance factor in the
Euclidean logit gradient, yielding the inference dynamics derived in
\cref{app:four-flows}. The update therefore accounts for how much a
parameter movement changes the inferred beliefs.

\paragraph{Natural-gradient inference and learning.}
For Bernoulli logits, the Fisher is $\bG_u=\enc\bC$. The same factor $\enc r_i(1-\enc r_i)$ multiplies the ordinary logit gradient, so applying $\bG_u^{-1}$ removes this attenuation from the inference drive. The remaining terms combine evidence from the input, recurrent interactions through the dictionary, and relaxation toward the reference. \Cref{app:four-flows} derives these terms and their cancellation explicitly.

Learning then holds the inferred moments fixed. The decoder gradient compares the observed input--latent moment $\langle\bx\enc\br^\top\rangle$ with the model prediction $\dec\bPhi\enc\bE$. Its Fisher uses the prior second moment, $\bG_\Phi=\dec\bE_p\otimes\dec\bSigma^{-1}$, giving a correction by $\dec\bE_p^{-1}$ rather than the posterior's diagonal Fisher. This correction is what can compensate for weak learning signals in rarely active directions. Exact compensation requires matching the minibatch-averaged second moments (\cref{app:learning-geometry}). The covariance update tracks expected squared residuals, while the learned initial prior tracks detached posterior means across inputs. These decoder, covariance, and initial-prior updates use the inferred statistics directly, without backpropagating through the inference trajectory. Their explicit flows and derivations appear in \cref{app:four-flows,app:flow-derivations}. The decoder evolves in learning time $\tau$ over an interval $\Delta t$, and $h$ denotes the separate inference integration interval. 

\subsubsection{Online prior and working memory}
\label{app:online-prior-working-memory}

Consider observing a scene whose underlying features persist. After the first observation, the inferred distribution contains information about that particular scene that the population prior did not contain. The next update should start from this informed belief. FOND's rolling rule makes it the new prior reference \citep{vafaii2025brain}: the KL penalty now measures departure from the current prediction, while the likelihood supplies evidence for revising it. The learned initial prior starts the sequence, and the rolling reference carries its evolving, input-specific belief forward. \Cref{fig:online-prior-working-memory} shows this predict--update sequence and its fixed-input specialization. \Cref{app:rolling-prior} derives its activation threshold, and \cref{app:sparsity-results} reports the rolling-prior experiments varying $\beta$ and the initial-prior cap.

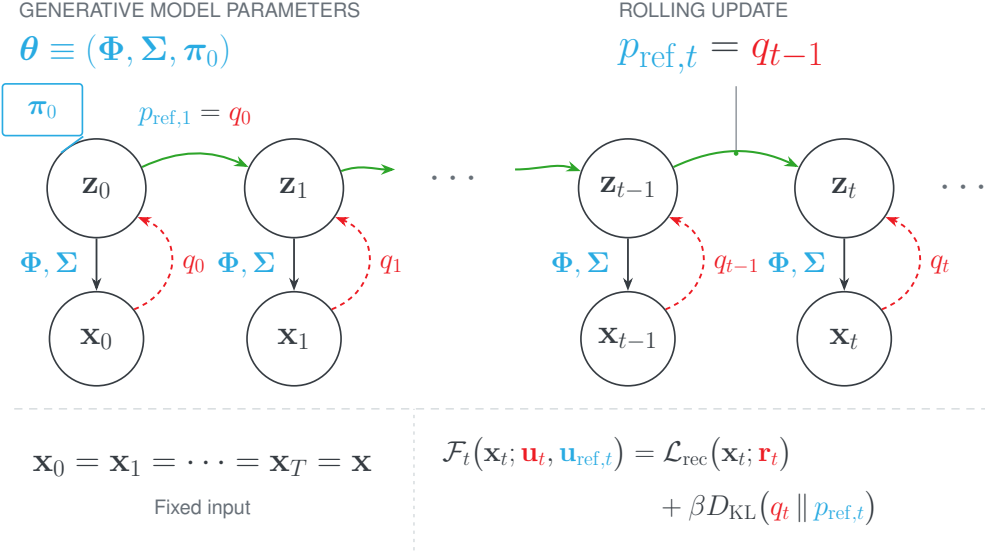
\begin{figure}[htbp]
\centering
\resizebox{\linewidth}{!}{
\begingroup
\usetikzlibrary{arrows.meta,calc}
\definecolor{wmink}{HTML}{363C42}
\definecolor{wmquiet}{HTML}{687078}
\definecolor{wmrule}{HTML}{D5DADD}
\definecolor{wminfer}{HTML}{ED1C24}
\definecolor{wmgener}{HTML}{29ABE2}
\definecolor{wmpredict}{HTML}{2ea52a}
\begin{tikzpicture}[x=1cm,y=1cm,>=Stealth,
 every node/.style={text=wmink,inner sep=0pt},
 state/.style={circle,draw=wmink,line width=1pt,minimum size=21mm,font=\huge},
 obs/.style={circle,draw=wmink,line width=1pt,minimum size=20mm,font=\huge},
 lab/.style={font=\LARGE},
 small/.style={font=\sffamily\large,text=wmquiet},
 flow/.style={->,line width=1.25pt,wmpredict},
 inferflow/.style={->,line width=1.15pt,wminfer,dash pattern=on 3pt off 2.7pt}]
\path[use as bounding box] (0,0) rectangle (21.4,12.2);
\node[anchor=west,small] at (0.35,11.65) {GENERATIVE MODEL PARAMETERS};
\node[anchor=west,font=\huge] at (0.35,10.75)
 {$\color{wmgener}\boldsymbol\theta\equiv(\boldsymbol\Phi,\boldsymbol\Sigma,\boldsymbol\pi_0)$};
\node[anchor=west,small] at (13.1,11.65) {ROLLING UPDATE};
\node[anchor=west,font=\Huge] at (13.1,10.65)
 {$\textcolor{wmgener}{p_{\mathrm{ref},t}}=\textcolor{wminfer}{q_{t-1}}$};
\draw[wmquiet,line width=.65pt] (15.6,10.02)--(15.6,8.59);
\fill[wmpredict] (15.6,8.59) circle (1.4pt);

\node[state] (z0) at (2.0,7.85) {$\mathbf z_0$};
\node[state] (z1) at (6.2,7.85) {$\mathbf z_1$};
\node[state] (zm) at (13.3,7.85) {$\mathbf z_{t-1}$};
\node[state] (zt) at (17.9,7.85) {$\mathbf z_t$};
\node[obs] (x0) at (2.0,4.65) {$\mathbf x_0$};
\node[obs] (x1) at (6.2,4.65) {$\mathbf x_1$};
\node[obs] (xm) at (13.3,4.65) {$\mathbf x_{t-1}$};
\node[obs] (xt) at (17.9,4.65) {$\mathbf x_t$};

\node[draw=wmgener,rounded corners=2pt,line width=1pt,minimum width=17mm,minimum height=11mm,
 font=\LARGE,text=wmgener] (prior) at (.85,9.52) {$\boldsymbol\pi_0$};
\draw[wmgener,line width=1.1pt] (prior.south east)--(z0.135);
\draw[flow] (z0.25) to[out=25,in=155] (z1.155);
\node[lab] at (4.1,9.35) {$\textcolor{wmgener}{p_{\mathrm{ref},1}}=\textcolor{wminfer}{q_0}$};
\draw[flow] (z1.20) to[out=20,in=180] (8.35,8.25);
\node[font=\Huge,text=wmquiet] at (9.65,8.05) {$\cdots$};
\draw[flow] (10.9,8.25) to[out=0,in=160] (zm.160);
\draw[flow] (zm.25) to[out=25,in=155] (zt.155);
\node[font=\Huge,text=wmquiet] at (20.5,7.85) {$\cdots$};

\foreach \z/\x/\s in {z0/x0/0,z1/x1/1,zm/xm/{t-1},zt/xt/t}{
 \draw[->,wmink,line width=1.1pt] (\z.south)--(\x.north);
 \node[lab,text=wmgener,anchor=east] at ($(\z.south)!0.5!(\x.north)+(-.38,0)$) {$\boldsymbol\Phi,\boldsymbol\Sigma$};
 \draw[inferflow] (\x.38) .. controls ($(\x.38)+(1.0,.40)$) and ($(\z.-35)+(1.0,-.40)$) .. (\z.-35);
 \node[lab,text=wminfer,anchor=west] at ($(\z.south)!0.5!(\x.north)+(1.82,0)$) {$q_{\s}$};
}
\draw[wmrule,line width=.85pt,dash pattern=on 5pt off 4pt] (.25,3.16)--(21.05,3.16);
\draw[wmrule,line width=.85pt,dash pattern=on 2pt off 3pt] (8.75,.15)--(8.75,3.16);
\node[font=\huge] at (4.25,1.95) {$\mathbf x_0=\mathbf x_1=\cdots=\mathbf x_T=\mathbf x$};
\node[small] at (4.25,1.02) {Fixed input};
\node[anchor=west,font=\LARGE] at (9.35,1.63) {$\begin{aligned}
 \mathcal F_t\bigl(\mathbf x_t;\textcolor{wminfer}{\mathbf u_t},\textcolor{wmgener}{\mathbf u_{\mathrm{ref},t}}\bigr)
 &={}\mathcal L_{\mathrm{rec}}\bigl(\mathbf x_t;\textcolor{wminfer}{\mathbf r_t}\bigr)\\[8pt]
 &\quad+\beta D_{\mathrm{KL}}\bigl(\textcolor{wminfer}{q_t}\,\|\,\textcolor{wmgener}{p_{\mathrm{ref},t}}\bigr)
\end{aligned}$};
\end{tikzpicture}
\endgroup}
\caption{Online inference with a rolling prior. The learned initial prior initializes the first belief. Under identity state dynamics, prediction carries each posterior into the next reference (green arrows). The observation model is shown by black arrows and inference by red dashed arrows. The fixed-input case presents the same observation at each step.}
\label{fig:online-prior-working-memory}
\end{figure}

\paragraph{Prior prediction.}
Let $q_t$ be the inferred distribution after observation $\bx_t$. To form the next reference, a state-transition model predicts which latent configurations remain plausible. For a transition kernel $\mathcal K(\bz'\mid\bz)$, prediction averages over the preceding belief \citep[Chapter~6]{sarkka2023bayesian}:
\begin{equation}
p_{\mathrm{ref},t}(\bz')
=\sum_{\bz}\mathcal K(\bz'\mid\bz)q_{t-1}(\bz).
\label{eq:app-online-prediction}
\end{equation}
If the latent state persists without process noise, then
$\mathcal K(\bz'\mid\bz)=\mathbf 1[\bz'=\bz]$, so prediction simply
copies the belief:
\begin{equation}
p_{\mathrm{ref},0}=p_0,
\qquad p_{\mathrm{ref},t}=q_{t-1}\quad(t\geq1),
\qquad
p_0(\bz)=\prod_i\operatorname{Bernoulli}
(z_i;\dec{\pi}_{0,i}).
\label{eq:app-online-identity}
\end{equation}
The identity transition means that the causes persist, not that we know which causes are present. Copying $q_{t-1}$ therefore retains both its feature preferences and its uncertainty. With no new likelihood contribution, the belief is preserved unchanged. This is the working memory supplied by the rolling reference. A new independent input starts again from the learned initial prior.

\paragraph{Posterior update.}
The observation can support the prediction or require it to change. Expected reconstruction error rewards explanations of the current input, while the KL penalty controls departure from the prediction. For a Bernoulli reference with means $\dec{\boldsymbol\pi}_{\mathrm{ref},t}$ and logits $\dec{\bu}_{\mathrm{ref},t}=\operatorname{logit}(\dec{\boldsymbol\pi}_{\mathrm{ref},t})$, the local objective is
\begin{equation}
\cF_t(\enc{\bu})
=\underbrace{-\mathbb E_{q_{\enc{\bu}}}
\log p_{\dec{\btheta}}(\bx_t\mid\bz)}_{
\mathcal L_{\mathrm{rec}}(\bx_t;\enc{\br})}
+\beta D_{\mathrm{KL}}\!\left(q_{\enc{\bu}}
\,\middle\|\,p_{\mathrm{ref},t}\right),
\qquad \enc{\br}=\sigma(\enc{\bu}).
\label{eq:app-online-objective}
\end{equation}
The decoder and reference remain fixed during this local update. At $\beta=1$, minimizing over unrestricted distributions gives the Bayesian posterior for the reference and observation. For general $\beta>0$, it gives a distribution proportional to $p_{\mathrm{ref},t}(\bz)p_{\dec{\btheta}}(\bx_t\mid\bz)^{1/\beta}$ \citep{bissiri2016general}. The factorized Bernoulli family gives tractable expectations and local updates while retaining marginal uncertainty and approximating posterior dependencies.

\paragraph{Recurrent belief revision.}
Write the negative mean-coordinate reconstruction gradient as $\ba_t(\enc\br)$. Including the Bernoulli uncertainty term, its $i$th coordinate is
\begin{equation}
[\ba_t(\enc\br)]_i
=(\dec\bphi^{(i)})^\top\dec\bSigma^{-1}\bx_t
-\tfrac12\dec W_{ii}
-\sum_{j\ne i}\dec W_{ij}\enc r_j.
\label{eq:app-online-drive}
\end{equation}

The first term supplies evidence for feature $i$, while the last adjusts it for other active explanations. When $\dec W_{ij}>0$, increasing belief in feature $j$ reduces the drive for $i$. Repeated updates thus let explanations revise one another even though $q$ is factorized.

Applying the posterior Fisher inverse gives:
\begin{equation}
\dot{\enc{\bu}}
=\ba_t(\enc{\br})
-\beta\bigl(\enc{\bu}-\dec{\bu}_{\mathrm{ref},t}\bigr).
\label{eq:app-online-flow}
\end{equation}

Initialize at the predictive logits. The KL force is initially zero, so evidence can change the belief before the fixed reference exerts a restoring force. One ETD1 step with interval $h$ integrates this relaxation while freezing the reconstruction drive (\cref{app:numerical-implementation}):
\begin{equation}
\enc{\bu}_t
=\dec{\bu}_{\mathrm{ref},t}
+\frac{1-e^{-\beta h}}{\beta}
\ba_t\!\left(\dec{\boldsymbol\pi}_{\mathrm{ref},t}\right),
\qquad \beta>0.
\label{eq:app-online-etd1}
\end{equation}

The resulting posterior becomes the reference for the next outer step. The next update therefore retains the inferred changes instead of pulling them back toward the original population prior. If several inner steps are taken, their reference stays fixed and the restoring force generally becomes nonzero. The parameter $\beta$ controls resistance to local belief change, and $h$ controls how far the local flow is integrated.

\paragraph{Evidence accumulation and sparsity.}
Rolling also changes the influence of the initial prior as a sequence proceeds. To see this, suppose the initial prior and likelihood are positive and each local objective is minimized exactly over unrestricted distributions. Repeated substitution of the posterior for the next reference gives
\begin{equation}
q_t(\bz)\ \propto\ p_0(\bz)
\prod_{j=0}^{t}p_{\dec{\btheta}}(\bx_j\mid\bz)^{1/\beta}.
\label{eq:app-online-evidence}
\end{equation}
Each observation contributes a likelihood factor, while the initial prior enters once. For a fixed repeated input, the likelihood power is $(t+1)/\beta$: the initial prior becomes weaker relative to reconstruction evidence. Replaying one stored observation supplies no new measurement, so this computation is not ordinary conditioning on one observation. Finite mean-field ETD1 steps approximate the local updates and need not attain this unrestricted endpoint.

This explains why a sparse initial prior need not enforce the same sparsity throughout a rollout. Its low probabilities set a high initial activation threshold, but the rolling reference can retain increases supported by reconstruction evidence. A cap on mean initial-prior activity therefore constrains the starting belief, not final posterior activity.

\subsection{Natural-gradient dynamics}
\label{app:four-flows}
For the free energy in \Cref{app:befond-free-energy}, inference changes the posterior while holding the model fixed. Learning then changes the model with the inferred statistics fixed:
\begin{align}
 \dot{\enc\bu}&=\dec\bPhi^\top\dec\bSigma^{-1}\bx-\dec\bW\enc\br
 -\tfrac12\operatorname{diag}(\dec\bW)\odot(\bm1-2\enc\br)-\beta(\enc\bu-\dec\bu_{\mathrm{ref}}),\label{eq:app-u-flow}\\
 \dot{\dec\bPhi}&=\eta\ssub{\Phi}(\bx\enc\br^\top-\dec\bPhi\enc\bE)\dec\bE_p^{-1},\label{eq:app-phi-flow}\\
 \dot{\dec\bSigma}&=\eta\ssub{\Sigma}(\enc\bV-\dec\bSigma),\label{eq:app-sigma-flow}\\
 \dot{\dec{\boldsymbol\pi}}_0&=\eta_0\beta(\enc\br-\dec{\boldsymbol\pi}_0).\label{eq:app-prior-flow}
\end{align}
Here $\dec\bu_{\mathrm{ref}}=\operatorname{logit}(\dec{\boldsymbol\pi}_{\mathrm{ref}})$. The initial prior $\dec{\boldsymbol\pi}_0$ tracks detached posterior means, and the rolling reference is refreshed within inference (\cref{app:online-prior-working-memory}). For minibatches, average $\bx\enc\br^\top$, $\enc\bE$, $\dec\bE_p$, and $\enc\bV$ before applying the updates. Derivations appear in \Cref{app:flow-derivations}.

\paragraph{Fisher preconditioning for inference.}
For the Bernoulli posterior, $\partial\log q/\partial\enc u_i=z_i-\enc r_i$. Independence under $q$ gives
\begin{align}
\bG_u=\mathbb E_q[(\bz-\enc\br)(\bz-\enc\br)^\top]
&=\enc\bC,\label{eq:app-bernoulli-fisher}\\
\nabla_{\enc\bu}\cF
&=\enc\bC\Bigl[
\dec\bW\enc\br-\dec\bPhi^\top\dec\bSigma^{-1}\bx
+\tfrac12\operatorname{diag}(\dec\bW)\odot(\bm1-2\enc\br)\Bigr]
\notag\\
&\quad+\enc\bC\beta(\enc\bu-\dec\bu_{\mathrm{ref}}).
\label{eq:app-bernoulli-fisher-gradient}
\end{align}
Multiplication by $-\bG_u^{-1}$ cancels $\enc\bC$ and yields \cref{eq:app-u-flow} whenever $0<\enc r_i<1$. Ordinary logit-gradient descent instead multiplies the same drive by $\enc r_i(1-\enc r_i)$, making it small near either boundary. The natural flow removes this factor from the logit update, while the probability still changes according to $\dot{\enc\br}=\enc\bC\dot{\enc\bu}$. This distinction matters when interpreting the revival of a nearly inactive latent.

\paragraph{Decoder Fisher information.}
Let $\langle\cdot\rangle$ denote the minibatch average. The Gaussian decoder score is $\dec\bSigma^{-1}(\bx-\dec\bPhi\bz)\bz^\top$. Averaging its outer product over the complete-data model, with the chosen prior references held fixed, gives
\begin{equation}
\bG_\Phi=\dec\bE_p\otimes\dec\bSigma^{-1},
\qquad
\nabla_{\dec\bPhi}\cF
=\dec\bSigma^{-1}\left(\dec\bPhi\enc\bE-\langle\bx\enc\br^\top\rangle\right),
\label{eq:app-decoder-fisher-explained}
\end{equation}
where both second moments are averaged across the minibatch before the Fisher action. Vectorization of decoder entries is implicit. The resulting matrix-form flow is \cref{eq:app-phi-flow}, with $\bx\enc\br^\top$ replaced by its minibatch average. This is the complete-data model Fisher, rather than the observed-data Fisher or an outer product of empirical loss gradients \citep{amari1998natural,martens2020new}.

The distinction between $\dec\bE_p$ and $\enc\bE$ is essential: the former supplies the metric, and the latter the curvature of the frozen reconstruction objective. Averaging per-example inverses would give a different update. Likewise, matching average firing rates alone does not match second moments. For example, a shared factorized prior with mean $\langle\enc\br\rangle$ has off-diagonal moments $\langle\enc r_i\rangle\langle\enc r_j\rangle$, whereas the averaged posterior has $\langle\enc r_i\enc r_j\rangle$. Their difference is the across-input covariance of the inferred means. \Cref{app:learning-geometry} states the exact matching condition and the unequal-moment extension.

\subsubsection{Model updates}
\label{app:flow-derivations}
\label{app:befond-flow-derivations}
\begin{proof}[Proof of \Cref{eq:app-u-flow,eq:app-prior-flow}]
For Bernoulli inference, differentiating \cref{eq:app-free-energy} gives
\[
 \nabla_{\enc\br}\cF=-\dec\bPhi^\top\dec\bSigma^{-1}\bx+\dec\bW\enc\br
 +\tfrac12\operatorname{diag}(\dec\bW)\odot(\bm1-2\enc\br)+\beta(\enc\bu-\dec\bu_{\mathrm{ref}}).
\]
The posterior Fisher $\enc\bC$ cancels the logit Jacobian, yielding \cref{eq:app-u-flow}. For the decoder, the complete-data Fisher block is $\dec\bE_p\otimes\dec\bSigma^{-1}$, which cancels the precision in the gradient and gives \cref{eq:app-phi-flow}.

For each diagonal variance and each initial-prior probability,
\begin{align}
 \frac{\partial\cF}{\partial\dec\Sigma_{mm}}
 &=\frac{\dec\Sigma_{mm}-\enc V_{mm}}{2\dec\Sigma_{mm}^2},
 &G_{\dec\Sigma_{mm}}&=\frac1{2\dec\Sigma_{mm}^2},\notag\\
 \frac{\partial[\beta\kl{q}{p_0}]}{\partial\dec\pi_{0,i}}
 &=\beta\frac{\dec\pi_{0,i}-\enc r_i}{\dec\pi_{0,i}(1-\dec\pi_{0,i})},
 &G_{\dec\pi_{0,i}}&=\frac1{\dec\pi_{0,i}(1-\dec\pi_{0,i})}.
\end{align}

Preconditioning proves \cref{eq:app-sigma-flow,eq:app-prior-flow}. With frozen targets, integration yields:
\begin{align}
 \dec\bSigma(\Delta t)&=\enc\bV+[\dec\bSigma(0)-\enc\bV]e^{-\eta\ssub{\Sigma}\Delta t},\notag\\
 \dec{\boldsymbol\pi}_0(\Delta t)&=\enc\br+[\dec{\boldsymbol\pi}_0(0)-\enc\br]e^{-\eta_0\beta\Delta t}.
 \label{eq:app-ema}
\end{align}
\end{proof}

\paragraph{Decoder-bias correction.}
\label{app:bias-derivation}
\begin{proof}[Proof of \Cref{eq:app-bias-flow}]
Hold the incoming reference and covariance fixed during one unclamped inference step. Only the feedforward drive depends on the bias, so
\[
 \frac{\partial\enc\bu^{\mathrm{new}}}{\partial\dec\bb}
 =-\frac{1-e^{-\beta h}}{\beta}\dec\bPhi^\top\dec\bSigma^{-1},
 \qquad \nabla_{\enc\bu}\cF=-\enc\bC\dot{\enc\bu}.
\]
Applying the chain rule at the updated posterior gives
\[
 \nabla_{\dec\bb}\cF
 =-\dec\bSigma^{-1}(\bx-\dec\bb-\dec\bPhi\enc\br)
 +\frac{1-e^{-\beta h}}{\beta}\dec\bSigma^{-1}\dec\bPhi\enc\bC\dot{\enc\bu}.
\]
Multiplication by the inverse bias Fisher block $\dec\bSigma$ and minibatch averaging gives \cref{eq:app-bias-flow}.
\end{proof}

\subsection{Decoder bias}
\label{app:bias}
A learned bias gives $p(\bx\mid\bz)=\mathcal N(\bx;\dec\bb+\dec\bPhi\bz,\dec\bSigma)$, and inference uses $\bx-\dec\bb$. With a fixed posterior, the bias fits the mean residual. Differentiating through one unclamped ETD1 inference step adds a response term:
\begin{equation}
 \dot{\dec\bb}=\eta\ssub{\Phi}\EE_{\mathrm{batch}}\!\left[
 \underbrace{\bx-\dec\bb-\dec\bPhi\enc\br}_{\text{reconstruction residual}}
 -\underbrace{\frac{1-e^{-\beta h}}{\beta}\dec\bPhi\enc\bC\dot{\enc\bu}}_{\text{inference response}}\right].
 \label{eq:app-bias-flow}
\end{equation}
The incoming reference is fixed, and $\enc\bC$ and $\dot{\enc\bu}$ are evaluated after that step. This expression assumes one shared covariance across the batch. \Cref{app:bias-derivation} derives the correction. The full-width synthetic experiment freezes the bias.

\subsection{Bernoulli uncertainty and sparse coding}
\label{app:bernoulli-sharpness}
The Bernoulli posterior variance is $\operatorname{Var}_q(z_i)=\enc r_i(1-\enc r_i)$, which vanishes at 0 and 1 and is maximal at $1/2$. In isolation, its cost favors confident absence or presence. In the full expected reconstruction loss, it cancels the quadratic self-term:
\begin{equation}
 \underbrace{\tfrac12\dec W_{ii}\enc r_i^2}_{\text{mean reconstruction}}
 +\underbrace{\tfrac12\dec W_{ii}\enc r_i(1-\enc r_i)}_{\text{posterior uncertainty}}
 =\tfrac12\dec W_{ii}\enc r_i.
 \label{eq:app-self-cancellation}
\end{equation}
Consequently, with the other coordinates fixed, reconstruction is affine in $\enc r_i$ and the KL term is strictly convex for $\beta>0$. The full objective has an interior coordinate optimum at finite evidence.

\subsection{Activation threshold under a rolling prior}
\label{app:rolling-prior}
\Cref{app:online-prior-working-memory} motivates the rolling rule. Here we make its sparsity effect explicit. Starting one ETD1 step at the reference, the gain in \cref{eq:app-online-etd1} tends to $1/\beta$ as $h\to\infty$. Since $\sigma(u)>1/2$ exactly when $u>0$, for $\beta>0$ we obtain
\begin{equation}
 \enc r_i^{\mathrm{new}}>\tfrac12
 \quad\Longleftrightarrow\quad
 (\dec\bphi^{(i)})^\top\dec\bSigma^{-1}\bx-\tfrac12\dec W_{ii}
 -\sum_{j\ne i}\dec W_{ij}\dec\pi_{\mathrm{ref},j}
 >\beta\log\frac{1-\dec\pi_{\mathrm{ref},i}}{\dec\pi_{\mathrm{ref},i}}.
 \label{eq:app-barrier}
\end{equation}
The right side is positive for a reference probability below $1/2$ and grows as that probability decreases. It sets the reconstruction evidence needed to activate the feature, after accounting for the other features' contributions on the left. Updating the reference changes this threshold at the next outer step. The prior cap and inference depth can therefore affect posterior sparsity through different parts of the rollout (\cref{app:sparsity-results}).

\subsection{Numerical integration and matrix-free updates}
\label{app:numerical-implementation}
\paragraph{Stiffness and exponential integration.}
Natural gradients specify the continuous-time evolution of beliefs and model parameters, but implementing these dynamics requires a numerical update.
A difficulty with explicit Euler is that stable continuous dynamics can become unstable after discretization. For the simple relaxation $\dot v=-\beta v$, Euler multiplies $v$ by $1-\beta h$, whereas exact integration multiplies it by $e^{-\beta h}$. Euler therefore requires $0<\beta h<2$ for decay. When fast relaxation accompanies slower dynamics, this stability restriction can force small steps even when the solution changes slowly. Exponential integrators address this source of stiffness by integrating the linear component exactly \citep{cox2002exponential,hochbruck2010exponential}.

\paragraph{Inference integration.}
Write \cref{eq:app-u-flow} as $\dot{\enc\bu}=\ba(\enc\br)-\beta(\enc\bu-\dec\bu_{\mathrm{ref}})$, where $\ba$ contains the reconstruction and Bernoulli uncertainty terms. Holding the reference fixed and freezing $\ba$ during one interval gives the \textit{first-order exponential time-differencing} update \citep[ETD1;][]{cox2002exponential,hochbruck2010exponential}:
\begin{equation}
\boxed{
\enc\bu_{\mathrm{new}}
\;=\;
\dec\bu_{\mathrm{ref}}
\,+\,
e^{-\beta h}(\enc\bu-\dec\bu_{\mathrm{ref}})
\,+\,
\frac{1-e^{-\beta h}}{\beta}\ba(\enc\br)~,
}
\qquad\beta>0~.
\label{eq:app-etd1-explained}
\end{equation}

The reference relaxation is thus preserved exactly. Here ETD1 follows the same natural-gradient direction as Euler, with effective step $(1-e^{-\beta h})/\beta$, bounded by $1/\beta$. Starting from the reference removes the middle term. \Cref{app:online-prior-working-memory} describes how that reference is passed between inference steps. At $h\to\infty$, ETD1 solves the subproblem with the reconstruction drive held fixed. We apply a finite number of recurrent inference steps before each decoder update.

\paragraph{Frozen-statistics learning.}
The same principle applies to learning. With minibatch statistics, prior references, and covariance fixed, and $\enc\bE,\dec\bE_p\succ0$, the decoder flow has the exact solution (\cref{app:learning-geometry})
\begin{equation}
\dec\bPhi(\Delta t)-\dec\bPhi^*
=\bigl(\dec\bPhi(0)-\dec\bPhi^*\bigr)
\exp\!\left[-\eta\ssub{\Phi}\Delta t\,\enc\bE\dec\bE_p^{-1}\right].
\end{equation}

Euler must accommodate the fastest relaxation mode, but exponential integration instead applies each mode's exact decay over the chosen interval. When $\enc\bE=\dec\bE_p$, every mode contracts by the same factor, recovering exact rarity compensation.

The configured decoder solver, \texttt{matrix\_free\_phi1\_auto}, evaluates a matrix-function action on the decoder residual. Let $s=\eta\ssub{\Phi}\Delta t$, $\dec{\bE}_{p,\lambda}=\dec\bE_p+\lambda\bI$ with Fisher damping $\lambda\ge0$, and $\bA=\enc\bE\dec{\bE}_{p,\lambda}^{-1}$. The corresponding frozen update is:
\begin{equation}
\begin{aligned}
\dec\bPhi(\Delta t)
&=\dec\bPhi(0)\\
&\quad+s\bigl(\langle\bx\enc\br^\top\rangle-\dec\bPhi(0)\enc\bE\bigr)
\dec{\bE}_{p,\lambda}^{-1}\varphi_1(-s\bA),
\end{aligned}
\label{eq:app-decoder-phi1}
\end{equation}
where brackets denote minibatch averages and
$\varphi_1(\bZ)=\sum_{m=0}^{\infty}\bZ^m/(m+1)!$ is the entire matrix function used in exponential integration~\citep{hochbruck2010exponential}.
For $\lambda=0$, this recovers the exact frozen-statistics solution above. The residual form avoids explicitly computing $\dec\bPhi^*$ or subtracting a near-identity exponential from the identity.
The implementation applies this function without constructing a dense $K\times K$ matrix. For a matrix $\bR$ whose rows are minibatch posterior means,
\begin{equation}
\enc\bE\bv
=\operatorname{diag}(\langle\enc\br\odot(\bm1-\enc\br)\rangle)\bv
+\frac1B\bR^\top(\bR\bv),
\label{eq:app-matrix-free-moment}
\end{equation}
and the same representation applies to $\dec\bE_p$.

Low-rank solves use the Woodbury identity~\citep{hager1989updating} when the rank is at most 1,024. Larger ranks use Jacobi-preconditioned conjugate gradients \citep{saad2003iterative}. Thus the minibatch rank and conditioning also affect cost (\cref{app:complexity}).

The solver combines scaled Taylor evaluation~\citep{almohy2011computing} with Chebyshev polynomial approximation~\citep{trefethen2019approximation}. It selects Taylor evaluation when the predicted work is within 32 operator applications, and otherwise uses a Chebyshev recurrence with degree chosen from the coefficient tail. Each operator application requires a posterior-moment product and a solve with the damped prior moment, so the polynomial degree and linear-solve iterations contribute to decoder-update cost. The selected synthetic recipes and audited whitened Gemma configurations request Chebyshev tolerance $10^{-5}$ and maximum degree 2,048 (512 for the raw 16K Gemma run). These control the matrix-function approximation.

The variance and initial-prior flows likewise integrate to exponential averages of their frozen targets (\cref{app:flow-derivations}). These updates exploit solvable components of the derived dynamics while leaving approximation to the remaining nonlinear terms and numerical matrix-function evaluation. Parameter floors, a mean initial-prior cap, and decoder constraints are applied as specified by each recipe. All inferred statistics are detached during learning, so the decoder update does not backpropagate through the inference trajectory.

\subsection{Algorithm}
\Cref{alg:fond} gives the Bernoulli inference and moment updates. During inference, the generative parameters remain fixed. Learning uses the inferred minibatch moments.

\begin{algorithm}[H]
\caption{FOND with a rolling inference reference}
\label{alg:fond}
\label{app:algorithm}
\begin{algorithmic}[1]
\Require Decoder $\dec\bPhi$, diagonal covariance $\dec\bSigma\succ0$, initial prior $\dec{\boldsymbol\pi}_0\in(0,1)^K$, inference depths, $\beta>0$, integration parameter $h$, and learning schedules.
\For{each training minibatch}
\State Center observations by the decoder bias, if learned.
\State $\enc\bu\gets\operatorname{logit}(\dec{\boldsymbol\pi}_0)$, independently for each observation.
\For{each outer step up to the sampled training horizon}
\State $\dec\bu_{\mathrm{ref}}\gets\enc\bu$, $\dec{\boldsymbol\pi}_{\mathrm{ref}}\gets\sigma(\dec\bu_{\mathrm{ref}})$.
\State $\dec\bE_p\gets\EE_{\mathrm{batch}}[\operatorname{diag}(\dec{\boldsymbol\pi}_{\mathrm{ref}}\odot(\bm1-\dec{\boldsymbol\pi}_{\mathrm{ref}}))+\dec{\boldsymbol\pi}_{\mathrm{ref}}\dec{\boldsymbol\pi}_{\mathrm{ref}}^\top]$.
\For{each inner inference step}
\State $\enc\br\gets\sigma(\enc\bu)$, $\dec\bW\gets\dec\bPhi^\top\dec\bSigma^{-1}\dec\bPhi$.
\State $\dot{\enc\bu}\gets\dec\bPhi^\top\dec\bSigma^{-1}\bx-\dec\bW\enc\br-\tfrac12\operatorname{diag}(\dec\bW)\odot(\bm1-2\enc\br)-\beta(\enc\bu-\dec\bu_{\mathrm{ref}})$.
\State $\enc\bu\gets\enc\bu+\dfrac{1-e^{-\beta h}}{\beta}\dot{\enc\bu}$.
\EndFor
\If{this is the final outer step}
\State $\enc\br\gets\sigma(\enc\bu)$, $\enc\bC\gets\operatorname{diag}[\enc\br\odot(\bm1-\enc\br)]$, $\enc\bE\gets\EE_{\mathrm{batch}}[\enc\bC+\enc\br\enc\br^\top]$.
\State $\enc V_{mm}\gets\EE_{\mathrm{batch}}[(\bx-\dec\bPhi\enc\br)_m^2+(\dec\bPhi\enc\bC\dec\bPhi^\top)_{mm}]$.
\State $\dec\bPhi^*\gets\EE_{\mathrm{batch}}[\bx\enc\br^\top]\enc\bE^{-1}$.
\State $\dec\bPhi\gets\dec\bPhi^*+(\dec\bPhi-\dec\bPhi^*)\exp[-\eta\ssub{\Phi}\Delta t\,\enc\bE\dec\bE_p^{-1}]$.
\State Update enabled variance and bias parameters by \cref{eq:app-ema,eq:app-bias-flow}, then apply the stated decoder constraints.
\EndIf
\EndFor
\If{an initial-prior update is scheduled}
\State Infer $\enc\br$ on the prior-update batch, then detach it.
\State $\dec{\boldsymbol\pi}_0\gets\EE_{\mathrm{batch}}[\enc\br]+(\dec{\boldsymbol\pi}_0-\EE_{\mathrm{batch}}[\enc\br])e^{-\eta_0\beta\Delta t}$.
\State Apply the specified prior floor and mean-probability cap.
\EndIf
\EndFor
\end{algorithmic}
\end{algorithm}

The decoder line is the frozen-moment matrix-exponential update. Implementations can apply its matrix functions without forming dense inverses. ETD1 uses gain $1/\beta$ at $h=\infty$. Evaluation resets the initial prior for each input and performs inference without learning.

\subsection{Computational complexity}
\label{app:complexity}
Let $B$ be minibatch size, $d$ input dimension, $K$ dictionary width, and $T$ the total number of inner inference steps for an input. A dense SAE encoder costs $O(BdK)$; its training step has the same order, with a larger constant for decoding and backpropagation. \befond applies the lateral operator as $\dec\bPhi^\top[\dec\bPhi\enc\br]$ (with diagonal observation precision inserted when used), so it does not form a dense $K\times K$ Gram matrix. Its costs are:
\begin{equation}
\begin{aligned}
\text{\befond inference:}&\qquad O(TBdK)~,\\
\text{\befond training step:}&\qquad O(TBdK+C_{\mathrm{NGD}}+C_{\mathrm{prior}})~,
\end{aligned}
\end{equation}
where $C_{\mathrm{NGD}}$ includes the moment products and linear solves used by the decoder update, and $C_{\mathrm{prior}}$ is the amortized cost of the separate initial-prior fits.

Solver work depends on minibatch rank, conditioning, and numerical tolerance. \Cref{app:numerical-implementation} describes the matrix-free evaluation. No gradient is propagated through the $T$ inference steps, avoiding storage of that trajectory for backpropagation.

If a conventional SAE update uses $c$ units of dense matrix-multiplication work, its cost is $O(cBdK)$. With $T=O(c)$, \befond's inference contribution has that same order. Our current implementation takes approximately $12$--$14\times$ as long as BatchTopK for the measured training prefix (\cref{app:training-runtime}).

\clearpage
\section{Extended theoretical analyses, synthetic experiments, and ablations}
\label{app:results}
\raggedbottom
\subsection{Experimental setup}
\label{app:methods}

\subsubsection{Datasets}
\label{app:datasets}
\paragraph{SynthSAEBench.}
We measure dictionary and support recovery using the known generating features of \synthsaebench \citep{chanin2026synthsaebench}. The full-width experiment uses $d=768$ and $K=K_\star=16{,}384$. The superposition experiment fixes $K=4{,}096$, $K_\star=16{,}384$, and varies $d$ from 256 to 1,536. Inputs are unwhitened, with parent scaling disabled (\texttt{scale\_children\_by\_parent=false}). Feature probabilities follow a Zipf law with exponent 0.5 over $[0.0005,0.4]$. The generator uses 128 roots, branching factor 4, depth 3, and probability compensation. 

\subsubsection{Training}
\label{app:training}
\befond alternates inference and learning as in \cref{alg:fond}. \Cref{tab:app-training} groups the synthetic settings. Prior-update examples are counted separately from decoder-update examples. Geometry arms are separately tuned. Their early checkpoints include recipe-matched reproductions.
\begin{table}[htbp]
\centering\small
\caption{SynthSAEBench training settings. Decoder and prior examples are counted separately. Full-width runs use $d=768$, $K=K_\star=16{,}384$; superposition runs use $K=4{,}096$.}
\label{tab:app-training}
\label{tab:app-geometry-settings}
\label{tab:app-superposition-settings}
\begin{tabular}{@{}p{0.25\linewidth}p{0.70\linewidth}@{}}
\toprule
Experiment & Settings \\
\midrule
\befond best / continuous & Seed 0; batch 2,048; 88,779 updates; 181.82M decoder and 18.18M prior examples; linear LR $0.02\to10^{-5}$, no warmup; $\beta=1.6$ \\
Inference / decoder & Bernoulli; ETD1 $h=\infty$; one inner step; uniform training depth 1--10, test depth 200. Prior Fisher, current moments, no damping; $\Delta t=1$; bias and variance fixed \\
Initialization / constraints & Gaussian scale 0.1443375673; prior logits $-4.7538364662$; soft decoder norm 3.5 with coefficient 100; prior mean cap $3\times10^{-4}$ and floor $10^{-5}$ \\
Prior / matrix functions & Prior update every 10 batches on 2,048 separate examples; depth $\le200$, KL threshold 0.01, rejection threshold 1, window 10. Taylor limit 32 applications; Chebyshev tolerance $10^{-5}$, maximum degree 2,048 \\
Original baselines & BatchTopK/JumpReLU/Matryoshka: target $L_0=35/45/20$, batch 1,024, LR $3\times10^{-4}$, 200M decoder examples. MP/ReLU-L1: target $L_0=35$, 200,000,512 decoder plus 19M monitoring examples. MF: 20 inference steps in training and testing, batch 1,024, 195,313 updates (200,000,512 examples), cosine schedule with dictionary / other-parameter LR $0.001/0.003$, damping 1 on the first step then 0.2, decoder norm cap 4 \\
Modified baselines & Target $L_0=45/40/35$ for BatchTopK/JumpReLU/Matryoshka; MinFire or resampling and learning-rate decay \\
Training streams & Model seed 0; \befond data seed 1; plain amortized baseline data seed 100101; native MP/ReLU-L1 use the global PyTorch and NumPy streams initialized with seed 0, without a separate training-stream seed; MF uses data seed 0 \\
\midrule
Geometry comparison & Batch 2,048; 100,000 updates; 150,000-step cosine schedule stopped at 100,000; bias and variance fixed. Tuples below are $(\beta,\text{norm},\text{prior cap},\text{peak LR})$ \\
N/E; E/E; N/N & $(1.125,2.5,0.001,0.1)$; $(0.8,3.5,0.01,0.1)$; $(1.125,3.5,0.0002,0.006)$. Natural inference: ETD1 $h=\infty$; Euclidean inference: Euler $h=128$ \\
\midrule
Superposition / lateral & Baseline comparison: initialization seeds 0, 1, 2. Lateral ablations: seed 0. Batch 1,024. Recipe differences are specified in Section~\ref{app:superposition-results} \\
\bottomrule
\end{tabular}
\end{table}

\subsubsection{Evaluation}
\label{app:evaluation}
\paragraph{Synthetic recovery.}
MCC is dictionary mean correlation coefficient, feature-macro F1 uses positive-cosine one-to-one matching to the generating dictionary. and rare F1 averages over its least frequent quartile. A latent is dead when it is inactive on every test example. Activity is $\enc r_i>0.5$ for Bernoulli \befond/MF and positive native output for the SAE baselines. Mean $L_0$ counts active latents per example. Tests use $10^6$ examples with seed 49100, including the MF reevaluation reported in \Cref{tab:synthsae-final}. The independent frequency reference uses $5\times10^5$ examples with seed 315031. \befond/MF use 200/20 test inference steps. Geometry validation uses 32,768 examples, seed 813011, and 100 inference steps.

\subsubsection{Implementation}

\paragraph{Training horizon.}
Each optimizer step draws one inference horizon for the whole minibatch, resets the posterior, runs that many outer steps, and updates the decoder only at the final step. The exponential sampler assigns:
\begin{equation}
 \Pr(t=j)=\frac{e^{-\lambda(j-1)}}{\sum_{s=1}^{t_{\mathrm{outer}}}e^{-\lambda(s-1)}},
 \qquad j=1,\ldots,t_{\mathrm{outer}},
 \label{eq:app-horizon}
\end{equation}
where $\lambda$ is solved numerically so the \textit{truncated discrete} distribution has the configured mean.

Probabilities use FP64. Brent's root solver uses tolerance $10^{-14}$. At mean 1 the distribution is a point mass at 1, and at mean $(t_{\mathrm{outer}}+1)/2$, it is uniform. Sampling uses NumPy's \texttt{default\_rng} seeded by the model seed and zero-based optimizer step, identically on every device and after resumption. The draw does not advance the data or latent-sampling generators.

K1 uses a fixed horizon of one and one inner step. Its initial-prior fit and test inference can still use longer horizons. Uniform sampling on 1--10 has mean 5.5, $\Pr(t=1)=0.1$, and 95th percentile 10. Exponential sampling on 1--50 with the same mean has $\Pr(t=1) \approx 0.182$ and 95th percentile $15$.

Synthetic baseline implementations use SAE Lens 6.51.0 (BatchTopK/JumpReLU/Matryoshka) or 6.37.0 (native MP/ReLU-L1).

\subsection{Inference proofs}
\label{app:proofs}

\subsubsection{Single-layer inference}

\label{app:proof-one-layer-inference}
\begin{proposition}[Two-atom case]
\label{prop:one-layer-inference}
Let $z_i \in \{0,1\}$ with $z_i \sim \operatorname{Bernoulli}(\dec{\pi}_i)$ independently,
with $0<\dec{\pi}_i<1$, and let
$\bx = \dec{\bphi}^{(1)}z_1 + \dec{\bphi}^{(2)}z_2 + \bepsilon$,
where $\bepsilon\sim\cN(\bm{0},\bI_M)$ independently of $\bz$.
If the dictionary atoms are linearly independent and
$\dec{W}_{12}=(\dec{\bphi}^{(1)})^\top\dec{\bphi}^{(2)}\neq0$, then no encoder
$\enc{\hat z}_1(\bx)=f(\enc{\bw}^\top\bx+\enc{b})$, for any scalar function $f$,
can equal $p(z_1=1\mid\bx)$ for every input $\bx$.
\end{proposition}

Assume independent Bernoulli latents with interior prior probabilities and unit observation covariance. Set:
$$
 a_i(\bx) \;=\; \operatorname{logit}(\dec\pi_i)+(\dec\bphi^{(i)})^\top\bx-\tfrac12\|\dec\bphi^{(i)}\|^2,
 \qquad \dec W_{ij} \;=\; (\dec\bphi^{(i)})^\top\dec\bphi^{(j)}.
$$
\begin{proof}[Proof of \Cref{prop:one-layer-inference}]
Marginalizing $z_2$ gives
\begin{align}
 \operatorname{logit}p(z_1=1\mid\bx)
 \;&=\; a_1(\bx)+\log\frac{1+e^{a_2(\bx)-\dec W_{12}}}{1+e^{a_2(\bx)}},\notag\\
 \nabla_{\bx}\operatorname{logit}p(z_1=1\mid\bx)
 \;&=\; \dec\bphi^{(1)}+[\sigma(a_2-\dec W_{12})-\sigma(a_2)]\dec\bphi^{(2)}.
\end{align}
When $\dec W_{12}\ne0$, the bracket is nonconstant. Linear independence of the atoms therefore makes the gradient direction vary with $\bx$. A posterior of the form $f(\enc\bw^\top\bx+\enc b)$ is constant along every direction orthogonal to $\enc\bw$, so its log-odds gradient must be parallel to $\enc\bw$ everywhere. This is a contradiction, even if $f$ is not assumed differentiable.
\end{proof}

\paragraph{Multiple atoms.}
\label{app:proof-one-layer-inference-K}
By assumption, some atom $\dec\bphi^{(i)}$, $i\ge2$, is neither orthogonal to nor a multiple of $\dec\bphi^{(1)}$.
\begin{proof}[Proof of \Cref{prop:one-layer-inference-K}]
Suppose $f(\enc\bw^\top\bx+\enc b)=p(z_1=1\mid\bx)$ for every $\bx$, and write $L(\bx)=\operatorname{logit}p(z_1=1\mid\bx)$. As in the proof of \cref{prop:one-layer-inference}, $\nabla L$ is then parallel to $\enc\bw$ everywhere. Differentiation of the finite partition functions gives:
\begin{equation}
 \nabla L \;=\; \dec\bphi^{(1)} \,+\, \sum_{k\ge2}\Delta_k\dec\bphi^{(k)}~,
 \qquad
 \Delta_k \;=\; \EE[z_k\mid z_1=1,\bx] \,-\, \EE[z_k\mid z_1=0,\bx]~.
 \label{eq:app-K-gradient}
\end{equation}
Along $\bx=t\bd$, with $\bd$ orthogonal to no atom, the fields of all other bits diverge while their couplings remain bounded. Each bit becomes certain independently of $z_1$. Hence, $\Delta_k\to0$ and $\nabla L\to\dec\bphi^{(1)}$. Therefore, $\enc\bw\parallel\dec\bphi^{(1)}$, and $\nabla L^\top\bv=0$ everywhere for every $\bv\perp\dec\bphi^{(1)}$.

Let $Q$ contain the atoms parallel to the chosen atom $i$, writing $\dec\bphi^{(j)}=\lambda_j\dec\bphi^{(i)}$. Choose $\bv\perp\dec\bphi^{(1)}$ with $(\dec\bphi^{(i)})^\top\bv\ne0$, and choose $\bd\perp\dec\bphi^{(i)}$ orthogonal to no atom outside $Q$. Such a $\bd$ exists because only finitely many proper hyperplanes are excluded. Along $\bx=\bx_0+t\bd$, all bits outside $Q\cup\{1\}$ freeze, leaving finite conditional partition functions $Z_0(s),Z_1(s)$ with $s=(\dec\bphi^{(i)})^\top\bx_0$. Thus
\[
 \nabla L^\top\bv\longrightarrow ((\dec\bphi^{(i)})^\top\bv)D'(s),\qquad D(s)=\log\frac{Z_1(s)}{Z_0(s)}.
\]
As $s$ moves from $-\infty$ to $+\infty$, every bit in $Q$ switches according to the sign of $\lambda_j$. Since $\dec W_{1j}=\lambda_j\dec W_{1i}$,
\[
 D(+\infty)-D(-\infty)=-\dec W_{1i}\sum_{j\in Q}|\lambda_j|\ne0.
\]
Therefore $D'$ is nonzero somewhere, contradicting $\nabla L^\top\bv=0$ for a single-projection encoder.
\end{proof}
Orthogonality and parallelism of two atoms are polynomial conditions on $\dec\bPhi$ that fail for almost every dictionary when $M\ge2$, so for almost every dictionary with $K\ge2$ no latent falls under the exception. If no two atoms are parallel, exceptions must be orthogonal to every other atom. For $K>M$, at most $M-1$ atoms can have that property. Whitening extends both results to any fixed positive-definite $\dec\bSigma$, with $\dec\bW=\dec\bPhi^\top\dec\bSigma^{-1}\dec\bPhi$.

\subsubsection{Iterative inference versus single projection encoders}
\label{app:proof-befond-kl}

\begin{proposition}[Fixed-prior inference versus single-projection encoders]
\label{prop:befond-kl}
Consider the Bernoulli--Gaussian model of
\cref{prop:one-layer-inference-K} with unit observation
covariance. Let
\[
\dec\bW=\dec\bPhi^\top\dec\bPhi,
\qquad
\dec{\bm J}=\dec\bW-\operatorname{diag}(\dec\bW),
\]
and suppose that
\[
\lambda_{\min}(\dec{\bm J})>-4
\]
and that some pair of atoms is neither orthogonal nor parallel.
Let $\mathcal A$ contain all encoders
\[
q_f(\bz\mid\bx)
=\prod_i\operatorname{Bernoulli}\bigl(z_i;\,f_i(\enc\bw_i^\top\bx+\enc b_i)\bigr)
\]
with arbitrary weights, biases, and measurable $f_i:\RR\to[0,1]$,
and let
\[
R(q)=\EE_{\bx\sim p(\bx)}\bigl[\KL{q(\cdot\mid\bx)}{p(\cdot\mid\bx)}\bigr].
\]
Let $q_t$ be the factorized posterior with means
\[
\enc\br(t)=\sigma(\enc\bu(t)),
\]
where $\enc\bu(t)$ follows \cref{eq:befond-inference-flow}
with $\beta=1$ and the prior fixed at the generating one,
\[
\dec\bu_0=\operatorname{logit}\dec{\boldsymbol\pi},
\qquad \enc\bu(0)=\dec\bu_0.
\]
Then the optimal factorized posterior
\[
q^*(\cdot\mid\bx)
=\operatorname*{arg\,min}_{q\ \text{factorized}}\KL{q}{p(\cdot\mid\bx)}
\]
is unique and
\[
\lim_{t\to\infty}R(q_t)=R(q^*)<\inf_{q_f\in\mathcal A}R(q_f).
\]
Consequently, for all sufficiently large finite $t$,
\[
R(q_t)<\inf_{q_f\in\mathcal A}R(q_f).
\]
\end{proposition}

\begin{proof}
Zero atoms decouple, since their rows of $\dec{\bm J}$ vanish, so we
drop them. Write
$a_i(\bx)=\operatorname{logit}(\dec\pi_i)+(\dec\bphi^{(i)})^\top\bx-\tfrac12\dec W_{ii}$
and
$\psi(\enc\br)=\sum_i\bigl[\enc r_i\log\enc r_i+(1-\enc r_i)\log(1-\enc r_i)\bigr]$
for the negative Bernoulli entropy. For a factorized $q$ with means
$\enc\br\in[0,1]^K$, using $z_i^2=z_i$, the divergence
$\KL{q}{p(\cdot\mid\bx)}$ equals, up to a term independent of $\enc\br$,
\[
F_{\bx}(\enc\br)=\psi(\enc\br)-\ba(\bx)^\top\enc\br
+\tfrac12\enc\br^\top\dec{\bm J}\enc\br .
\]
Its Hessian
$\operatorname{diag}\bigl(1/(\enc r_i(1-\enc r_i))\bigr)+\dec{\bm J}$
is smallest at $\enc\br=\tfrac12\bm1$, where it equals
$4\bI+\dec{\bm J}$, so $\nabla^2F_{\bx}\succeq4m\,\bI$ with
$m=1+\lambda_{\min}(\dec{\bm J})/4>0$. Hence $F_{\bx}$ has a unique
minimizer $\enc\br^*(\bx)\in(0,1)^K$, which defines $q^*$ and satisfies
$\enc\bu^*:=\operatorname{logit}\enc\br^*=\ba-\dec{\bm J}\enc\br^*$.

\textit{Convergence.} With $\beta=1$ and
$\dec\bu_0=\operatorname{logit}\dec{\boldsymbol\pi}$,
\cref{eq:befond-inference-flow} reads
$\dot{\enc\bu}=\ba-\dec{\bm J}\enc\br-\enc\bu=-\nabla F_{\bx}(\enc\br)$,
a smooth, globally Lipschitz flow in $\enc\bu$. Let
$V(t)=\KL{q_{\enc\br^*}}{q_{\enc\br(t)}}$, the Bregman divergence of
$\psi$. Then
$\dot V=-\langle\nabla F_{\bx}(\enc\br),\enc\br-\enc\br^*\rangle
\le-\bigl[F_{\bx}(\enc\br)-F_{\bx}(\enc\br^*)\bigr]$ by convexity, while
$\tfrac{d}{dt}F_{\bx}(\enc\br)
=-\nabla F_{\bx}^\top\operatorname{diag}\bigl(\enc r_i(1-\enc r_i)\bigr)\nabla F_{\bx}\le0$.
Hence
\[
t\bigl[F_{\bx}(\enc\br(t))-F_{\bx}(\enc\br^*)\bigr]
\le\int_0^t\bigl[F_{\bx}(\enc\br(s))-F_{\bx}(\enc\br^*)\bigr]ds
\le V(0)\le\sum_i-\log\min\{\dec\pi_i,1-\dec\pi_i\}
\]
for every $\bx$. Since
$R(q_t)-R(q^*)=\EE\bigl[F_{\bx}(\enc\br(t))-F_{\bx}(\enc\br^*)\bigr]$,
this gives $R(q_t)\to R(q^*)$ at rate $O(1/t)$.

\textit{$\enc r_i^*$ is not a single-projection function.}

Choose $i,j$ with $\dec W_{ij}\neq0$ and $\dec\bphi^{(i)},\dec\bphi^{(j)}$ not parallel. Implicit differentiation of $\nabla F_{\bx}(\enc\br^*)=\bm0$
gives $\partial\enc\br^*/\partial\bx=\bm H_{\bx}^{-1}\dec\bPhi^\top$ with
$\bm H_{\bx}=\nabla^2F_{\bx}(\enc\br^*)\succeq4m\,\bI$, and
$\partial\enc\bu^*/\partial\bx
=\dec\bPhi^\top-\dec{\bm J}\bm H_{\bx}^{-1}\dec\bPhi^\top$. Both are
bounded. Along a ray $\bx=t\bd$ with $\bd$ orthogonal to no atom,
$\enc\bu^*=\ba-\dec{\bm J}\enc\br^*$ forces every coordinate to
saturate, so $\partial\enc\br^*/\partial\bx
=\operatorname{diag}\bigl(\enc r_k^*(1-\enc r_k^*)\bigr)\,\partial\enc\bu^*/\partial\bx\to\bm0$
and $\nabla_{\bx}\enc u_i^*\to\dec\bphi^{(i)}$. Now let
$Q=\{k:\dec\bphi^{(k)}=\lambda_k\dec\bphi^{(j)}\}$ and take a ray
$\bx=t\bv$ with $\bv\perp\dec\bphi^{(j)}$ orthogonal to no atom outside
$Q$. The coordinates outside $Q$ saturate, while $\enc\br^*_Q$ converges
to the minimizer $\bs\in(0,1)^{Q}$ of the reduced problem, which is
strongly convex because
$\lambda_{\min}(\dec{\bm J}_{QQ})\ge\lambda_{\min}(\dec{\bm J})$. With
$\bm H_Q=\operatorname{diag}\bigl(1/(s_k(1-s_k))\bigr)+\dec{\bm J}_{QQ}\succ0$,
the limiting equations give
$\partial\enc\br^*_Q/\partial\bx\to\bm H_Q^{-1}\blambda(\dec\bphi^{(j)})^\top$,
and since $\dec W_{ik}=\lambda_k\dec W_{ij}$ for $k\in Q$,
\[
\nabla_{\bx}\enc u_i^*(t\bv)\longrightarrow
\dec\bphi^{(i)}-\dec W_{ij}\,\bigl(\blambda^\top\bm H_Q^{-1}\blambda\bigr)\,\dec\bphi^{(j)} .
\]
The coefficient is nonzero, so this limit is not parallel to
$\dec\bphi^{(i)}$. The gradient of $\enc u_i^*$ therefore has no fixed
direction, and neither $\enc u_i^*$ nor $\enc r_i^*$ is a function of a
single projection.

\textit{Positive approximation gap.} Let
$\delta=\inf_{f,\bw,b}\EE\bigl[(f(\bw^\top\bx+b)-\enc r_i^*(\bx))^2\bigr]$
over measurable $f:\RR\to[0,1]$. If $\delta=0$, take
$g_n=f_n(\bw_n^\top\bx+b_n)\to\enc r_i^*$ in $L^2(p)$, with unit
$\bw_n$ (absorbing scale into $f_n$. Constant functions can take any
direction) and, along a subsequence, $\bw_n\to\bw$. Since $p(\bx)$ is
continuous and positive, the convergence also holds in
$L^2_{\mathrm{loc}}$. Each $g_n$ is invariant along $\bw_n^\perp$, so
$\int g_n\,\partial_{\bv_n}\psi_0=0$ for every test function $\psi_0$
and $\bv_n\perp\bw_n$. Taking $\bv_n=(\bI-\bw_n\bw_n^\top)\bv\to\bv$ for
$\bv\perp\bw$ and passing to the limit gives
$\int\enc r_i^*\,\partial_{\bv}\psi_0=0$, so the smooth function
$\enc r_i^*$ is constant along $\bw^\perp$, contradicting the previous
paragraph. Hence $\delta>0$.

\textit{Gap in expected divergence.} For $q_f\in\mathcal A$ with means
$\enc\br_f$, $4m$-strong convexity and the optimality of $\enc\br^*$
give
\[
R(q_f)-R(q^*)
=\EE\bigl[F_{\bx}(\enc\br_f(\bx))-F_{\bx}(\enc\br^*(\bx))\bigr]
\ge2m\,\EE\lVert\enc\br_f-\enc\br^*\rVert_2^2
\ge2m\,\delta>0 .
\]
Together with $R(q_t)\to R(q^*)$, this proves the claim.
\end{proof}

The fixed-reference specialization of \cref{alg:fond}
discretizes this flow. With $\beta=1$ and the
reference at the prior, its inner step is
$\operatorname{logit}\enc\br\gets\operatorname{logit}\enc\br-(1-e^{-h})\nabla F_{\bx}(\enc\br)$,
that is, mirror descent with mirror map $\psi$. Since
$\nabla^2F_{\bx}\preceq\bigl(1+\lambda_{\max}(\dec{\bm J})/4\bigr)\nabla^2\psi$,
the same conclusion holds after finitely many inner steps, at a
geometric rate, whenever $1-e^{-h}\le4/(4+\lambda_{\max}(\dec{\bm J}))$
\citep{lu2018relatively}. This discrete guarantee requires both $\beta=1$ and the stated step-size bound.

The condition $\lambda_{\min}(\dec{\bm J})>-4$ is exactly the condition
for the mean-field objective to be strongly convex, so it has a single
optimum that \befond approaches. A width-independent sufficient condition follows for atoms of equal norm $c$ under unit covariance: $\lambda_{\min}(\dec{\bm J})\ge-c^2$, with equality when $K>M$, so it reduces to $c<2$ in units of the noise standard deviation. For stronger coupling the objective can have several local minima, one per competing explanation, and the proposition does not apply. 

Rolling-reference inference (\cref{app:online-prior-working-memory})
allows the inference budget to be allocated across reference updates.
It includes fixed-reference inference by assigning the entire budget
to the first inner block and returning its output before refreshing
the reference. With a general positive-definite $\dec\bSigma$, whitening gives the same statement with $\dec\bW=\dec\bPhi^\top\dec\bSigma^{-1}\dec\bPhi$ and
$\dec{\bm J}=\dec\bW-\operatorname{diag}(\dec\bW)$.
The strong-convexity condition and discrete step-size bound use this
$\dec{\bm J}$, while the sufficient norm bound applies to whitened atoms
$\dec\bSigma^{-1/2}\dec\bphi^{(i)}$.

\subsubsection{Rain--sprinkler example}
\label{app:explaining-away-example}

\cref{fig:explaining-away} illustrates the two-atom case with independent rain and sprinkler latents, each with prior probability $0.1$. The coordinates of $\bx$ are lawn and roof wetness, with:
$$
\dec{\bPhi}=\begin{pmatrix}2&2\\3&0\end{pmatrix},
\qquad \dec{\bSigma}=(1.05)^2I_2.
$$
Rain contributes to both coordinates, while the sprinkler contributes only to lawn wetness. The atoms are neither orthogonal nor parallel, so \cref{prop:one-layer-inference-K} applies after whitening. Moreover, $\lambda_{\min}(\dec{\bm J})=-4/1.05^2>-4$, satisfying \cref{prop:befond-kl}.

The exact posterior is computed by enumerating the four latent configurations.

\paragraph{Fixed-prior inference.}
The \befond panel shows the unique optimal factorized Bernoulli posterior. We solve the stationary equations from \cref{prop:befond-kl} by fixed-point iteration, initialized at $\enc\br=(0.1,0.1)$, with $\beta=1$ and the generating prior held fixed. The update map has contraction factor $1/1.05^2<1$. Iteration continues until the fixed-point residual divided by $1-1/1.05^2$ is below $2\times10^{-13}$. The panel shows the converged solution of \cref{eq:befond-inference-flow} with the decoder and reference fixed.

\paragraph{Single-projection encoder.}
We numerically fit $\enc r_i(\bx)=f_i(\enc\bw_i^\top\bx)$ to minimize the same model-averaged reverse KL to the exact posterior. Each unit direction is parameterized by an angle. A $12\times12$ angular grid over $[0,\pi)^2$ supplies five starting points for local refinement. For each direction pair, the scalar response functions are solved from the conditional stationarity equations of this KL objective and represented by piecewise-linear probabilities. Final fits use $2{,}001$ grid points per function, extending eight noise standard deviations beyond the projected component means, with constant extrapolation. Conditional expectations during fitting use 160-point Gauss--Hermite quadrature. The fitted rain and sprinkler directions are $65.80^\circ$ and $-20.02^\circ$, respectively.

We evaluate both methods using 160-point Gauss--Hermite quadrature per input coordinate and mixture component. Averaging $D_{\mathrm{KL}}(q(\bz\mid\bx)\|p(\bz\mid\bx))$ over $\bx\sim p(\bx)$ gives $0.008508$ nats for converged fixed-prior \befond and $0.021554$ nats for the fitted single-projection encoder.

The white point has roof wetness $1.5$, halfway between the dry and wet conditional means, and therefore equal roof likelihoods with and without rain. Raising roof wetness to $3$ while holding lawn wetness fixed strengthens rain and suppresses the sprinkler. \Cref{tab:explaining-away} reports the exact and factorized probabilities at these points.

\begingroup
\setlength{\intextsep}{3pt}
\renewcommand{\arraystretch}{0.9}
\begin{table}[H]
    \centering
    \small
    \caption{Posterior activation probabilities at the marked observations in Figure~\ref{fig:explaining-away}, rounded to six decimal places. Neutral and wet correspond to $\bx=(2,1.5)^\top$ and $\bx=(2,3)^\top$, respectively.}
    \label{tab:explaining-away}
    \setlength{\tabcolsep}{6pt}
    \begin{tabular}{lrrrr}
        \toprule
        & \multicolumn{2}{c}{Neutral roof} & \multicolumn{2}{c}{Wet roof} \\
        \cmidrule(lr){2-3}\cmidrule(lr){4-5}
        Method & Rain & Sprinkler & Rain & Sprinkler \\
        \midrule
        Exact posterior & 0.292138 & 0.292138 & 0.960707 & 0.033017 \\
        \befond & 0.228996 & 0.228996 & 0.974111 & 0.019506 \\
        Single-projection encoder & 0.217514 & 0.139225 & 0.958493 & 0.064003 \\
        \bottomrule
    \end{tabular}
\end{table}
\endgroup

\subsection{Inference and learning ablations}
\label{app:inference-learning-ablations}
\subsubsection{Inference ablation: lateral interactions}
\label{app:superposition-results}
\Cref{fig:superposition,fig:superposition-offdiag} use $K=4{,}096$. The baseline comparison uses initialization seeds 0, 1, and 2, with individual results in \Cref{tab:superposition-baselines}. The lateral ablations use seed 0. Metric definitions are in \Cref{app:evaluation}. The lateral ablations use matched learning rates, $\beta$, and decoder norms. Their full \befond control differs from the selected \befond recipe in the baseline comparison at $d=512,384,256$: the ablations use $(\mathrm{LR},\beta)=(0.001,1),(0.001,1),(0.001,2)$, whereas the baseline comparison uses $(0.0003,0.5),(0.0003,1),(0.0003,1.5)$. Excitatory-only runs at $d=1{,}024$ and $1{,}536$ stopped early.
At $d=256$, suppressive interactions raise MCC from $0.2810$ to $0.6862$ within the matched ablation recipes. Excitatory-only inference activates almost all 4,096 latents but reaches MCC $0.2825$. Suppressive-only inference has one dead latent at $d=1{,}536$ and zero at the other dimensions. Both full and zero-off-diagonal inference have zero throughout.

\begin{table}[p]
\centering
\caption{Individual-seed results for Figure~\ref{fig:superposition} ($K=4{,}096$, $K_\star=16{,}384$). MCC values retain the recorded precision. Dead latents are counted on $10^6$ held-out examples using each model's native activity rule ($r_i>0.5$ for \befond and Mean-field + Adam).}
\label{tab:superposition-baselines}
\scriptsize
\setlength{\tabcolsep}{3pt}
\begin{tabular}{@{}lr rrr rrr@{}}
\toprule
& & \multicolumn{3}{c}{MCC} & \multicolumn{3}{c}{Dead latents} \\
\cmidrule(lr){3-5}\cmidrule(lr){6-8}
Method & $d$ & Seed 0 & Seed 1 & Seed 2 & 0 & 1 & 2 \\
\midrule
BatchTopK & 1536 & 0.6447949409484863 & 0.6418958902359009 & 0.6367465853691101 & 178 & 217 & 206 \\
 & 1024 & 0.7172077894210815 & 0.7101552486419678 & 0.715070366859436 & 78 & 80 & 66 \\
 & 768 & 0.7237797379493713 & 0.7247610688209534 & 0.7268125414848328 & 49 & 47 & 39 \\
 & 512 & 0.6935576796531677 & 0.6925675868988037 & 0.6897946000099182 & 37 & 38 & 55 \\
 & 384 & 0.6238665580749512 & 0.6256771087646484 & 0.6227874159812927 & 56 & 82 & 58 \\
 & 256 & 0.4887199401855469 & 0.48809993267059326 & 0.4880616068840027 & 152 & 176 & 157 \\
\midrule
JumpReLU & 1536 & 0.7402234673500061 & 0.7397291660308838 & 0.7378072738647461 & 0 & 0 & 0 \\
 & 1024 & 0.741813063621521 & 0.739402711391449 & 0.7388435006141663 & 3 & 2 & 2 \\
 & 768 & 0.7435916662216187 & 0.7454241514205933 & 0.7452358603477478 & 2 & 0 & 1 \\
 & 512 & 0.7406731843948364 & 0.7403227090835571 & 0.7406840324401855 & 3 & 2 & 1 \\
 & 384 & 0.7123968005180359 & 0.7126669883728027 & 0.712004542350769 & 78 & 90 & 96 \\
 & 256 & 0.5576815605163574 & 0.55924391746521 & 0.5596640110015869 & 1249 & 1232 & 1224 \\
\midrule
Matryoshka & 1536 & 0.6665480136871338 & 0.6840691566467285 & 0.6623464822769165 & 266 & 223 & 279 \\
 & 1024 & 0.7525583505630493 & 0.7481846213340759 & 0.7474654912948608 & 88 & 86 & 85 \\
 & 768 & 0.7555773258209229 & 0.7561160326004028 & 0.7561805248260498 & 54 & 51 & 49 \\
 & 512 & 0.6980507969856262 & 0.6989551186561584 & 0.6928752660751343 & 54 & 87 & 84 \\
 & 384 & 0.6256262063980103 & 0.6234501600265503 & 0.6201913356781006 & 93 & 87 & 85 \\
 & 256 & 0.5091139078140259 & 0.5098989605903625 & 0.5087519288063049 & 170 & 153 & 144 \\
\midrule
Matching Pursuit & 1536 & 0.7041427493095398 & 0.7060015201568604 & 0.7078589200973511 & 0 & 0 & 0 \\
 & 1024 & 0.7029024362564087 & 0.7033072113990784 & 0.6972198486328125 & 0 & 0 & 0 \\
 & 768 & 0.7092912197113037 & 0.7096635699272156 & 0.705768346786499 & 0 & 0 & 0 \\
 & 512 & 0.7090048789978027 & 0.7061803340911865 & 0.7071677446365356 & 0 & 0 & 0 \\
 & 384 & 0.6933020353317261 & 0.6994995474815369 & 0.6963708400726318 & 0 & 0 & 0 \\
 & 256 & 0.6802412867546082 & 0.679160475730896 & 0.6853885054588318 & 0 & 0 & 0 \\
\midrule
ReLU-L1 & 1536 & 0.597491443157196 & 0.5992552042007446 & 0.598943829536438 & 0 & 0 & 0 \\
 & 1024 & 0.59971684217453 & 0.6019572019577026 & 0.6006413698196411 & 0 & 0 & 0 \\
 & 768 & 0.6030566692352295 & 0.6011125445365906 & 0.6026996374130249 & 0 & 0 & 0 \\
 & 512 & 0.5990475416183472 & 0.5983516573905945 & 0.5974198579788208 & 0 & 0 & 0 \\
 & 384 & 0.5921348333358765 & 0.5912116765975952 & 0.590924084186554 & 0 & 0 & 0 \\
 & 256 & 0.5765519738197327 & 0.5788543224334717 & 0.5773016214370728 & 0 & 0 & 0 \\
\midrule
Mean-field + Adam & 1536 & 0.6974799036979675 & 0.7041385769844055 & 0.584987998008728 & 0 & 0 & 291 \\
 & 1024 & 0.7075880169868469 & 0.7071225643157959 & 0.7132526636123657 & 3 & 9 & 2 \\
 & 768 & 0.7050071358680725 & 0.7113231420516968 & 0.7108327746391296 & 8 & 12 & 11 \\
 & 512 & 0.719290018081665 & 0.7244423627853394 & 0.7237823009490967 & 10 & 3 & 7 \\
 & 384 & 0.7211331129074097 & 0.7238309979438782 & 0.7178806662559509 & 14 & 5 & 20 \\
 & 256 & 0.46739497780799866 & 0.4664973020553589 & 0.47486913204193115 & 1645 & 1655 & 1664 \\
\midrule
\befond & 1536 & 0.8128377795219421 & 0.8148884773254395 & 0.8158183097839355 & 0 & 0 & 0 \\
 & 1024 & 0.8047415018081665 & 0.8055942058563232 & 0.8029000759124756 & 0 & 0 & 0 \\
 & 768 & 0.8140117526054382 & 0.8091920614242554 & 0.8121296167373657 & 0 & 0 & 0 \\
 & 512 & 0.812289297580719 & 0.8088412880897522 & 0.8094125986099243 & 0 & 0 & 0 \\
 & 384 & 0.799635648727417 & 0.8007169365882874 & 0.7975541353225708 & 0 & 0 & 0 \\
 & 256 & 0.7783401012420654 & 0.7805794477462769 & 0.7804576754570007 & 0 & 0 & 0 \\
\bottomrule
\end{tabular}
\end{table}

\begin{figure}[H]    
\centering    
\includegraphics[width=\linewidth]{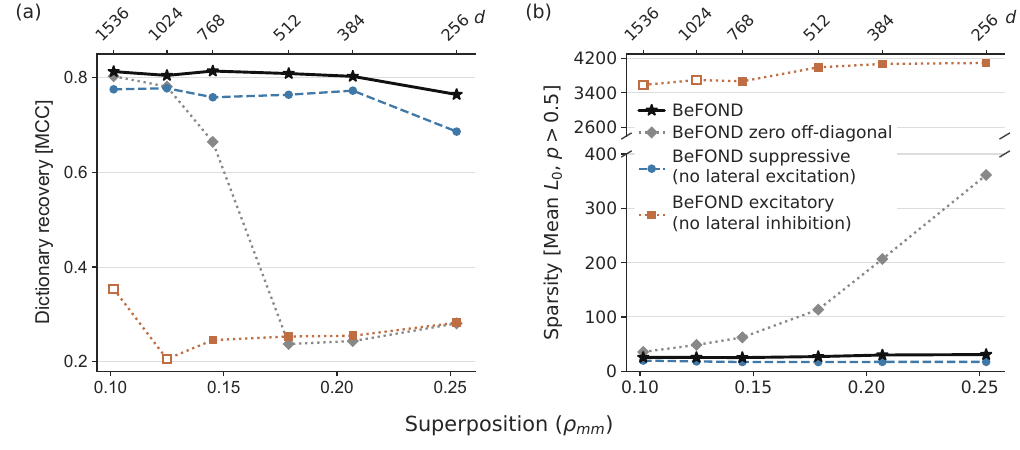}      
\caption{Lateral suppression supports feature recovery under superposition. Hollow markers denote early-stopped excitatory runs ($K=4{,}096$).}    
\label{fig:superposition-offdiag} 
\end{figure}

\subsubsection{Learning ablation: inference and learning geometry}
\label{app:geometry-results}
The Euclidean ablations also reveal a substantial gain in training
efficiency (\cref{fig:euclidean}). Natural inference and
learning (N/N) reach MCC $0.500$ and Rare F1 $0.280$ after just
1,000 updates, exceeding both Euclidean variants after 100,000 updates on their respectively tuned runs.

\begin{figure}[H]    
\centering    
\includegraphics[width=\linewidth]{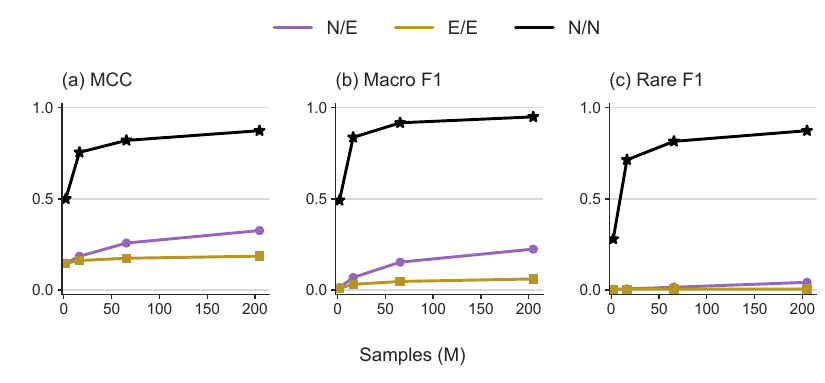}     
\caption{The natural-inference/natural-learning recipe reaches higher recovery with fewer samples. Panels show MCC, Macro F1, and Rare F1 on SynthSAEBench ($d=768$, $K=K_\star=16{,}384$). N/E denotes inference/learning geometry: natural (N) or Euclidean (E). Recipes are tuned separately (\cref{app:training}).}    
\label{fig:euclidean} 
\end{figure}

The natural-inference/natural-learning recipe reaches stronger recovery in fewer decoder updates. Each geometry uses a separately tuned recipe, with settings given in \Cref{tab:app-training}. \Cref{app:learning-geometry} derives equal contraction for frozen, matched second moments and extends the result to unequal moments.

\subsubsection{Rolling prior}
\label{app:sparsity-results}

\cref{fig:synthsae-beta-prior-cap} varies $\beta$ and the cap on the mean initial prior in 72 completed fits: six values of each hyperparameter and seeds 0 and 1. Increasing $\beta$ reduces final activity in this sweep, but also reduces recovery at the larger values. Several looser caps produce overlapping curves because the learned prior lies below those caps. The initial-prior cap constrains the starting activity, while the rolling-reference dynamics determine activity later in the rollout.

\begin{figure}[htbp]
\centering
\includegraphics[width=\linewidth]{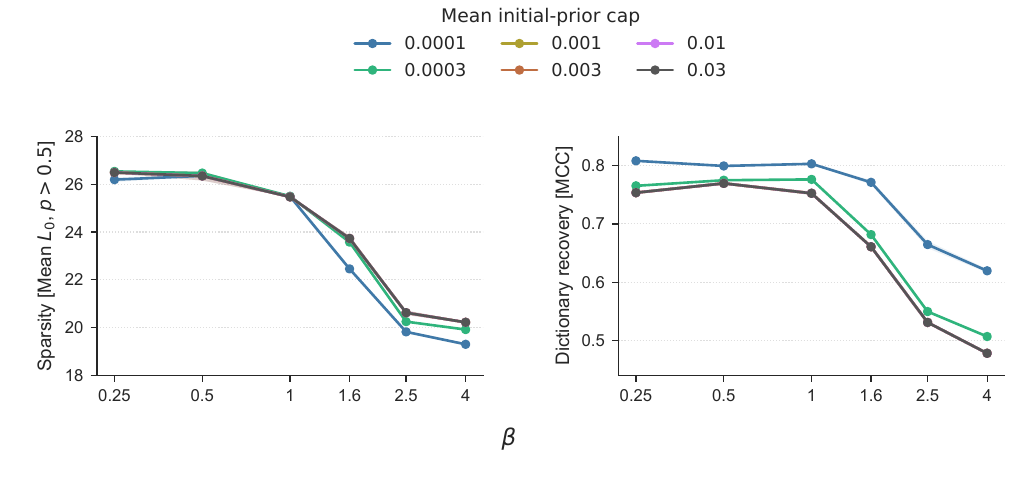}
\caption{Sparsity and recovery versus $\beta$ on \synthsaebench ($d=768$, $K=4{,}096$, $K_\star=16{,}384$). Colors indicate the mean initial-prior cap, lines show the mean over two seeds, and shading shows their range. Each fit uses 12,000 updates and is evaluated on 8,192 examples after 5,000 inference steps.}
\label{fig:synthsae-beta-prior-cap}
\end{figure}
All fits use raw, unwhitened inputs, batch size 1,024, a learning rate decaying from 0.001 to $10^{-5}$, one inner ETD1 step with $h=\infty$, and an exponential training horizon with mean 5.5 capped at 50. The decoder dictionary starts at norm 2 and is clipped at norm 4. The decoder covariance is fixed, decoder bias is learned, and Fisher damping is $10^{-5}$. Initial-prior fits use 100 inference steps on 1,024 examples every 500 updates, with no positive prior floor. Total exposure is 12.3M main plus 24,576 prior examples per fit.

\subsection{Exact and approximate rarity compensation}
\label{app:learning-geometry}
\label{app:proof-exact-rarity-compensation}
Fix a training minibatch and hold its inferred distributions, prior references, and covariance $\dec\bSigma\succ0$ fixed during the decoder update. Throughout this subsection, $\enc\bE\succ0$ and $\dec\bE_p\succ0$ denote the minibatch averages of the posterior and prior second moments in \cref{eq:app-moments}, and $\cF$ denotes the resulting minibatch-averaged free energy.
In Fisher-preconditioned decoder updates, the matrix gradient is implicitly vectorized columnwise before multiplication by $\bG_\Phi^{-1}$, and the result is reshaped to the decoder matrix.
\begin{proof}[Proof of \Cref{thm:exact-rarity-compensation}]
For the unit-covariance likelihood in the theorem, the decoder-dependent terms are
\[
 \cF(\dec\bPhi)
 =\tfrac12\EE_{\mathrm{batch}}\EE_q[\|\bx-\dec\bPhi\bz\|_2^2]+\mathrm{const}.
\]
The posterior and prior references are fixed, so the KL term is constant in $\dec\bPhi$. Expanding the quadratic gives
\begin{align}
 \nabla_{\dec\bPhi}\cF
 &=\dec\bPhi\enc\bE-\EE_{\mathrm{batch}}[\bx\enc\br^\top],\notag\\
 \nabla^2_{\operatorname{vec}(\dec\bPhi)}\cF
 &=\enc\bE\otimes\bI\succ0.
\end{align}
Hence the minimizer is unique,
$\dec\bPhi^*=\EE_{\mathrm{batch}}[\bx\enc\br^\top]\enc\bE^{-1}$,
and $\nabla_{\dec\bPhi}\cF=(\dec\bPhi-\dec\bPhi^*)\enc\bE$.

For a model draw $\tilde{\bx}\mid\bz\sim\mathcal N(\dec\bPhi\bz,\bI)$, the decoder score is
\[
 \nabla_{\dec\bPhi}\log p_{\dec\bPhi}(\tilde{\bx}\mid\bz)
 =(\tilde{\bx}-\dec\bPhi\bz)\bz^\top.
\]
The conditional residual has second moment $\bI$. Averaging the vectorized score's outer product over these model draws and the fixed prior references therefore gives $\bG_\Phi=\dec\bE_p\otimes\bI$. Its inverse acts on a matrix gradient by right multiplication with $\dec\bE_p^{-1}$. Thus, when $\dec\bE_p=\enc\bE$,
\begin{align*}
 \text{Euclidean:}\quad
 \frac{d}{d\tau}(\dec\bPhi-\dec\bPhi^*)
 &=-\eta\ssub{\Phi}(\dec\bPhi-\dec\bPhi^*)\enc\bE,\\
 \text{Natural:}\quad
 \frac{d}{d\tau}(\dec\bPhi-\dec\bPhi^*)
 &=-\eta\ssub{\Phi}(\dec\bPhi-\dec\bPhi^*).
\end{align*}
Integrating the natural flow yields
\begin{equation}
 \dec\bPhi(\Delta t)-\dec\bPhi^*
 =e^{-\eta\ssub{\Phi}\Delta t}(\dec\bPhi(0)-\dec\bPhi^*).
 \label{eq:app-exact-rarity}
\end{equation}

The matching condition is $\dec\bE_p=\enc\bE$ for the minibatch-averaged second moments.

To compare relaxation times, let $\enc\bE\bv_j=\lambda_j\bv_j$. Since $\enc\bE\succ0$, its eigenvectors form an orthonormal basis and $\lambda_j>0$. Along the Euclidean trajectory,
\[
 (\dec\bPhi(\Delta t)-\dec\bPhi^*)\bv_j
 =e^{-\eta\ssub{\Phi}\lambda_j\Delta t}(\dec\bPhi(0)-\dec\bPhi^*)\bv_j.
\]
Its modal relaxation times are $1/(\eta\ssub{\Phi}\lambda_j)$, so the ratio of the slowest to the fastest is $\lambda_{\max}(\enc\bE)/\lambda_{\min}(\enc\bE)=\kappa(\enc\bE)$. Natural flow instead has relaxation time $1/\eta\ssub{\Phi}$ in every direction. Since $z_i^2=z_i$ for Bernoulli latents,
\[
 \rho_i=\enc{\bE}_{ii}
 =\EE_{\mathrm{batch}}\EE_q[z_i]
 =\EE_{\mathrm{batch}}[\enc r_i].
\]
The Rayleigh quotient bounds each diagonal entry between the smallest and largest eigenvalues, giving
\[
 \lambda_{\max}(\enc\bE)\ge\rho_{\max},
 \qquad
 \lambda_{\min}(\enc\bE)\le\rho_{\min},
 \qquad
 \kappa(\enc\bE)\ge\frac{\rho_{\max}}{\rho_{\min}}.
\]
For one feature with known activity, $\enc\bE=p$, recovering the Euclidean rate $\eta\ssub{\Phi} p$ and natural rate $\eta\ssub{\Phi}$ used in the scalar example.
\end{proof}

\begin{corollary}[Power-law feature frequencies]
\label{cor:rarity-width-scaling}
Under \cref{thm:exact-rarity-compensation}, suppose the averaged posterior activation probabilities satisfy $\rho_i=c\,i^{-\alpha}$ for $i=1,\ldots,K$, with $c>0$ and $\alpha>0$. The ratio of the slowest to the fastest Euclidean relaxation time is at least $K^\alpha$; whereas, for the matched natural-gradient flow it is one.
\end{corollary}
\begin{proof}
The largest and smallest frequencies are $c$ and $cK^{-\alpha}$. Substituting their ratio into the theorem gives $\kappa(\enc\bE)\ge K^\alpha$. \cref{eq:app-exact-rarity} gives the common natural-gradient relaxation time.
\end{proof}

With the inferred codes fixed and the averaged second moments matched, natural-gradient flow has the same local contraction rate in every decoder direction.

\paragraph{General observation covariance.}
For fixed $\dec\bSigma\succ0$, the corresponding gradient, Hessian, and decoder Fisher are
\begin{align}
 \nabla_{\dec\bPhi}\cF
 &=\dec\bSigma^{-1}(\dec\bPhi-\dec\bPhi^*)\enc\bE,\notag\\
 \nabla^2_{\operatorname{vec}(\dec\bPhi)}\cF
 &=\enc\bE\otimes\dec\bSigma^{-1},
 \qquad \bG_\Phi=\dec\bE_p\otimes\dec\bSigma^{-1}.
\end{align}
The optimum is unchanged. Euclidean flow now obeys
$\dot{\dec\bPhi}=-\eta\ssub{\Phi}\dec\bSigma^{-1}(\dec\bPhi-\dec\bPhi^*)\enc\bE$;
its Hessian has condition number $\kappa(\enc\bE)\kappa(\dec\bSigma)$ because the Kronecker eigenvalues are products. Natural preconditioning cancels the observation precision and gives
\[
 \frac{d}{d\tau}(\dec\bPhi-\dec\bPhi^*)
 =-\eta\ssub{\Phi}(\dec\bPhi-\dec\bPhi^*)\enc\bE\dec\bE_p^{-1}.
\]
The coefficients are fixed during the update, hence
\begin{equation}
 \dec\bPhi(\Delta t)-\dec\bPhi^*
 =(\dec\bPhi(0)-\dec\bPhi^*)
 \exp[-\eta\ssub{\Phi}\Delta t\,\enc\bE\dec\bE_p^{-1}].
 \label{eq:app-general-contraction}
\end{equation}
When $\dec\bE_p=\enc\bE$, this again reduces to \cref{eq:app-exact-rarity}, independently of the fixed observation covariance. The following extension treats unequal moments.

\begin{proposition}[Rarity compensation with unequal moments]
\label{prop:unequal-rarity-compensation}
Under the same frozen-substep assumptions, let $\enc\bE\succ0$, $\dec\bE_p\succ0$, and $\eta\ssub{\Phi}>0$, without requiring $\enc\bE=\dec\bE_p$. The decoder error obeys \cref{eq:app-general-contraction}. Define
\[
 \lambda_*:=\lambda_{\min}\!\left(\dec\bE_p^{-1/2}\enc\bE\dec\bE_p^{-1/2}\right)>0.
\]
Then, for $\Delta t\ge0$,
\begin{equation}
 \| (\dec\bPhi(\Delta t)-\dec\bPhi^*)\dec\bE_p^{1/2}\|_F
 \le e^{-\eta\ssub{\Phi}\lambda_*\Delta t}
 \|(\dec\bPhi(0)-\dec\bPhi^*)\dec\bE_p^{1/2}\|_F.
 \label{eq:app-unequal-rarity}
\end{equation}
In particular, if
\[
 (1-\epsilon)\dec\bE_p\preceq\enc\bE\preceq(1+\epsilon)\dec\bE_p,\qquad 0\le\epsilon<1,
\]
then
\begin{equation}
 \| (\dec\bPhi(\Delta t)-\dec\bPhi^*)\dec\bE_p^{1/2}\|_F
 \le e^{-\eta\ssub{\Phi}(1-\epsilon)\Delta t}
 \|(\dec\bPhi(0)-\dec\bPhi^*)\dec\bE_p^{1/2}\|_F.
 \label{eq:app-approx-rarity}
\end{equation}
\end{proposition}

\begin{proof}
\cref{eq:app-general-contraction} follows from the quadratic decoder gradient derived above. Multiplying its error dynamics on the right by $\dec\bE_p^{1/2}$ gives:
$$
 \frac{d}{d\tau}\bigl[(\dec\bPhi-\dec\bPhi^*)\dec\bE_p^{1/2}\bigr]
 =-\eta\ssub{\Phi}\bigl[(\dec\bPhi-\dec\bPhi^*)\dec\bE_p^{1/2}\bigr]
 \bigl(\dec\bE_p^{-1/2}\enc\bE\dec\bE_p^{-1/2}\bigr).
$$
The matrix in the final parentheses is symmetric positive definite, so its decay exponential has spectral norm $e^{-\eta\ssub{\Phi}\lambda_*\Delta t}$. Applying $\|AB\|_F\le\|A\|_F\|B\|_2$ proves \cref{eq:app-unequal-rarity}. Under the stated moment bounds, this matrix has eigenvalues in $[1-\epsilon,1+\epsilon]$, hence $\lambda_*\ge1-\epsilon$, which proves \cref{eq:app-approx-rarity}.
\end{proof}

\subsubsection{Activity-normalized inference}
\label{app:activity-normalization-proof}
\begin{theorem}[Activity-normalized inference]
\label{thm:activity-normalized-inference}
Fix the observation, decoder, covariance, and current prior
reference during inference. For the factorized Bernoulli
posterior, let $\enc{r}_i=\sigma(\enc{u}_i)\in(0,1)$.
At the same posterior state, the unit-rate Euclidean and
natural-gradient flows in membrane-potential (logit) coordinates satisfy
$$
\boxed{
\dot{\enc{u}}_i^{\mathrm{NGD}}
=
-\frac{1}{\enc{r}_i(1-\enc{r}_i)}
\frac{\partial\cF}{\partial\enc{u}_i},
\qquad
\dot{\enc{u}}_i^{\mathrm{GD}}
=
\enc{r}_i(1-\enc{r}_i)\,\dot{\enc{u}}_i^{\mathrm{NGD}}.
}
$$
Thus natural-gradient inference removes the multiplicative Bernoulli variance factor from each membrane-potential update.
\end{theorem}

\begin{proof}[Proof of \Cref{thm:activity-normalized-inference}]
For factorized Bernoulli latents, $\enc r_i=\sigma(\enc u_i)\in(0,1)$, so
\[
 \frac{\partial\enc r_i}{\partial\enc u_i}=\enc r_i(1-\enc r_i),
 \qquad \enc\bC=\operatorname{diag}[\enc\br\odot(\bm1-\enc\br)].
\]
The score is $\partial\log q_{\enc\bu}(\bz\mid\bx)/\partial\enc u_i=z_i-\enc r_i$.
Independence therefore gives the posterior Fisher $\boldsymbol G(\enc\bu)=\enc\bC$.
The chain rule yields $\nabla_{\enc\bu}\cF=\enc\bC\nabla_{\enc\br}\cF$.
Natural preconditioning cancels this factor:
\[
 \dot{\enc u}_i^{\mathrm{NGD}}
 =-\frac{1}{\enc r_i(1-\enc r_i)}\frac{\partial\cF}{\partial\enc u_i}
 =-\frac{\partial\cF}{\partial\enc r_i},
 \qquad \dot{\enc u}_i^{\mathrm{GD}}=\enc r_i(1-\enc r_i)\dot{\enc u}_i^{\mathrm{NGD}}.
\]
\end{proof}

When the inference drive favors activation, this cancellation 
removes an additional suppression of weakly active latents'
logit updates. When the inference drive favors activation, this cancellation removes the variance factor from weakly active latents’ logit updates. Activation probabilities evolve according to $\dot{\enc{r}}_i=\enc{r}_i(1-\enc{r}_i)\dot{\enc{u}}_i$.

\subsubsection{Local conditioning under Adam normalization}
\label{app:adam-conditioning}
Adam normalizes updates using coordinatewise estimates of gradient second moments \citep{kingma2015adam,balles2018dissecting}. The following calculation isolates the effect of this normalization on decoder conditioning.

\begin{proposition}[Local conditioning under frozen Adam normalization]
\label{prop:adam-conditioning}
Let $\bx=\dec\bPhi_*\bz+\boldsymbol\xi$, where the known codes have independent coordinates $z_i\sim\operatorname{Bernoulli}(\rho_i)$ with $0<\rho_i<1$, and $\boldsymbol\xi\sim\mathcal N(\bm0,\bI_d)$ is independent of $\bz$. For the unconstrained population loss
\[
 L(\dec\bPhi)=\tfrac12\EE\|\bx-\dec\bPhi\bz\|_2^2,
 \qquad \bE=\EE[\bz\bz^\top],
 \qquad \boldsymbol{H}=\bE\otimes\bI_d,
\]
consider gradients averaged over $B$ independent samples. Freeze the Adam denominator at the square root of its expected second-moment target at $\dec\bPhi_*$, with stabilizer $\epsilon>0$. The resulting diagonal preconditioner is
\[
 \bD_\epsilon
 =\operatorname{diag}_{i=1}^K\!\left[(\epsilon+\sqrt{\rho_i/B})^{-1}\right]\otimes\bI_d.
\]
Writing $\bA_\epsilon=\bD_\epsilon^{1/2}\boldsymbol{H}\bD_\epsilon^{1/2}$, its condition number satisfies
\begin{equation}
 \kappa(\bA_\epsilon)
 \ge
 \frac{\rho_{\max}/(\epsilon+\sqrt{\rho_{\max}/B})}
      {\rho_{\min}/(\epsilon+\sqrt{\rho_{\min}/B})}
 \ge \sqrt{\frac{\rho_{\max}}{\rho_{\min}}}.
 \label{eq:adam-conditioning-bound}
\end{equation}
Thus $\rho_i=c\,i^{-\alpha}$, with $0<c<1$ and $\alpha>0$, gives $\kappa(\bA_\epsilon)\ge K^{\alpha/2}$. Matched complete-data Fisher preconditioning instead gives $\boldsymbol{H}^{-1/2}\boldsymbol{H}\boldsymbol{H}^{-1/2}=\bI$ and condition number one.
\end{proposition}

\begin{proof}
The Hessian is $\boldsymbol{H}$, and $\bE=\operatorname{diag}(\rho_i(1-\rho_i))+\boldsymbol\rho\boldsymbol\rho^\top\succ0$. At the optimum, a single-sample gradient is $-\boldsymbol\xi\bz^\top$, with mean zero and vectorized second moment $\boldsymbol{H}$. Independence gives second moment $\boldsymbol{H}/B$ for the averaged gradient, yielding $\bD_\epsilon$ above. Each diagonal entry of $\bA_\epsilon$ is $a_i=\rho_i/(\epsilon+\sqrt{\rho_i/B})$, repeated $d$ times. Rayleigh quotients give $\kappa(\bA_\epsilon)\ge\max_i a_i/\min_i a_i$. Since $a_i$ increases with $\rho_i$, the first bound follows. For $s=\sqrt{\rho_{\max}/\rho_{\min}}\ge1$ and $t=\sqrt{\rho_{\min}/B}$, the ratio is $s^2(\epsilon+t)/(\epsilon+st)\ge s$. Substitution gives the width bound. The complete-data Fisher equals $\boldsymbol{H}$ for this model, proving the final claim.
\end{proof}

This local conditioning analysis uses known latent codes and fixes Adam's gradient second-moment estimate at its expected value at the decoder optimum, before taking the square root to form the normalization. Adam's normalization depends on gradient variance \citep{balles2018dissecting}, while the model Fisher measures distributional sensitivity \citep{martens2020new}. \Cref{app:complexity} discusses the computational cost of the implemented Fisher updates.

\subsection{Learning and rare-feature recovery}
\label{app:full-width-results}

\subsubsection{Resampling and auxiliary-loss variants}
\label{app:modified-baselines}
Resampling and auxiliary losses aim to revive inactive features \citep{bricken2023monosemanticity,jermyn2024ghost}.
\cref{fig:synthsae-original-modified} and \Cref{tab:synthsae-original-modified} compare the original baselines with modified recipes using MinFire (BatchTopK), resampling (JumpReLU and Matryoshka), learning-rate decay, and retuned sparsity. These modifications substantially improve baseline recovery, but \befond retains the highest MCC, Macro F1, and Rare F1. The comparison evaluates complete recipes with different target sparsities. Measured $L_0$ is reported in the table.

The modified recipes follow the released full-width sweep \citep{chanind2026fullwidth}, using seed 0, 200M samples, batch size 1,024, Adam with learning rate $3\times10^{-4}$ and moments $(0.9,0.999)$, BF16 autocast, gradient-norm clipping at 1, and learning-rate decay during the final 20\% of training. Target $L_0$ is 45 for BatchTopK, 40 for JumpReLU, and 35 for Matryoshka.

MinFire (the sweep implementation's \texttt{MinFireGate} in \texttt{gates.py}) reserves each latent's strongest positive activation before filling the remaining batch-wide TopK budget. The selected BatchTopK recipe applies this gate every batch for the first 90\% of training, then uses ordinary BatchTopK. JumpReLU uses an input straight-through estimator and a base $L_0$ coefficient of 0.1 with an adaptive multiplier. JumpReLU and Matryoshka resample every 2,500 updates during the first 80\% of training, with inactivity windows of 1,000 and 100 updates, respectively. Resampled decoder atoms follow reconstruction-residual directions, and their optimizer moments are reset. BatchTopK and Matryoshka retain auxiliary-loss coefficient 1. Matryoshka uses prefixes 128, 512, 2,048, and 16,384 with per-level auxiliary losses.

Gemma Scope uses reconstruction and $L_0$ sparsity losses, without a separate dead-latent auxiliary loss, alongside straight-through estimators, tied initialization, learning-rate and sparsity warmup, and decoder normalization \citep[Sections~2.2 and~3.2]{lieberum2024gemma}. 

\begin{figure}[htbp]
\centering
\includegraphics[width=\linewidth]{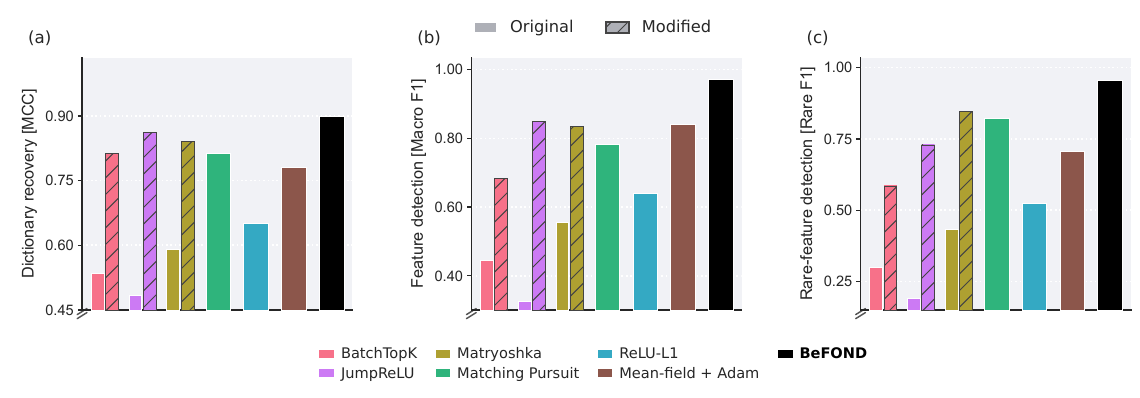}
\caption{\befond retains the highest MCC, Macro F1, and Rare F1 after baseline modifications aimed at preventing inactive features. Hatched bars denote modified recipes ($d=768$, $K=K_\star=16{,}384$, seed 0); numerical results appear in Table~\ref{tab:synthsae-original-modified}.}
\label{fig:synthsae-original-modified}
\end{figure}

\begin{table}[p]
\centering\small
\caption{Full-width SynthSAEBench results ($d=768$, $K=K_\star=16{,}384$). Individual rows report initialization seeds 0--2; mean rows give pointwise 95\% Student-$t$ confidence intervals for the three recovery metrics, as in Figure~\ref{fig:synthsae}. Each evaluation uses $10^6$ held-out examples. $L_0$ and dead counts in mean rows are arithmetic means. The lower block retains the modified recipes ($+$) and continuous-run checkpoint, each at seed 0. The \befond runs use 181.82M decoder-update plus 18.18M prior examples.}
\label{tab:synthsae-final}
\label{tab:synthsae-original-modified}
\label{tab:full-width-dead-latents}
\label{tab:synthsae-full-learning-final}
\setlength{\tabcolsep}{1.9pt}
\renewcommand{\arraystretch}{0.92}
\begin{tabular}{llrrrrr}
\toprule
Method & Seed & MCC $\uparrow$ & Macro F1 $\uparrow$ & Rare F1 $\uparrow$ & $L_0$ & Dead $\downarrow$ \\
\midrule
BatchTopK & 0 & 0.534523 & 0.445173 & 0.299436 & 34.833798 & 1,577 \\
 & 1 & 0.535311 & 0.444452 & 0.299666 & 35.031285 & 1,733 \\
 & 2 & 0.537442 & 0.446642 & 0.296203 & 34.948331 & 1,707 \\
 & Mean & $0.5358\pm 0.0038$ & $0.4454\pm 0.0028$ & $0.2984\pm 0.0048$ & 34.94 & 1672.33 \\
\midrule
JumpReLU & 0 & 0.483953 & 0.326567 & 0.190439 & 45.148172 & 8,069 \\
 & 1 & 0.481728 & 0.326672 & 0.184864 & 44.826161 & 8,090 \\
 & 2 & 0.482102 & 0.325518 & 0.184310 & 45.098772 & 8,117 \\
 & Mean & $0.4826\pm 0.0030$ & $0.3263\pm 0.0016$ & $0.1865\pm 0.0084$ & 45.02 & 8092.00 \\
\midrule
Matryoshka & 0 & 0.590345 & 0.556314 & 0.433018 & 19.978349 & 4,923 \\
 & 1 & 0.574629 & 0.530946 & 0.366127 & 19.991146 & 5,505 \\
 & 2 & 0.583862 & 0.552021 & 0.415913 & 19.989287 & 5,049 \\
 & Mean & $0.5829\pm 0.0196$ & $0.5464\pm 0.0337$ & $0.4050\pm 0.0863$ & 19.99 & 5159.00 \\
\midrule
Matching Pursuit & 0 & 0.812752 & 0.782551 & 0.822595 & 29.646241 & 2 \\
 & 1 & 0.811610 & 0.781264 & 0.821854 & 29.648383 & 1 \\
 & 2 & 0.815241 & 0.784698 & 0.827665 & 29.688899 & 3 \\
 & Mean & $0.8132\pm 0.0046$ & $0.7828\pm 0.0043$ & $0.8240\pm 0.0079$ & 29.66 & 2.00 \\
\midrule
ReLU-L1 & 0 & 0.651464 & 0.641205 & 0.522534 & 35.089073 & 174 \\
 & 1 & 0.650650 & 0.640259 & 0.524647 & 34.972258 & 135 \\
 & 2 & 0.650945 & 0.641409 & 0.524073 & 34.976428 & 137 \\
 & Mean & $0.6510\pm 0.0010$ & $0.6410\pm 0.0015$ & $0.5238\pm 0.0027$ & 35.01 & 148.67 \\
\midrule
Mean-field + Adam & 0 & 0.779715 & 0.841050 & 0.704307 & 45.107146 & 514 \\
 & 1 & 0.783435 & 0.846995 & 0.743286 & 48.116016 & 640 \\
 & 2 & 0.790151 & 0.855224 & 0.729628 & 68.019993 & 556 \\
 & Mean & $0.7844\pm 0.0131$ & $0.8478\pm 0.0177$ & $0.7257\pm 0.0491$ & 53.75 & 570.00 \\
\midrule
\textbf{BeFOND} & 0 & 0.898559 & 0.970345 & 0.955392 & 36.033549 & 1 \\
 & 1 & 0.898064 & 0.970567 & 0.949700 & 35.679190 & 0 \\
 & 2 & 0.897350 & 0.969745 & 0.950726 & 35.485085 & 0 \\
 & Mean & $\mathbf{0.8980}\pm 0.0015$ & $\mathbf{0.9702}\pm 0.0011$ & $\mathbf{0.9519}\pm 0.0075$ & 35.73 & 0.33 \\
\midrule
\multicolumn{7}{l}{\textit{Additional seed-0 checkpoints}} \\
BatchTopK+ & 0 & 0.8123 & 0.6826 & 0.5851 & 45.01 & 504 \\
JumpReLU+ & 0 & 0.8612 & 0.8487 & 0.7288 & 39.96 & 367 \\
Matryoshka+ & 0 & 0.8405 & 0.8348 & 0.8457 & 34.95 & 495 \\
\befond & 0 & 0.8974 & 0.9701 & 0.9482 & 35.70 & 0 \\
\bottomrule
\end{tabular}
\end{table}

\paragraph{Recovery across feature frequencies.}
\cref{fig:synthsae-frequency-quintiles} resolves the comparison into frequency quintiles. The \befond checkpoint is the best run in \Cref{tab:synthsae-final} (MCC 0.8986), and the original and modified baseline checkpoints are those in \cref{fig:synthsae-original-modified}. Improvements from the modified recipes extend into the rare tail, while \befond retains higher coverage and F1 in every quintile. Coverage measures dictionary alignment. F1 measures detection of the matched feature's activity, so the two need not coincide.

\begin{figure}[htbp]
\centering
\includegraphics[width=\linewidth]{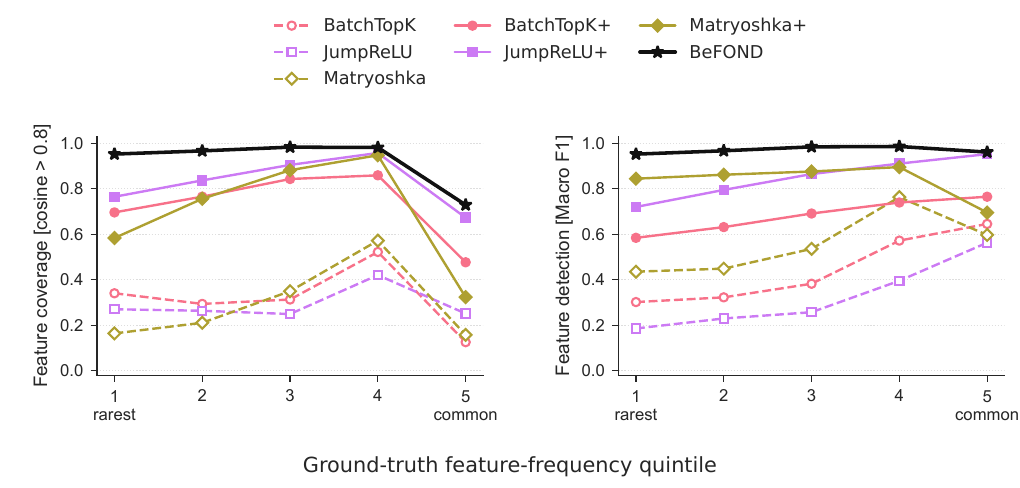}
\caption{Recovery across feature-frequency quintiles ($d=768$, $K=K_\star=16{,}384$, seed 0). Dashed lines show original baselines, and ``+'' denotes their modified recipes. Coverage is the fraction of features matched above cosine 0.8. Quintiles use each run's 500K-example frequency reference. F1 uses 1M held-out examples. Random draws differ between some runs.}
\label{fig:synthsae-frequency-quintiles}
\end{figure}

These fifths differ from Rare F1's least-frequent quarter.

\paragraph{Dictionary capacity.}
\label{app:dead-latents}
\Cref{tab:synthsae-final} combines recovery, sparsity, and dead-latent counts. Main-text \cref{fig:synthsae} reports three-seed means: MCC 0.8980 and Rare F1 0.9519, with 0.33 dead latents on average. Its seed-0 run reaches MCC 0.8986 and Rare F1 0.9554, with one dead latent. The separate continuous run below reaches MCC 0.8974 and Rare F1 0.9482, with zero. MP has two dead latents but lower recovery, so keeping latents active is not sufficient for accurate feature recovery.

\subsubsection{Training sample efficiency}
\label{app:learning-supplemental-results}
\cref{fig:synthsae-full-learning-curves} reports MCC, Macro F1, and Rare F1 throughout training in the full-width setting. Its \befond curve follows one continuous run, whose endpoint is identified separately in \Cref{tab:synthsae-final}.
\begin{figure}[H]
\centering
\includegraphics[width=\linewidth]{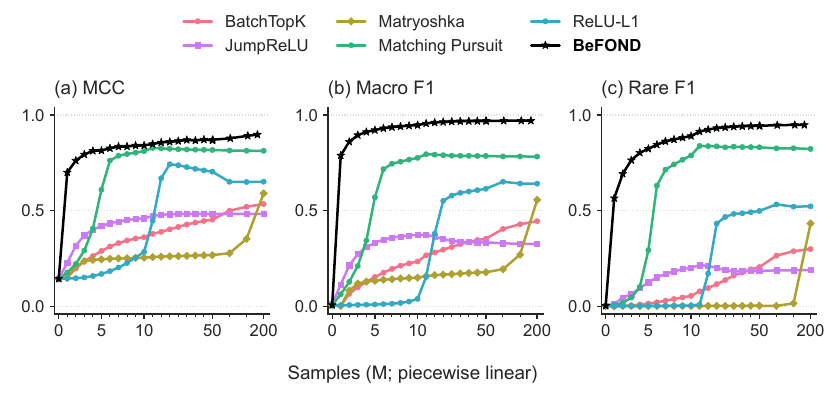}
\caption{Feature recovery throughout training on unwhitened SynthSAEBench ($d=768$; 16,384 learned and ground-truth features), using plain baseline recipes and seed 0. (a) Dictionary mean correlation coefficient (MCC). (b) Macro F1 over ground-truth features after positive-cosine one-to-one matching. (c) Rare F1 over the least frequent quartile, defined by a shared independent 500K-example frequency reference (seed 315031). Every checkpoint is evaluated on 1M held-out examples (seed 49100). BeFOND uses 200 inference steps with activity $p>0.5$; baselines use positive native activations. The x-axis counts decoder-update examples and is piecewise linear: 1M examples at 0--10M, 5M at 10--50M, and 25M beyond 50M occupy equal horizontal distances. Cross-region slopes therefore do not have a common per-sample interpretation. Baseline points through 10M, after 10M through 50M, and beyond 50M come from separate matching-recipe runs with their original full-budget learning-rate schedules; they do not form a single continuous trajectory. Overlapping checkpoints are selected by these fixed windows, not by score. BeFOND uses one continuous run. Samples are rounded up to full batches. Baselines finish at approximately 200M decoder-update examples and BeFOND at 181.82M plus 18.18M separate prior-update examples, which are excluded from the x-axis. Baseline auxiliary exposure is not counted on this axis. Final measured $L_0$ values are BatchTopK: 34.83, JumpReLU: 45.15, Matryoshka: 19.98, Matching Pursuit: 29.65, ReLU-L1: 35.09, BeFOND: 35.70. One seed is shown; no uncertainty intervals are estimated.}
\label{fig:synthsae-full-learning-curves}
\end{figure}

\section{Extended Gemma experiments}
\label{app:gemma-full-results}

\begin{table*}[t]
\centering
\caption{\saebench results on whitened Gemma-2-2B layer-12 activations after $49{,}999{,}872$ training tokens. Entries are mean $\pm$ sample standard deviation over three training seeds for \befond and five for each baseline. For SAE baselines, we selected target $L_0$ values by mean top-1 accuracy, and then used it for all probing columns, and the model selected by mean disentanglement supplies all RAVEL columns. Parentheses give both values when they differ. Sparse probing measures how well $k=1$ or $k=5$ features predict concepts, averaged equally across 113 tasks: \textbf{accuracy} is the fraction of correct predictions; \textbf{AUC} measures how well positives rank above negatives across thresholds; \textbf{F1} balances precision (correctness of positive predictions) and recall (coverage of positives). RAVEL measures whether interventions change a targeted attribute (\textbf{cause}) while preserving other attributes (\textbf{isolation}), and \textbf{disentanglement} is their mean. Higher is better throughout. Within each width and metric, bold marks the highest mean, selected before rounding.}
\label{tab:gemma}
\begingroup
\definecolor{saebenchBatchTopK}{HTML}{f77189}
\definecolor{saebenchMatryoshka}{HTML}{aea031}
\definecolor{saebenchMF}{HTML}{8c564b}
\definecolor{saebenchBeFOND}{HTML}{000000}
\setlength{\tabcolsep}{3pt}
\renewcommand{\arraystretch}{1.12}
\resizebox{\textwidth}{!}{%
\begin{tabular}{@{}l*{9}{r}@{}}
\toprule
& \multicolumn{6}{c}{Sparse probing} & \multicolumn{3}{c}{RAVEL} \\
\cmidrule(lr){2-7}\cmidrule(lr){8-10}
Model & \multicolumn{3}{c}{$k=1$} & \multicolumn{3}{c}{$k=5$} & & & \\
\cmidrule(lr){2-4}\cmidrule(lr){5-7}
& Acc. $\uparrow$ & AUC $\uparrow$ & F1 $\uparrow$ & Acc. $\uparrow$ & AUC $\uparrow$ & F1 $\uparrow$ & \makecell{Disentangle-\\ment $\uparrow$} & Cause $\uparrow$ & Isolation $\uparrow$ \\
\midrule
\rowcolor{black!6}
\multicolumn{10}{l}{\textbf{Dictionary width: 16k}} \\
\addlinespace[2pt]
\makecell[l]{\textcolor{saebenchBatchTopK}{\textbf{BatchTopK}}\\{\scriptsize (target $L_0=400$)}} & \entry{.767}{.002} & \entry{\mathbf{.798}}{.003} & \entry{.761}{.004} & \entry{\mathbf{.814}}{.002} & \entry{\mathbf{.862}}{.002} & \entry{\mathbf{.812}}{.002} & \entry{.729}{.005} & \entry{.692}{.019} & \entry{.767}{.017} \\[2pt]
\makecell[l]{\textcolor{saebenchMatryoshka}{\textbf{Matryoshka BatchTopK}}\\{\scriptsize ($L_0=200$ probing; $400$ RAVEL)}} & \entry{.768}{.003} & \entry{.794}{.003} & \entry{.762}{.002} & \entry{.810}{.002} & \entry{.858}{.002} & \entry{.808}{.002} & \entry{.731}{.003} & \entry{.695}{.013} & \entry{.767}{.015} \\[2pt]
\makecell[l]{\textcolor{saebenchMF}{\textbf{Mean-field + Adam}}\\{\scriptsize ($T=20$ inference steps)}} & \entry{.761}{.003} & \entry{.761}{.003} & \entry{.749}{.004} & \entry{.801}{.002} & \entry{.842}{.001} & \entry{.798}{.002} & \entry{.696}{.003} & \entry{.639}{.006} & \entry{.753}{.006} \\[2pt]
\makecell[l]{\textcolor{saebenchBeFOND}{\textbf{\befond}}\\{\scriptsize ($T=10$ inference steps)}} & \entry{\mathbf{.771}}{.004} & \entry{.781}{.004} & \entry{\mathbf{.765}}{.005} & \entry{.812}{.002} & \entry{.856}{.001} & \entry{.811}{.002} & \entry{\mathbf{.786}}{.002} & \entry{\mathbf{.759}}{.002} & \entry{\mathbf{.813}}{.003} \\[2pt]
\midrule
\rowcolor{black!6}
\multicolumn{10}{l}{\textbf{Dictionary width: 65k}} \\
\addlinespace[2pt]
\makecell[l]{\textcolor{saebenchBatchTopK}{\textbf{BatchTopK}}\\{\scriptsize ($L_0=400$ probing; $800$ RAVEL)}} & \entry{.770}{.003} & \entry{.800}{.002} & \entry{.765}{.003} & \entry{\mathbf{.814}}{.003} & \entry{\mathbf{.862}}{.002} & \entry{\mathbf{.812}}{.003} & \entry{.745}{.005} & \entry{.712}{.006} & \entry{.777}{.008} \\[2pt]
\makecell[l]{\textcolor{saebenchMatryoshka}{\textbf{Matryoshka BatchTopK}}\\{\scriptsize (target $L_0=400$)}} & \entry{.772}{.001} & \entry{\mathbf{.801}}{.001} & \entry{.766}{.002} & \entry{.811}{.001} & \entry{.858}{.001} & \entry{.809}{.002} & \entry{.737}{.004} & \entry{.704}{.007} & \entry{.771}{.006} \\[2pt]
\makecell[l]{\textcolor{saebenchMF}{\textbf{Mean-field + Adam}}\\{\scriptsize ($T=20$ inference steps)}} & \entry{.757}{.002} & \entry{.757}{.002} & \entry{.744}{.002} & \entry{.794}{.001} & \entry{.832}{.001} & \entry{.790}{.002} & \entry{.733}{.003} & \entry{.697}{.007} & \entry{.769}{.006} \\[2pt]
\makecell[l]{\textcolor{saebenchBeFOND}{\textbf{\befond}}\\{\scriptsize ($T=10$ inference steps)}} & \entry{\mathbf{.776}}{.004} & \entry{.794}{.004} & \entry{\mathbf{.769}}{.006} & \entry{.812}{.002} & \entry{.856}{.001} & \entry{.811}{.002} & \entry{\mathbf{.814}}{.000} & \entry{\mathbf{.789}}{.001} & \entry{\mathbf{.840}}{.001} \\[2pt]
\midrule
\rowcolor{black!6}
\multicolumn{10}{l}{\textbf{Dictionary width: 131k}} \\
\addlinespace[2pt]
\makecell[l]{\textcolor{saebenchBatchTopK}{\textbf{BatchTopK}}\\{\scriptsize (target $L_0=600$)}} & \entry{.773}{.002} & \entry{\mathbf{.804}}{.002} & \entry{.768}{.003} & \entry{.814}{.002} & \entry{\mathbf{.862}}{.001} & \entry{.813}{.002} & \entry{.751}{.003} & \entry{.720}{.008} & \entry{.782}{.013} \\[2pt]
\makecell[l]{\textcolor{saebenchMatryoshka}{\textbf{Matryoshka BatchTopK}}\\{\scriptsize ($L_0=200$ probing; $400$ RAVEL)}} & \entry{.773}{.004} & \entry{.799}{.004} & \entry{.768}{.004} & \entry{.814}{.001} & \entry{.860}{.002} & \entry{.812}{.001} & \entry{.739}{.002} & \entry{.713}{.008} & \entry{.764}{.006} \\[2pt]
\makecell[l]{\textcolor{saebenchMF}{\textbf{Mean-field + Adam}}\\{\scriptsize ($T=20$ inference steps)}} & \entry{.754}{.003} & \entry{.754}{.003} & \entry{.740}{.005} & \entry{.791}{.002} & \entry{.830}{.003} & \entry{.787}{.002} & \entry{.734}{.005} & \entry{.694}{.007} & \entry{.774}{.004} \\[2pt]
\makecell[l]{\textcolor{saebenchBeFOND}{\textbf{\befond}}\\{\scriptsize ($T=10$ inference steps)}} & \entry{\mathbf{.785}}{.002} & \entry{.802}{.003} & \entry{\mathbf{.778}}{.003} & \entry{\mathbf{.816}}{.001} & \entry{.857}{.002} & \entry{\mathbf{.814}}{.002} & \entry{\mathbf{.816}}{.001} & \entry{\mathbf{.790}}{.002} & \entry{\mathbf{.842}}{.003} \\[2pt]
\midrule
\rowcolor{black!6}
\multicolumn{10}{l}{\textbf{Dictionary width: 524k}} \\
\addlinespace[2pt]
\makecell[l]{\textcolor{saebenchBatchTopK}{\textbf{BatchTopK}}\\{\scriptsize (target $L_0=800$)}} & \entry{.773}{.004} & \entry{.803}{.004} & \entry{.768}{.006} & \entry{.814}{.003} & \entry{\mathbf{.860}}{.002} & \entry{.813}{.003} & \entry{.748}{.004} & \entry{.719}{.010} & \entry{.776}{.010} \\[2pt]
\makecell[l]{\textcolor{saebenchMatryoshka}{\textbf{Matryoshka BatchTopK}}\\{\scriptsize ($L_0=200$ probing; $400$ RAVEL)}} & \entry{.774}{.002} & \entry{.799}{.002} & \entry{.768}{.002} & \entry{.813}{.001} & \entry{.859}{.001} & \entry{.812}{.000} & \entry{.743}{.003} & \entry{.708}{.013} & \entry{.778}{.009} \\[2pt]
\makecell[l]{\textcolor{saebenchMF}{\textbf{Mean-field + Adam}}\\{\scriptsize ($T=20$ inference steps)}} & \entry{.751}{.007} & \entry{.751}{.007} & \entry{.735}{.009} & \entry{.786}{.009} & \entry{.825}{.010} & \entry{.782}{.009} & \entry{.685}{.029} & \entry{.622}{.047} & \entry{.748}{.013} \\[2pt]
\makecell[l]{\textcolor{saebenchBeFOND}{\textbf{\befond}}\\{\scriptsize ($T=10$ inference steps)}} & \entry{\mathbf{.789}}{.002} & \entry{\mathbf{.805}}{.004} & \entry{\mathbf{.784}}{.004} & \entry{\mathbf{.817}}{.001} & \entry{.859}{.001} & \entry{\mathbf{.816}}{.001} & \entry{\mathbf{.827}}{.001} & \entry{\mathbf{.800}}{.001} & \entry{\mathbf{.854}}{.002} \\[2pt]
\bottomrule
\end{tabular}%
}
\endgroup
\end{table*}

\begin{figure}[t]
    \centering
    \includegraphics[width=\linewidth]{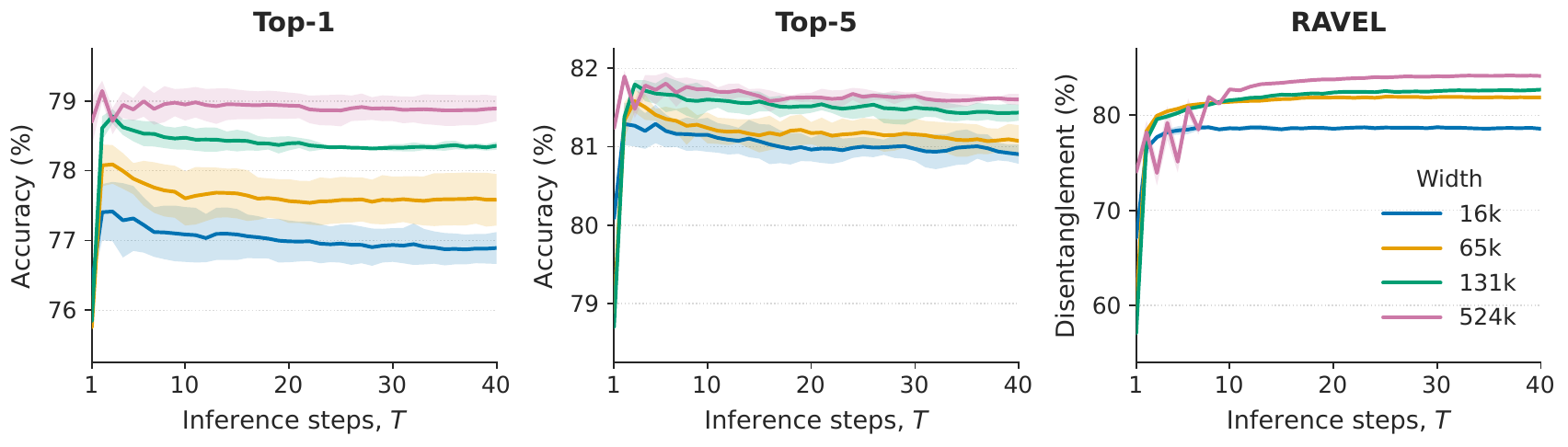}
    \caption{
        \textbf{Performance as a function of inference depth.}
        Sparse probing top-1 accuracy (left), top-5 accuracy (middle), and RAVEL disentanglement (right) for \befond on Gemma-2-2B layer-12 activations, evaluated over $T=1,\ldots,40$ inference steps at dictionary widths of $16$k, $65$k, $131$k, and $524$k. All evaluations use checkpoint $12{,}207$, corresponding to
        $49{,}999{,}872$ training tokens. Lines show means and shaded regions indicate $\pm 1$ sample
        standard deviation across three training seeds. 
        Sparse probing peaks within the first few steps, whereas RAVEL generally benefits from longer inference.
    }
    \label{fig:gemma-inference-depth}
\end{figure}

\subsection{Experimental setup}
\paragraph{Dataset.}
We extract layer-12 post-residual activations from Gemma-2-2B \citep{gemmateam2024gemma2} ($d=2{,}304$). Documents from Pile Uncopyrighted \citep{monology2023pile} are tokenized separately with BOS and truncated to 1,024 tokens. BOS and padding positions are excluded. Validation, calibration, and training use disjoint document ranges. Raw and whitened activations are evaluated separately.
Validation and calibration each use 65,536 tokens, with up to 500M training tokens. Extraction uses BF16, batch 4, and no document packing. Whitening statistics were fitted on 500M training tokens.

\paragraph{Training.}
\Cref{tab:app-gemma-training} summarizes the training and evaluation settings. The selected \befond results in \Cref{tab:gemma} use seeds 0--2 and checkpoint 12,207, corresponding to 49,999,872 decoder-training tokens. Evaluation uses ten outer inference steps with one inner step per reference update. The readout retains posterior means above 0.5 and sets the remainder to zero, followed by decoder-norm rescaling. \par At widths 16K, 65K, and 131K, the recipes use $\beta=1$, ETD1 step $h=1.5$, reference retention $0.999$, and decoder-norm targets $4$, $1.5$, and $0.75$, respectively. At 524K they use $\beta=1.5$, $h=\infty$, full reference retention, and norm target $0.5$. All use a global batch of 4,096 and a linear learning-rate schedule with peak $10^{-3}$ and minimum $10^{-5}$. Warmup occupies 5\% of training at 16K--131K and 1\% at 524K.
\begin{table}[htbp]
\centering\small
\caption{Gemma training and evaluation settings. Preprocessing, checkpoint, inference depth, and readout identify an evaluation configuration.}
\label{tab:app-gemma-training}
\label{tab:app-evaluation}
\begin{tabular}{@{}p{0.25\linewidth}p{0.70\linewidth}@{}}
\toprule
Method / metric & Settings \\
\midrule
\befond & Widths 16,384--524,288; exponentially sampled training depth (mean 16, cap 100; Appendix~\ref{app:methods}); FP32 inference/parameters; prior Fisher with current moments and no damping; $\Delta t=1$; decoder solver \texttt{matrix\_free\_phi1\_auto} \\
Mean-field + Adam & Whitened inputs, 50M tokens, 20 inference steps, readout $p>0.5$. After 2,000 updates, clip the dictionary gradient at three times its running-mean norm and skip batches whose negative ELBO exceeds five times its running mean, capped at approximately 20\% of recent batches \\
BatchTopK / Matryoshka & Whitened inputs; batch 4,096; 12,207 updates (49,999,872 tokens); constant LR $3\times10^{-4}$; FP32 SAEs, BF16 Gemma; SAELens 6.51.1. ZCA fitted on 500M training tokens. Target $L_0$ values in Table~\ref{tab:sae-bench-whitened} were selected on a separate pilot \\
Baseline initialization\newline / losses & Decoder norm initialized to 0.1; decoder bias subtracted before encoding; activations rescaled by decoder norm; auxiliary-loss coefficient 1; threshold EMA rate 0.01 \\
Matryoshka levels & 4,096, 16,384, and the full width (16,384 appears once for the smallest model); Matryoshka auxiliary loss enabled \\
Gemma Scope & Release \texttt{google/gemma-scope-2b-pt-res}, checkpoint paths \texttt{layer\_12/width\_\{w\}/average\_l0\_\{s\}}, with $(w,s)$ given by each row of Table~\ref{tab:gemma-scope-l12}. Training: 4B tokens at 16K, 8B at 32K--524K, and 16B at 1M~\citep{lieberum2024gemma} \\
\midrule
Sparse probing & 113 tasks; final non-padding token; context 1,024; 1/2/5 selected features; L1 logistic regression; equal task averaging \\
Probe selection & Rank absolute normalized class-mean difference; 10 inverse-penalty values from $10^5$ to $10^{-5}$; validation ROC-AUC; refit on training data \\
RAVEL & Average cause and isolation over attributes within each entity, then equally over entities; disentanglement is their mean. Aggregation requires completed results for every configured entity \\
\bottomrule
\end{tabular}
\end{table}

\paragraph{Evaluation.}
Sparse probing reports mean test accuracy across 113 tasks, using the best 1 or 5 features selected on training data. RAVEL reports cause, isolation, and their mean~\citep{huang2024ravel}. \Cref{tab:app-evaluation} specifies the evaluation settings \citep{karvonen2025saebench}.

\paragraph{Implementation.}
The Gemma reference environment uses Python 3.11.14, PyTorch 2.7.1+cu126, TransformerLens 2.16.1, Transformers 4.57.3, and SAE Lens 6.24.6.

\subsection{Results}
The whitened baseline comparison reports sparse probing and RAVEL in \Cref{tab:sae-bench-whitened}, and \Cref{tab:gemma-scope-l12} reports the Gemma Scope checkpoints.

\begin{table*}[t]
\centering
\caption{SAEBench results for SAEs trained on 50M tokens of whitened Gemma-2-2B layer-12 activations. Entries report mean $\pm$ sample standard deviation over five training seeds (1--5), evaluated at the final checkpoint (49,999,872 tokens). Sparse-probing metrics are macro-averaged over 113 tasks; RAVEL overall is the disentanglement score. All scores are fractions, and higher is better. $L_0$ denotes target training sparsity, fixed using a separate seed-0 pilot. Bold indicates the highest mean across $L_0$ values within each width--architecture group, selected before rounding.
}
\label{tab:sae-bench-whitened}
\begingroup
\setlength{\tabcolsep}{2pt}
\resizebox{\textwidth}{!}{%
\begin{tabular}{lr*{9}{r}}
\toprule
& & \multicolumn{6}{c}{Sparse probing \citep{kantamneni2025sparse}} & \multicolumn{3}{c}{RAVEL \citep{huang2024ravel}} \\
\cmidrule(lr){3-8}\cmidrule(lr){9-11}
Width & Target $L_0$ & \multicolumn{3}{c}{$k=1$} & \multicolumn{3}{c}{$k=5$} & & & \\
\cmidrule(lr){3-5}\cmidrule(lr){6-8}
& & Acc. $\uparrow$ & AUC $\uparrow$ & F1 $\uparrow$ & Acc. $\uparrow$ & AUC $\uparrow$ & F1 $\uparrow$ & Overall $\uparrow$ & Cause $\uparrow$ & Isolation $\uparrow$ \\
\midrule
\multicolumn{11}{l}{\textit{BatchTopK (16k)}} \\
16k & 200 & \entry{.765}{.004} & \entry{.791}{.004} & \entry{.758}{.005} & \entry{.810}{.002} & \entry{.859}{.002} & \entry{.808}{.003} & \entry{.717}{.009} & \entry{\mathbf{.698}}{.018} & \entry{.736}{.021} \\
16k & 400 & \entry{\mathbf{.767}}{.002} & \entry{\mathbf{.798}}{.003} & \entry{\mathbf{.761}}{.004} & \entry{\mathbf{.814}}{.002} & \entry{\mathbf{.862}}{.002} & \entry{\mathbf{.812}}{.002} & \entry{\mathbf{.729}}{.005} & \entry{.692}{.019} & \entry{\mathbf{.767}}{.017} \\
\midrule
\multicolumn{11}{l}{\textit{Matryoshka BatchTopK (16k)}} \\
16k & 200 & \entry{\mathbf{.768}}{.003} & \entry{.794}{.003} & \entry{\mathbf{.762}}{.002} & \entry{\mathbf{.810}}{.002} & \entry{\mathbf{.858}}{.002} & \entry{\mathbf{.808}}{.002} & \entry{.718}{.004} & \entry{.688}{.018} & \entry{.747}{.013} \\
16k & 400 & \entry{.767}{.001} & \entry{\mathbf{.797}}{.001} & \entry{.761}{.002} & \entry{.809}{.002} & \entry{.857}{.001} & \entry{.808}{.002} & \entry{\mathbf{.731}}{.003} & \entry{\mathbf{.695}}{.013} & \entry{\mathbf{.767}}{.015} \\
\midrule
\multicolumn{11}{l}{\textit{BatchTopK (65k)}} \\
65k & 400 & \entry{\mathbf{.770}}{.003} & \entry{.800}{.002} & \entry{\mathbf{.765}}{.003} & \entry{\mathbf{.814}}{.003} & \entry{\mathbf{.862}}{.002} & \entry{\mathbf{.812}}{.003} & \entry{.740}{.003} & \entry{.700}{.015} & \entry{\mathbf{.781}}{.012} \\
65k & 800 & \entry{.768}{.002} & \entry{\mathbf{.803}}{.002} & \entry{.764}{.002} & \entry{.813}{.001} & \entry{.860}{.001} & \entry{.812}{.001} & \entry{\mathbf{.745}}{.005} & \entry{\mathbf{.712}}{.006} & \entry{.777}{.008} \\
\midrule
\multicolumn{11}{l}{\textit{Matryoshka BatchTopK (65k)}} \\
65k & 200 & \entry{.772}{.002} & \entry{.798}{.002} & \entry{\mathbf{.766}}{.002} & \entry{\mathbf{.813}}{.001} & \entry{\mathbf{.861}}{.002} & \entry{\mathbf{.811}}{.001} & \entry{.734}{.009} & \entry{\mathbf{.707}}{.015} & \entry{.761}{.010} \\
65k & 400 & \entry{\mathbf{.772}}{.001} & \entry{\mathbf{.801}}{.001} & \entry{.766}{.002} & \entry{.811}{.001} & \entry{.858}{.001} & \entry{.809}{.002} & \entry{\mathbf{.737}}{.004} & \entry{.704}{.007} & \entry{\mathbf{.771}}{.006} \\
\midrule
\multicolumn{11}{l}{\textit{BatchTopK (131k)}} \\
131k & 600 & \entry{\mathbf{.773}}{.002} & \entry{\mathbf{.804}}{.002} & \entry{\mathbf{.768}}{.003} & \entry{\mathbf{.814}}{.002} & \entry{\mathbf{.862}}{.001} & \entry{\mathbf{.813}}{.002} & \entry{\mathbf{.751}}{.003} & \entry{\mathbf{.720}}{.008} & \entry{\mathbf{.782}}{.013} \\
131k & 800 & \entry{.770}{.003} & \entry{.803}{.003} & \entry{.765}{.003} & \entry{.814}{.002} & \entry{.859}{.002} & \entry{.812}{.002} & \entry{.746}{.004} & \entry{.720}{.011} & \entry{.771}{.005} \\
\midrule
\multicolumn{11}{l}{\textit{Matryoshka BatchTopK (131k)}} \\
131k & 200 & \entry{\mathbf{.773}}{.004} & \entry{.799}{.004} & \entry{\mathbf{.768}}{.004} & \entry{\mathbf{.814}}{.001} & \entry{\mathbf{.860}}{.002} & \entry{\mathbf{.812}}{.001} & \entry{.739}{.008} & \entry{\mathbf{.723}}{.019} & \entry{.754}{.006} \\
131k & 400 & \entry{.773}{.001} & \entry{\mathbf{.801}}{.001} & \entry{.767}{.002} & \entry{.812}{.001} & \entry{.858}{.001} & \entry{.810}{.002} & \entry{\mathbf{.739}}{.002} & \entry{.713}{.008} & \entry{\mathbf{.764}}{.006} \\
\midrule
\multicolumn{11}{l}{\textit{BatchTopK (524k)}} \\
524k & 600 & \entry{.772}{.004} & \entry{.801}{.004} & \entry{.766}{.004} & \entry{.813}{.002} & \entry{\mathbf{.861}}{.002} & \entry{.812}{.003} & \entry{.747}{.006} & \entry{.719}{.011} & \entry{.775}{.007} \\
524k & 800 & \entry{\mathbf{.773}}{.004} & \entry{\mathbf{.803}}{.004} & \entry{\mathbf{.768}}{.006} & \entry{\mathbf{.814}}{.003} & \entry{.860}{.002} & \entry{\mathbf{.813}}{.003} & \entry{\mathbf{.748}}{.004} & \entry{\mathbf{.719}}{.010} & \entry{\mathbf{.776}}{.010} \\
\midrule
\multicolumn{11}{l}{\textit{Matryoshka BatchTopK (524k)}} \\
524k & 200 & \entry{\mathbf{.774}}{.002} & \entry{.799}{.002} & \entry{\mathbf{.768}}{.002} & \entry{\mathbf{.813}}{.001} & \entry{\mathbf{.859}}{.001} & \entry{\mathbf{.812}}{.000} & \entry{.734}{.005} & \entry{.705}{.011} & \entry{.764}{.011} \\
524k & 400 & \entry{.771}{.002} & \entry{\mathbf{.800}}{.002} & \entry{.767}{.003} & \entry{.810}{.001} & \entry{.858}{.001} & \entry{.809}{.002} & \entry{\mathbf{.743}}{.003} & \entry{\mathbf{.708}}{.013} & \entry{\mathbf{.778}}{.009} \\
\bottomrule
\end{tabular}%
}
\endgroup
\end{table*}

\begin{table*}[t]
\centering
\caption{SAEBench evaluations of released Gemma Scope layer-12 checkpoints \citep{lieberum2024gemma,karvonen2025saebench}, organized by dictionary width and average sparsity $L_0$.}
\label{tab:gemma-scope-l12}
\resizebox{\textwidth}{!}{%
\begin{tabular}{lr*{12}{r}}
\toprule
& & \multicolumn{9}{c}{Sparse probing} & \multicolumn{3}{c}{RAVEL} \\
\cmidrule(lr){3-11}\cmidrule(lr){12-14}
Width & $L_0$
& \multicolumn{3}{c}{$k=1$}
& \multicolumn{3}{c}{$k=2$}
& \multicolumn{3}{c}{$k=5$}
& \multirow{2}{*}{\begin{tabular}{@{}c@{}}Disent-\\angle $\uparrow$\end{tabular}}
& \multirow{2}{*}{Cause $\uparrow$}
& \multirow{2}{*}{\begin{tabular}{@{}c@{}}Isola-\\tion $\uparrow$\end{tabular}} \\
\cmidrule(lr){3-5}\cmidrule(lr){6-8}\cmidrule(lr){9-11}
& & Acc. $\uparrow$ & AUC $\uparrow$ & F1 $\uparrow$ & Acc. $\uparrow$ & AUC $\uparrow$ & F1 $\uparrow$ & Acc. $\uparrow$ & AUC $\uparrow$ & F1 $\uparrow$ & & & \\
\midrule
16k & 22 & 0.7337 & 0.7476 & 0.7159 & 0.7529 & 0.7861 & 0.7441 & 0.7771 & 0.8243 & 0.7717 & 0.5989 & 0.5121 & 0.6857 \\
16k & 41 & 0.7519 & 0.7710 & 0.7366 & 0.7686 & 0.8056 & 0.7601 & 0.7893 & 0.8370 & 0.7835 & 0.6644 & 0.6163 & 0.7126 \\
16k & 82 & 0.7627 & 0.7857 & 0.7545 & 0.7821 & 0.8207 & 0.7768 & 0.8047 & 0.8523 & 0.8016 & 0.6960 & 0.6489 & 0.7431 \\
16k & 176 & 0.7679 & 0.7938 & 0.7602 & 0.7890 & 0.8262 & 0.7851 & 0.8129 & 0.8587 & 0.8119 & 0.7206 & 0.6508 & \textbf{0.7904} \\
16k & 445 & \textbf{0.7725} & \textbf{0.8033} & \textbf{0.7687} & \textbf{0.7942} & \textbf{0.8363} & \textbf{0.7921} & \textbf{0.8171} & \textbf{0.8632} & \textbf{0.8162} & \textbf{0.7488} & \textbf{0.7205} & 0.7771 \\
\midrule
32k & 12 & 0.7179 & 0.7264 & 0.6960 & 0.7350 & 0.7665 & 0.7181 & 0.7570 & 0.8046 & 0.7489 & 0.6050 & 0.5461 & 0.6639 \\
32k & 22 & 0.7379 & 0.7498 & 0.7221 & 0.7535 & 0.7824 & 0.7418 & 0.7758 & 0.8187 & 0.7711 & 0.6416 & 0.6251 & 0.6582 \\
32k & 40 & 0.7520 & 0.7695 & 0.7422 & 0.7662 & 0.7995 & 0.7590 & 0.7872 & 0.8327 & 0.7842 & 0.6755 & 0.6200 & 0.7311 \\
32k & 76 & 0.7607 & 0.7815 & 0.7534 & 0.7761 & 0.8138 & 0.7721 & 0.7990 & 0.8466 & 0.7958 & 0.7075 & 0.6714 & 0.7436 \\
32k & 155 & 0.7675 & 0.7907 & 0.7609 & 0.7865 & 0.8255 & 0.7833 & 0.8054 & 0.8528 & 0.8025 & 0.7249 & 0.6881 & 0.7616 \\
32k & 360 & \textbf{0.7749} & \textbf{0.8049} & \textbf{0.7713} & \textbf{0.7897} & \textbf{0.8306} & \textbf{0.7866} & \textbf{0.8140} & \textbf{0.8601} & \textbf{0.8129} & \textbf{0.7450} & \textbf{0.7105} & \textbf{0.7794} \\
\midrule
65k & 21 & 0.7361 & 0.7459 & 0.7139 & 0.7507 & 0.7785 & 0.7359 & 0.7743 & 0.8186 & 0.7686 & 0.6469 & 0.6053 & 0.6885 \\
65k & 38 & 0.7501 & 0.7656 & 0.7399 & 0.7660 & 0.7962 & 0.7589 & 0.7874 & 0.8290 & 0.7834 & 0.6767 & 0.6345 & 0.7188 \\
65k & 72 & 0.7647 & 0.7816 & 0.7552 & 0.7810 & 0.8128 & 0.7751 & 0.7987 & 0.8451 & 0.7965 & 0.7041 & 0.6596 & 0.7487 \\
65k & 141 & 0.7717 & 0.7930 & 0.7653 & 0.7858 & 0.8231 & 0.7814 & 0.8087 & 0.8541 & 0.8069 & 0.7346 & \textbf{0.7158} & 0.7535 \\
65k & 297 & \textbf{0.7788} & \textbf{0.8032} & \textbf{0.7731} & \textbf{0.7947} & \textbf{0.8343} & \textbf{0.7925} & \textbf{0.8178} & \textbf{0.8611} & \textbf{0.8169} & \textbf{0.7487} & 0.7016 & \textbf{0.7959} \\
\midrule
131k & 12 & 0.7108 & 0.7198 & 0.6870 & 0.7311 & 0.7585 & 0.7146 & 0.7525 & 0.7952 & 0.7443 & 0.6004 & 0.5764 & 0.6245 \\
131k & 20 & 0.7279 & 0.7362 & 0.7074 & 0.7483 & 0.7752 & 0.7335 & 0.7698 & 0.8132 & 0.7621 & 0.6408 & 0.5950 & 0.6866 \\
131k & 36 & 0.7379 & 0.7513 & 0.7222 & 0.7565 & 0.7843 & 0.7471 & 0.7810 & 0.8226 & 0.7767 & 0.6828 & 0.6451 & 0.7205 \\
131k & 67 & 0.7554 & 0.7713 & 0.7444 & 0.7693 & 0.8019 & 0.7622 & 0.7913 & 0.8347 & 0.7876 & 0.7083 & 0.6734 & 0.7432 \\
131k & 129 & 0.7677 & 0.7865 & 0.7608 & 0.7837 & 0.8173 & 0.7794 & 0.8027 & 0.8475 & 0.8011 & 0.7265 & 0.6561 & \textbf{0.7969} \\
131k & 264 & \textbf{0.7814} & \textbf{0.8033} & \textbf{0.7773} & \textbf{0.7941} & \textbf{0.8290} & \textbf{0.7914} & \textbf{0.8138} & \textbf{0.8575} & \textbf{0.8125} & \textbf{0.7486} & \textbf{0.7281} & 0.7691 \\
\midrule
262k & 11 & 0.7068 & 0.7168 & 0.6786 & 0.7232 & 0.7503 & 0.7054 & 0.7430 & 0.7872 & 0.7310 & 0.6166 & 0.5987 & 0.6345 \\
262k & 21 & 0.7240 & 0.7355 & 0.7037 & 0.7443 & 0.7732 & 0.7318 & 0.7657 & 0.8078 & 0.7607 & 0.6355 & 0.6169 & 0.6540 \\
262k & 36 & 0.7375 & 0.7533 & 0.7252 & 0.7573 & 0.7879 & 0.7479 & 0.7772 & 0.8195 & 0.7723 & 0.6795 & 0.6567 & 0.7023 \\
262k & 67 & 0.7510 & 0.7684 & 0.7427 & 0.7666 & 0.7960 & 0.7596 & 0.7911 & 0.8325 & 0.7883 & 0.7039 & 0.6652 & 0.7426 \\
262k & 121 & 0.7656 & 0.7854 & 0.7572 & 0.7818 & 0.8155 & 0.7772 & 0.8029 & 0.8480 & 0.8008 & 0.7158 & 0.6737 & 0.7579 \\
262k & 243 & \textbf{0.7733} & \textbf{0.7951} & \textbf{0.7655} & \textbf{0.7928} & \textbf{0.8291} & \textbf{0.7906} & \textbf{0.8119} & \textbf{0.8561} & \textbf{0.8105} & \textbf{0.7523} & \textbf{0.7219} & \textbf{0.7827} \\
\midrule
524k & 22 & 0.7129 & 0.7276 & 0.6976 & 0.7218 & 0.7529 & 0.7104 & 0.7453 & 0.7893 & 0.7384 & 0.6138 & 0.5800 & 0.6475 \\
524k & 29 & 0.7258 & 0.7421 & 0.7139 & 0.7400 & 0.7725 & 0.7297 & 0.7629 & 0.8076 & 0.7573 & 0.6451 & 0.6002 & 0.6900 \\
524k & 46 & 0.7387 & 0.7552 & 0.7264 & 0.7535 & 0.7872 & 0.7457 & 0.7770 & 0.8211 & 0.7731 & 0.6803 & 0.6602 & 0.7005 \\
524k & 65 & 0.7523 & 0.7688 & 0.7412 & 0.7663 & 0.7981 & 0.7589 & 0.7879 & 0.8303 & 0.7856 & 0.6958 & 0.6583 & 0.7332 \\
524k & 115 & 0.7637 & 0.7815 & 0.7550 & 0.7791 & 0.8110 & 0.7734 & 0.8004 & 0.8419 & 0.7981 & 0.7251 & 0.6760 & 0.7741 \\
524k & 227 & \textbf{0.7735} & \textbf{0.7955} & \textbf{0.7682} & \textbf{0.7927} & \textbf{0.8268} & \textbf{0.7890} & \textbf{0.8106} & \textbf{0.8520} & \textbf{0.8094} & \textbf{0.7555} & \textbf{0.7330} & \textbf{0.7781} \\
\midrule
1m & 19 & 0.7028 & 0.7182 & 0.6875 & 0.7165 & 0.7497 & 0.7081 & 0.7396 & 0.7847 & 0.7350 & 0.5910 & 0.5570 & 0.6250 \\
1m & 26 & 0.7234 & 0.7388 & 0.7047 & 0.7357 & 0.7693 & 0.7273 & 0.7587 & 0.8084 & 0.7523 & 0.6371 & 0.6120 & 0.6622 \\
1m & 58 & 0.7354 & 0.7548 & 0.7239 & 0.7550 & 0.7876 & 0.7499 & 0.7755 & 0.8193 & 0.7737 & 0.6722 & 0.6560 & 0.6884 \\
1m & 73 & 0.7526 & 0.7714 & 0.7435 & 0.7671 & 0.8013 & 0.7633 & 0.7887 & 0.8295 & 0.7866 & 0.6933 & 0.6533 & 0.7334 \\
1m & 107 & 0.7633 & 0.7797 & 0.7545 & 0.7737 & 0.8030 & 0.7684 & 0.7959 & 0.8385 & 0.7937 & 0.7201 & 0.6857 & 0.7545 \\
1m & 207 & \textbf{0.7788} & \textbf{0.7983} & \textbf{0.7726} & \textbf{0.7919} & \textbf{0.8264} & \textbf{0.7881} & \textbf{0.8095} & \textbf{0.8514} & \textbf{0.8081} & \textbf{0.7420} & \textbf{0.6860} & \textbf{0.7979} \\
\bottomrule
\end{tabular}%
}
\end{table*}

\subsection{Wall-clock training times on H200 GPUs}
\label{app:training-runtime}
\Cref{tab:gemma-training-times} compares wall-clock training times on eight NVIDIA H200 GPUs, with one model measured at a time and seed 0 throughout. Each measurement covers the first 999,424 main-training tokens of the original 50M-token schedule. Global batch size is 4,096 for \befond, BatchTopK, and Matryoshka, and 2,048 for MF20+clip.

\begin{table}[htbp]
\centering\small
\caption{Elapsed training time in seconds for 999,424 training tokens on eight NVIDIA H200 GPUs (lower is better). Each model is measured once at seed 0. Timings include \befond's prior updates and exclude evaluation, initialization, warmup, data preparation, checkpointing, and logging. Matryoshka denotes Matryoshka BatchTopK. Mean-field + Adam uses 20 inference steps, gradient clipping, and batch skipping.}
\label{tab:gemma-training-times}
\begin{tabular}{@{}rrrrr@{}}
\toprule
Width $K$ & \befond & BatchTopK & Matryoshka & Mean-field + Adam \\
\midrule
16,384  & 33.125  & 2.719  & 2.991  & 2.714 \\
65,536  & 118.959 & 9.202  & 10.737 & 10.157 \\
131,072 & 234.665 & 17.935 & 20.075 & 19.914 \\
524,288 & 989.321 & 70.121 & 77.382 & 78.331 \\
\bottomrule
\end{tabular}
\end{table}

Inputs are loaded, shuffled, whitened, and placed in GPU memory before timing. Initialization and CUDA/JIT warmup are excluded, and the complete initial training state is restored after warmup. The timed interval includes inference, losses, parameter updates, normalization, required communication, and \befond's prior updates with convergence checks. Evaluation, checkpoint writing, and logging are excluded. PyTorch timings use the maximum CUDA-synchronized elapsed time across ranks, while JAX waits for all output shards on all devices.

\befond retains its original stochastic training horizons, capped at 100 with one inner step, and warms the complete measured prefix to cover adaptive solver paths. MF uses 20 inference steps. Its clipping and loss-based skipping begin after 2,000 updates and are inactive during this prefix. Model tensors are FP32, with native FP64 intermediates retained in \befond's moment reductions and Woodbury calculations. \befond and the amortized baselines use PyTorch, while MF uses JAX. The distributed baseline implementations preserve the global TopK budget and synchronize gradients, thresholds, and dead-feature histories.

\befond takes approximately $12$--$14\times$ as long as BatchTopK across these widths.

\clearpage
\flushbottom


\end{document}